\pdfoutput=1
\newif\ifdraft\newif\ifdraft
\newif\ifabbrv

\documentclass[11pt]{article}

\usepackage[utf8]{inputenc}
\usepackage{hyperref}

\usepackage{amsmath,amssymb,color,graphicx}
\ifabbrv
\else
\usepackage{amsthm}
\usepackage[top=1in,bottom=1in,left=1in,right=1in]{geometry}
\usepackage{fixltx2e}
\usepackage{fullpage}
\fi

\usepackage{mathtools}
\usepackage{euscript}

\usepackage{subfig}

\usepackage{enumerate}
\usepackage[linewidth=1pt]{mdframed}

\usepackage{textcomp}
\usepackage{microtype}
\usepackage{xspace}
\usepackage{tikz}  
\usepackage{floatflt}
\usepackage{wrapfig}
\usepackage{soul}
\usepackage{bm}

\ifdraft

\else

\fi

\usepackage{scalerel}

\usepackage{hyperref}%
\hypersetup{%
   breaklinks,%
   ocgcolorlinks,
   colorlinks=true,%
   linkcolor=[rgb]{0.0,0.0,0.0},%
   citecolor=[rgb]{0,0,0.0}
}

\ifdraft

\else
\def\micha#1{}
\def\pankaj#1{}
\def\alex#1{}
\def\danny#1{}
\fi

\newcommand{\seclab}[1]{\label{sec:#1}}
\newcommand{\theolab}[1]{\label{theo:#1}}
\newcommand{\lemlab}[1]{\label{lemma:#1}}

\newcommand{\corlab}[1]{\label{cor:#1}}
\newcommand{\figlab}[1]{\label{fig:#1}}

\newcommand{\subseclab}[1]{\label{subsec:#1}}

\newcommand{\lemref}[1]{Lemma~\ref{lemma:#1}}
\newcommand{\secref}[1]{Section~\ref{sec:#1}}

\newcommand{\corref}[1]{Corollary~\ref{cor:#1}}
\newcommand{\figref}[1]{Figure~\ref{fig:#1}}

\newcommand{\subsecref}[1]{Section~\ref{subsec:#1}}

\ifabbrv
\else
\newtheorem{theorem}{Theorem}[section]
\newtheorem{corollary}[theorem]{Corollary}

\newtheorem{lemma}[theorem]{Lemma}

\newtheorem*{claim*}{Claim}
\fi

\graphicspath{{./figs/}}

\newcommand*{\reals}{{\mathbb R}}

\newcommand*{\Reals}{\reals}
\newcommand*{\ceil}[1]{\left\lceil #1\right\rceil}

\newcommand*{\C}{\EuScript{C}}

\newcommand*{\bd}{{\partial}}
\newcommand*{\eps}{{\varepsilon}}

\newcommand*{\inprod}[2]{\langle #1, #2\rangle}

\newcommand*{\grid}{\mathbb{G}}
\newcommand*{\gridcells}{\EuScript{G}}

\newcommand*{\cl}{\mathop{\mathrm{cl}}}

\newcommand*{\annular}{\boxbox}

\newcommand*{\Int}{\mathop{\mathrm{int}}}

\newcommand*{\abs}[1]{\mathopen| #1 \mathclose|}
\newcommand*{\norm}[1]{\mathopen|| #1 \mathclose||}

\newcommand*{\assign}{\coloneqq}

\newcommand*{\verts}{\EuScript{V}}

\newcommand*{\mparagraph}[1]{\medskip\noindent\textbf{#1}.}

\newcommand*{\mycent}{{\normalfont \textrm{\textcent}}}

\newcommand*{\altern}[1]{\alpha(#1)}
\newcommand*{\envir}{\EuScript{W}}
\newcommand*{\freesp}{\EuScript{F}}
\newcommand*{\fdfreesp}{\mathbf{F}}

\newcommand*{\RA}{\mathrm{RA}}
\newcommand*{\RAp}[1]{\mathrm{RA}({#1})}

\newcommand*{\corridors}{\EuScript{K}}
\newcommand*{\corridor}{K}

\newcommand*{\gridvert}{\EuScript{V}}
\newcommand*{\nearvert}{\widetilde{\EuScript{V}}}

\newcommand*{\fdpi}{\boldsymbol{\pi}}
\newcommand*{\loplancost}[1]{\$({#1})}
\newcommand*{\plancost}[1]{\mycent({#1})}
\newcommand*{\pathcost}[1]{\mycent({#1})}
\newcommand*{\robA}{A}
\newcommand*{\robB}{B}
\newcommand*{\graph}{\mathcal{G}}
\newcommand*{\gedges}{\mathcal{E}}
\newcommand*{\gverts}{\mathcal{C}}
\newcommand*{\freept}[1]{\freesp[#1]}

\newcommand*{\enVerts}{\mathsf{X}}

\newcommand*{\moveseq}[1]{\mathopen\langle #1 \mathclose\rangle}

\newcommand*{\linfnorm}[2]{\norm{#1-#2}_\infty}
\newcommand*{\lonorm}[2]{\norm{#1-#2}_1}
\newcommand*{\ltnorm}[2]{\norm{#1-#2}_2}

\newcommand*{\sanctum}[1]{#1^{S}}

\newcommand*{\plpt}[2]{#1[#2]}
\newcommand*{\geodesic}[2]{\varrho(#1,#2)}
\newcommand*{\georegion}[3]{\varrho_{#3}(#1,#2)}
\newcommand*{\geopt}[3]{\varrho_{\freept{#3}}(#1,#2)}
\newcommand*{\fd}[1]{\boldsymbol{#1}}
\newcommand*{\linfd}[2]{\ldist{#1}{#2}{\infty}}

\newcommand*{\ldist}[3]{d_{#3}(#1,#2)}

\newcommand*{\Abox}{\Box_A}
\newcommand*{\Bbox}{\Box_B}
\newcommand*{\antip}[1]{\mathrm{anti}(#1)}
\newcommand*{\nil}{\textsc{nil}}

\newcommand*{\gridfaces}{\mathsf{F}_\eps}

\newcommand*{\ccat}{\|}

\newcommand*{\shiftseg}[2]{{#1}^{(#2)}}
\newcommand*{\pointseg}[2]{{#1}_{#2}}
\newcommand*{\lines}{\EuScript{L}}
\newcommand*{\corlen}[1]{\mathrm{len}(#1)}
\newcommand*{\dclose}[2]{(#1,#2)\text{-close}}

\def\polyn{\mathop{\mathrm{poly}}}

\makeatletter
\DeclareFontEncoding{LS1}{}{\noaccents@} 
\makeatother
\DeclareFontSubstitution{LS1}{stix}{m}{n}

\DeclareSymbolFont{symbolsSTIX}{LS1}{stixscr}{m}{n}
\DeclareSymbolFont{symbols2STIX}{LS1}{stixfrak}{m}{n}
\DeclareMathSymbol{\boxbox}{\mathbin}{symbols2STIX}{"B7}
\DeclareMathSymbol{\boxwhite}{\mathord}{symbolsSTIX}{"B8}
\DeclareMathSymbol{\boxblackbox}{\mathord}{symbolsSTIX}{"BA}

\newcommand{\ignore}[1]{}

\makeatletter
\long\def\@makecaption#1#2{
   \vskip 10pt
   \setbox\@tempboxa\hbox{{\footnotesize \textbf{#1.} #2}}
   \ifdim \wd\@tempboxa >\hsize         
       {\footnotesize \textbf{#1.} #2\par}
     \else                              
       \hbox to\hsize{\hfil\box\@tempboxa\hfil}
   \fi}
\makeatother

\makeatletter
\DeclareRobustCommand\onedot{\futurelet\@let@token\@onedot}
\def\@onedot{\ifx\@let@token.\else.\null\fi\xspace}

\def\eg{e.g\onedot} 
\def\ie{i.e\onedot} 
\def\cf{\emph{cf}\onedot}

\def\etal{\emph{et~al}\onedot}
\makeatother

\begin{document}

\newcommand\funding{Work by P.A. and A.S. has been partially supported by IIS-1814493, CCF-2007556, and CCF-2223870.
Work by D.H. has been supported in part by the Israel Science Foundation (grant nos.~1736/19 and~2261/23), by NSF/US-Israel-BSF (grant no.~2019754), by the Israel Ministry of Science and Technology (grant no.~103129), by the Blavatnik Computer Science Research Fund, and by the Yandex Machine Learning Initiative for Machine Learning at Tel Aviv University.
Work by M.S. has been supported by Israel Science Foundation Grants 260/18 and 495/23.
}

\newcommand\mynewline{}

\ifabbrv
\else 

\title{Near-Optimal Min-Sum Motion Planning for Two Square Robots in a Polygonal Environment\thanks{An abridged preliminary version of this work appears in the Proceedings of the 2024 Annual ACM-SIAM Symposium on Discrete Algorithms (SODA).\newline\funding}}

\author{Pankaj K. Agarwal\thanks{%
Department of Computer Science, Duke University, Durham, NC
27708-0129, USA; {\tt pankaj@cs.duke.edu}.}
\and
Dan Halperin\thanks{%
School of Computer Science, Tel Aviv University, Tel~Aviv 69978, Israel;
{\tt danha@tauex.tau.ac.il}.}
\and
Micha Sharir\thanks{%
School of Computer Science, Tel Aviv University, Tel~Aviv 69978, Israel;
{\tt michas@tauex.tau.ac.il}.}
\and
Alex Steiger\thanks{%
Department of Computer Science, Duke University, Durham, NC
27708-0129, USA; {\tt asteiger@cs.duke.edu}.}
}

\date{}

\maketitle

\begin{abstract}
Let $\envir \subset \Reals^2$ be a planar polygonal environment (\ie, a polygon potentially with holes) with a total of $n$ vertices, and let $A,B$ be two robots, each modeled as an axis-aligned unit square, that can translate inside $\envir$. Given source and target placements $s_A,t_A,s_B,t_B \in \envir$ of $\robA$ and $\robB$, respectively, the goal is to compute a \emph{collision-free motion plan} $\fdpi^*$, \ie, a motion plan that continuously moves $\robA$ from $s_A$ to $t_A$ and $B$ from $s_B$ to $t_B$ so that $\robA$ and $\robB$ remain inside $\envir$ and do not collide with each other during the motion. Furthermore, if such a plan exists, then we wish to return a plan that minimizes the sum of the lengths of the paths traversed by the robots, $\abs{\fdpi^*}$. Given $\envir, s_A,t_A,s_B,t_B$ and a parameter $\eps > 0$, we present an $n^2\eps^{-O(1)} \log n$-time $(1+\eps)$-approximation algorithm for this problem. We are not aware of any polynomial time algorithm for this problem, nor do we know whether the problem is NP-Hard. Our result is the first polynomial-time $(1+\eps)$-approximation algorithm for an optimal motion planning problem involving two robots moving in a polygonal environment.
\end{abstract}
\fi

\section{Introduction}
The basic motion-planning problem is to decide whether a robot (i.e., a rigid or multi-link moving object) 
can move from a given start position to a given target position without colliding with obstacles on its way, 
and avoiding collision of different parts of the robot. If the answer is positive, we also want to plan such 
a motion. 
With the advancement of robotics, we witness the growing deployment of \emph{teams} 
of robots in logistics, wildlife monitoring, buildings and bridges inspection and more. 
Motion planning for many robots requires that, in addition to not colliding with obstacles, 
the robots should not collide with one another, which in turn necessitates studying the problem 
in high-dimensional \emph{configuration spaces}. Furthermore, we wish to ensure a good quality of the
motion, such as being short or having a small makespan.
Already for two simple robots, such as unit squares or discs, translating in a planar polygonal environment,
little is known when it comes to optimizing the robot motion. Although polynomial-time algorithms are known for 
computing a collision-free motion plan of two simple robots~\cite{SS91}, no polynomial-time algorithm is known for computing
a plan such that the sum (or the maximum) of the path lengths of the two robots is minimized, nor is the problem
known to be NP-hard. Even a polynomial-time constant-factor approximation algorithm is not known for this problem (without further restrictions).

\mparagraph{Problem statement}
Let $\Box = \{x \in \Reals^2 \mid \norm{x}_\infty \leq 1\}$ denote the unit-radius axis-aligned square 
centered at the origin, referred to as a \emph{unit square} for short. 
For a point $p \in \Reals^2$ and a real value $\lambda \geq 0$, we use $p+\lambda\Box$ to denote the 
axis-parallel square of radius $\lambda$ centered at $p$. Let $\robA$ and $\robB$ be two robots, each modeled as a
unit square, that can translate inside the same closed planar polygonal environment (a connected polygon possibly
with holes) $\envir$ with $n$ vertices. A placement of $\robA$ or $\robB$ is represented by a point 
in $\envir$ --- the position of its center. For such a placement to be free of collision with $\bd\envir$, the boundary of $\envir$,
the representing point should be at $L_\infty$-distance at least $1$ from $\bd\envir$. We denote by $\freesp$,
the \emph{free space} of a single robot, the subset of $\envir$ consisting of such points.  Note that the robots may be at $L_\infty$-distance $1$ from $\bd \envir$ and hence they are allowed to make contact with the obstacles.
A (joint) \emph{configuration} of $\robA$ and $\robB$ 
is represented as a pair $(p_A,p_B) \in \envir \times \envir$, where $p_A$ (resp., $p_B$) is the placement 
of $\robA$ (resp., $\robB$). We also represent a configuration as a point $p \in \Reals^4$, where the first 
(resp., last) pair of coordinates represent the placement of $\robA$ (resp., $\robB$). The 
\emph{configuration space}, called \emph{C-space} for short, namely the set of all configurations, 
is thus represented as $\envir\times\envir\subset\Reals^4$. 
A configuration $\fd{p} = (p_A,p_B) \in \Reals^4$ is called \emph{free} if $p_A, p_B\in\freesp$, that is,
$p_A+\Box, p_B +\Box \subseteq \envir$, and 
$\norm{p_A-p_B}_\infty \geq 2$. Such a free configuration is called a \emph{kissing configuration} 
if $\norm{p_A - p_B}_\infty = 2$, \ie, the robots touch each other (but their interiors remain disjoint). 
Let $\fdfreesp \assign \fdfreesp(\envir)$ 
denote the (four-dimensional) \emph{free space}, namely the set of all free configurations.
Clearly, $\fdfreesp \subset \freesp \times \freesp$.

Two free configurations $\fd{s},\fd{t} \in \fdfreesp$ are \emph{reachable} if they lie in the same connected component of $\fdfreesp$, \ie,
there is a path contained in $\fdfreesp$ from $\fd{s}$ to $\fd{t}$.
For two reachable free configurations
$\fd{s} \assign (s_A,s_B), \fd{t} \assign (t_A,t_B) \in \fdfreesp$, a path $\fdpi \subseteq \fdfreesp$ 
from $\fd{s}$ to $\fd{t}$ is called a \emph{(feasible) plan} of $\robA$ and $\robB$ from 
$\fd{s}$ to $\fd{t}$, or an \emph{$(\fd{s},\fd{t})$-plan} for brevity.
With a slight abuse of notation, we also use $\fdpi$ as a (continuous) parameterization 
$\fdpi: [0,1] \rightarrow \fdfreesp$, with $\fdpi(0) = \fd{s}$ and $\fdpi(1) = \fd{t}$.
For a path $\fdpi \subseteq \fdfreesp$, let $\pi_A$ (resp., $\pi_B$) be the projection of 
$\fdpi$ onto the two-dimensional plane spanned by the first (resp., last) two coordinates, 
which specifies the path followed by $\robA$ (resp., $\robB$) that $\fdpi$ induces; 
we have $\pi_A, \pi_B \subset \freesp$. 
Let $\pathcost{\pi_A}, \pathcost{\pi_B}$ denote the (Euclidean) arc length of the paths $\pi_A,\pi_B$, 
respectively, in $\Reals^2$. We define $\plancost{\fdpi}$, the \emph{cost} of $\fdpi$, to be the sum 
of the lengths of $\pi_A$ and $\pi_B$, \ie, $\plancost{\fdpi} = \pathcost{\pi_A} + \pathcost{\pi_B}$.
Let $\fdpi^*(\fd{s},\fd{t})$ denote an \emph{optimal} $(\fd{s},\fd{t})$-plan, \ie, 
a plan that minimizes the sum of the lengths of the two paths.\footnote{%
  The existence of $\fdpi^*$ can be proved using a simple compactness argument, since $\freesp$ and $\fdfreesp$ are closed.} 
If $\fd{s}$ and $\fd{t}$ are not reachable, \ie, they lie in different connected components of $\fdfreesp$, then $\fdpi^*(\fd{s},\fd{t})$ does not exist.
We refer to the problem of computing $\fdpi^*(\fd{s},\fd{t})$ as the 
\emph{(optimal) min-sum motion-planning problem}.
In this paper we study the min-sum motion-planning problem for two translating axis-aligned unit squares, and present 
a $(1+\eps)$-approximation algorithm that runs in $n^2\eps^{-O(1)}\log n$ time.

\mparagraph{Related work}
Algorithmic motion planning has been studied for well over fifty years in 
computer science and beyond.
The rigorous study of algorithmic motion planning dates back to the work of Schwartz and Sharir~\cite{schwartz1983piano} and 
Canny~\cite{Can88}. See~\cite{hks-r-18,hss-amp-18,Lav06,m-spn-18} for a review of
key relevant results.
We mention here only a small sample of these results---the ones that are most closely related to the problem at hand.

When only one square robot translates, or more generally when only one convex polygonal robot of a constant 
description complexity (that is, with a constant number of vertices) translates, the problem is 
equivalent---through C-space formulation---to moving a point robot amid polygonal 
obstacles with $O(n)$ vertices, and it can be solved in $O(n\log n)$ 
time~\cite{chen2015computing,DBLP:journals/siamcomp/HershbergerS99,DBLP:conf/stoc/Wang21}.
Interestingly, the analogous problem in 
3D, namely finding the shortest path for a point robot amid polyhedral obstacles, 
is NP-hard~\cite{CanRei87} and fast $(1+\eps)$-approximation algorithms are known \cite{CanRei87,SCY00}. Note that this hard problem has only three degrees of freedom 
of motion, and there are other optimal motion-planning problems for robots with three degrees of freedom 
that are NP-hard~\cite{DBLP:conf/compgeom/AsanoKY96,DBLP:conf/cccg/AsanoKY03}. Our two-square problem 
has four degrees of freedom, which suggests it might be NP-hard as well, though, as we have remarked earlier, this is 
an open problem.

Computing a feasible (not necessarily optimal) plan for a team of translating unit square robots in a polygonal environment 
is PSPACE-hard~\cite{DBLP:journals/ijrr/SoloveyH16} (see also~\cite{DBLP:conf/fun/BrockenHKLS21,DBLP:conf/fun/BrunnerCDHHSZ21,DBLP:journals/tcs/HearnD05,hopcroft1984complexity,DBLP:journals/ipl/SpirakisY84,DBLP:conf/aaai/YuL13} for related intractibility results).
Notwithstanding a rich literature on multi-robot motion planning 
in both continuous and discrete setting (robots moving on a graph in the latter setting),~see, \eg,~\cite{DBLP:journals/trob/DayanSPH23,
DBLP:journals/ijrr/KaramanF11,
KavSveLatOve96,
DBLP:journals/cacm/Salzman19,
DBLP:journals/arobots/ShomeSDHB20,
stern2019multiagent,
DBLP:journals/arobots/TurpinMMK14}, little 
is known about algorithms producing paths with provable quality guarantees.
Approximation algorithms for minimizing the total path-length 
are given in~\cite{DBLP:journals/comgeo/AgarwalGHT23,SolomonHalperin2018,DBLP:conf/rss/SoloveyYZH15}
for a set of unit-disc robots assuming a certain separation between the start and goal positions, as well as from the obstacles. 
The separation assumption makes the problem considerably easier. A feasible plan always exists, and one can first compute an optimal path for each robot independently, ignoring other robots and then locally modify them so that the robots do not collide with each other during their motion. An $O(1)$-approximation algorithm was proposed in~\cite{demaine2019coordinated} for computing a plan that minimizes the makespan for a set of unit discs (or squares) in the plane without obstacles, again assuming some separation.
Computing the min-sum motion plan for two unit squares/discs even in the absence of 
obstacles is non-trivial \cite{Esteban2022,DBLP:conf/cccg/KirkpatrickL16}.
We are unaware of any constant-factor approximation algorithms for the min-sum motion-planning problem even for two unit squares/discs in 
a planar polygonal environment without any assumptions on the work environment or on the start/final configurations.

Quite a few of the algorithmic results for teams of robots distinguish between the labeled and unlabeled versions: In the labeled version, like in the two-square problem studied here, each robot is designated its own unique target position.  In the unlabeled case, each robot can finish at any of the (collective) target positions, as long as at the end of the motion all the target position are occupied by robots. For a team of unlabeled unit discs, an approximate solution for the minimum total path length is given in~\cite{DBLP:conf/rss/SoloveyYZH15}, assuming a certain separation between the start and goal positions of the robots, as well as from the obstacles.  A similar result has also been obtained for a team of labeled unit discs in~\cite{SolomonHalperin2018}, using the slightly more relaxed requirement of the existence of \emph{revolving areas} around the start and target positions. In both cases the approximation bounds are crude, and we omit them here. The latter result for labeled unit discs has recently been improved~\cite{DBLP:journals/comgeo/AgarwalGHT23}, to give an $O(1)$-approximation of the optimal total length of the paths, under exactly the same conditions as in~\cite{SolomonHalperin2018}.

The central and prevalent family of practical motion-planning techniques in robotics is based on sampling 
of the underlying C-space; see, e.g., ~\cite{KavSveLatOve96, Lav06,LavKuf00},
and~\cite{DBLP:journals/cacm/Salzman19} for a recent review. The original sampling-based motion-planning 
techniques aimed at finding a feasible solution by creating a roadmap of free configurations and connections 
between them in the C-space, while deferring the (necessarily suboptimal) optimization to a graph search 
on the resulting roadmap. This two-stage approach has detrimental effect on the quality of the approximation. 
For example, for a point robot moving amid polyhedra in 3-space, this approach could lead to paths that are 
hundredfold longer than optimal with high probability~\cite{DBLP:conf/wafr/NechushtanRH10}. This shortcoming 
was rectified in a breakthrough paper by Karaman and Frazzoli~\cite{DBLP:journals/ijrr/KaramanF11}, who 
presented a series of variants of the fundamental sampling-based techniques, that are guaranteed to be 
asymptotically optimal, namely converge to an optimal (e.g., shortest) path, when the number of samples 
tends to infinity. Most sampling-based planners come with only asymptotic guarantees of this type. 
Finite-time guarantees for sampling-based planners for a team of unit-disc robots are given 
in~\cite{DBLP:journals/trob/DayanSPH23}. 

Another major line of work on optimizing multi-robot motion plans addresses a discrete version of the 
problem, where robots are moving on graphs. In this setting the robots are often referred to as \emph{agents},
and the problem is called \emph{Multi Agent Path Finding (MAPF)}. There is a rich literature on MAPF, 
and we refer the reader to the recent survey~\cite{stern2019multiagent}. A commonly used optimization 
criterion (particularly in the study of MAPF, but elsewhere as well) is \emph{makespan}, where we wish 
to minimize the time by which all the robots reach their destination, assuming they move in some prespecified maximum speed; 
see, e.g., \cite{demaine2019coordinated,DBLP:conf/aaai/YuL13}.

There are a variety of additional optimization criteria in robot motion planning. A common one, related 
to motion safety, is requiring high \emph{clearance}, namely, requiring that the robot stays far from the 
obstacles in its environment---this can be obtained using Voronoi diagrams (e.g., \cite{o1985retraction}). 
In the context of multi-robot planning we may also require that the robots stay sufficiently far from one 
another (e.g.,~\cite{DBLP:journals/trob/DayanSPH23}). A natural requirement is to produce paths that are at once 
short and far away from obstacles, which is a more intricate task even for a single robot translating in 
the plane; see, e.g., \cite{DBLP:journals/talg/AgarwalFS18,DBLP:journals/comgeo/WeinBH07,DBLP:journals/ijrr/WeinBH08}.
%

\mparagraph{Our contributions}
We consider the following simple case of min-sum motion-planning for two unit-square robots. 
Let $\envir$ 
be a polygonal environment, \ie, a polygon possibly with holes. As already stated, we assume that the 
two robots $\robA$ and $\robB$ are axis-parallel squares of side-length $2$. Given a source and a target 
free configurations $(s_A,s_B), (t_A,t_B) \in \envir$, the goal is to compute a collision-free motion 
plan for $\robA$ from $s_A$ to $t_A$ and $\robB$ from $s_B$ to $t_B$, such that the sum of the lengths
of the two tours traversed by the robots is minimized, or otherwise report that there is no such 
collision-free motion plan. 
Our main result is the following theorem, which provides an efficient $\eps$-approximation algorithm for this problem\footnote{In principle, our approach extends to two identical centrally-symmetric regular convex polygons, but the analysis becomes even more technical, so for simplicity we only focus on unit squares.}.

%
\begin{theorem}
\theolab{two-robots-alg}
Let $\envir$ be a closed polygonal environment with $n$ vertices, let $\robA,\robB$ be two axis-parallel unit-square robots translating inside $\envir$,
and let $\fd{s},\fd{t}$ be source and target configurations of $\robA,\robB$. 
For any $\eps \in (0,1)$, a motion plan $\fdpi$ from $\fd{s}$ to $\fd{t}$ with 
$\plancost{\fdpi} \leq (1+\eps)\plancost{\fdpi^*}$, if there exists a such a motion, can be computed in 
$n^2 \eps^{-O(1)} \log n$ time, where $\fdpi^*$ is an optimal $(\fd{s},\fd{t})$-plan.
\end{theorem}

Although our result falls short of answering whether the min-sum problem for two robots is in P, it is a significant contribution to the theory of optimal multi-robot motion planning. First, as mentioned above, a polynomial-time algorithm was not known, even for constant-factor approximation, and we present an FPTAS for this problem. Second, we prove several structural properties of an optimal plan, which could lead to a polynomial-time algorithm in some special cases, \eg, when $\envir$ is rectilinear and we consider the $L_1$-length of a path. Note that our FPTAS does not rule out the possibility of the problem being NP-hard because, as in other NP-hard optimal motion-planning problems, the construction might use a polynomial number of bits. Finally, our algorithm is very simple and follows the widely-used sampling paradigm. More precisely,
we sample a finite set $\verts\subset \fdfreesp$ of free configurations that 
contains $\fd{s}, \fd{t}$. We connect a pair of configurations $\fd{p}\assign (p_A,p_B)$ and $\fd{q}\assign(q_A,q_B)$ in $\verts$ by an 
edge if there is a \emph{simple} (feasible) plan from $\fd{p}$ to $\fd{q}$, namely, we can move $\robA$ from $p_A$ to $q_A$ 
(not necessarily along a straight segment) while keeping $\robB$ parked at $p_B$ and then move $\robB$ from $p_B$ to $q_B$ while $\robA$ is parked at $q_A$, 
or vice-versa. 
The cost of the edge $(\fd{p},\fd{q})$ is the minimum cost of such a plan. We then 
compute a shortest path in this graph.
The question is, of course, how we (efficiently) choose a small number of free configurations (linear in $n$) so that the resulting 
graph is guaranteed to contain a path from $\fd{s}$ to $\fd{t}$ that corresponds to a near-optimal $(\fd{s},\fd{t})$-plan.
Most of this paper is about answering this question. We note that the runtime of our algorithm nearly matches that of 
the best known algorithm for finding \emph{any} $(\fd{s},\fd{t})$-plan for two unit squares in a 
planar polygonal environment, which takes $O(n^2)$ time \cite{SS91}.

There are four main technical contributions of this paper. First, we prove a few key properties of an 
optimal plan (\secref{optimal-plans}). Concretely, we show that there is always an optimal plan in which only one robot moves 
at any given time while the other robot is \emph{parked} (remains stationary). Thus an optimal plan 
can be represented as a sequence of \emph{moves}, where each move is specified as a 3-tuple $(R,\pi,p)$, 
where $R\in \{ \robA, \robB \}$ is the robot that is moving along a path $\pi \subseteq \envir$ and 
the other robot is parked at $p \in \freesp$, where $\pi \times \{p\}$ or $\{p\} \times \pi$ is in $\fdfreesp$ ($\pi$ also
encodes the starting and terminating placements of $R$ in this move). We refer to such a plan as 
a \emph{decoupled plan}.\footnote{We note that the notion of \emph{decoupled} has been 
used in multiple ways in the context of multi-robot motion planning~\cite{Lat91}.}

Second, we show that among all decoupled plans, there exists one in which 
for each move $(R,\pi,p)$, except possibly the first and the last moves, there is a point $q \in \pi$ 
such that $(p,q)$ (or $(q,p)$ as the case might be) is a kissing configuration. We refer to such a plan as a \emph{kissing plan}. 
We use the kissing property to prove that there exists an optimal, kissing plan composed of $O(\plancost{\pi}+1)$ moves.
Our usage of kissing configurations is different from earlier work (see, \eg, \cite{DBLP:journals/dcg/AronovBSSV99,DBLP:conf/icra/FortuneWY86,DBLP:journals/siamcomp/HopcroftW86}) in a few ways. First, the focus of these works is on motion in contact. For example, Aronov \etal~\cite{DBLP:journals/dcg/AronovBSSV99} use a continuum of kissing configurations to reduce the dimension of the underlying joint configuration space of a pair or of a triple of robots, under various extra conditions. In contrast, kissing configurations in this paper arise as part of individual robot moves, often a singular/discrete configuration, in a (possibly long) alternating sequence of moves. Second, earlier work deals with feasible motion, while we show that there exists optimal plans in which almost every move contains a kissing configuration.

Finally, we prove that there is always a kissing plan in which neither of the
robots is ever parked deep inside \emph{corridors}. A formal definitions of corridors is given in \secref{prelim}, 
but intuitively a corridor is a (narrow) region of $\freesp$ bounded by two of its edges that is far 
from all vertices of $\freesp$ and not wide enough to let one robot pass the other.

Next, using these three properties of an optimal plan, we show that we can deform an optimal kissing 
plan to a \emph{tame plan}, at a slight increase of its cost, in which (roughly speaking) a robot is 
always parked near a vertex of $\envir$ or of a corridor at each move.
Furthermore, the deformed plan $\widetilde{\fdpi}$ is composed of $O(\plancost{\fdpi}+1)$
moves and remains a kissing plan (\secref{approx-tame}).
Ensuring the kissing property in this deformation is delicate and requires a rather involved argument,
so we first prove the existence of a tame plan without ensuring the kissing property (\secref{approx-tame}).
This weaker property already leads to an $n^3\eps^{-O(1)}\log n$-time $(1+\eps)$-approximation algorithm. A key ingredient in computing these deformations is the notion of 
\emph{revolving areas} within $\freesp$, the two-dimensional free space with respect to one robot, roughly a unit square inside $\freesp$ 
(again see below for a precise definition). We can show that if each of $s_A,s_B,t_A,t_B$ lies
in a revolving area, then there is an $(\fd{s},\fd{t})$-plan $\fdpi$ composed of $O(1)$ moves with 
cost $\plancost{\fdpi} \leq \geodesic{s_A}{s_B} + \geodesic{t_A}{t_B} + O(1)$, where $\geodesic{\cdot}{\cdot}$ is the geodesic distance between two points in $\freesp$.
%
The notion of revolving areas was used in~\cite{DBLP:journals/comgeo/AgarwalGHT23,SolomonHalperin2018} to make a strong separation assumption on each of the start and target configurations, which was exploited to compute a near-optimal plan. Here, we prove the existence of revolving areas in the neighborhood of a non-tame plan and use them for auxiliary parking spots to convert the plan into a near-optimal tame plan.

The existence of an kissing, tame, near-optimal $(\fd{s},\fd{t})$-plan $\fdpi^*$ enables us to choose a set $\verts$ of $n\eps^{-O(1)}$ (nearly) kissing configurations and to build a graph $\graph$ over them so that $\fdpi^*$
can be retracted to a path in $\graph$ 
at a slight increase in its cost, thereby reducing the problem to computing a shortest path in $\graph$. Ensuring that the two robots do not collide with each other in the retracted path requires care and thus the retraction map is somewhat involved.
This retraction step introduces $O(\eps)$ additive error, so we need a separate procedure to handle the case when 
$\plancost{\fdpi^*}$ is small, say, at most $1/4$. 
By exploiting the topology of $\fdfreesp$, we describe an $O(n\log^2 n)$-time $O(1)$-approximation 
algorithm for computing an optimal $(\fd{s},\fd{t})$-plan when $\plancost{\fdpi^*} \le 1/4$ 
(\secref{small-opt}). We then plug it into the above algorithm
to obtain a $(1+\eps)$-approximation algorithm for all values of $\plancost{\fdpi^*}$.

\section{Preliminaries}
\seclab{prelim}

\begin{figure}[t]
\centering
\includegraphics[scale=0.90]{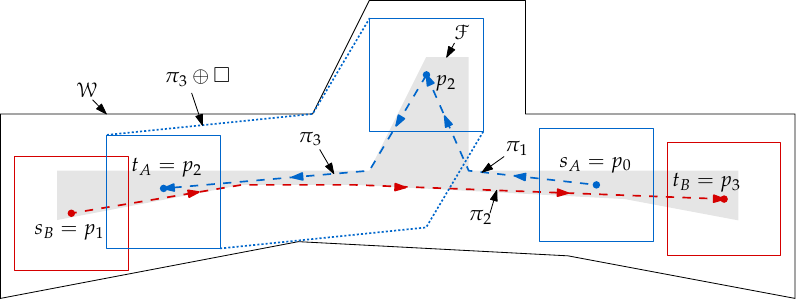}
\caption[A plan with three moves.]{An $(\fd{s},\fd{t})$-plan $\fdpi$ with 
$\moveseq{\fdpi} = (A,\pi_1,s_B), (B, \pi_2, p_2), (A, \pi_3, t_B)$. $(s_A,s_B)$ is $x$-separated and not $y$-separated.}
\figlab{prelim-example}
\end{figure}

\mparagraph{Definitions} 
Let $\freesp$ be the free space of one robot as defined above.
For a point $p \in \envir$, let $\freept{p} \assign \{x \in \freesp \mid \norm{x - p}_\infty \geq 2\}$ 
be the set of all placements $x \in \freesp$ of $\robA$ such that $A$ does not collide with $\robB$ if 
$\robB$ is placed at $p$, \ie, $\Int(x+\Box) \cap \Int(p+\Box) = \varnothing$. It is well known that 
$\freesp$ and $\freept{p}$ are polygonal and have $O(n)$ vertices, and that they can be computed in $O(n \log^2 n)$ time \cite{DBLP:books/lib/BergCKO08}.
See \figref{prelim-example}.
Let $\enVerts$ be the set of vertices of $\freesp$.
We regard $s_A,s_B,t_A,t_B$ as additional vertices of $\freesp$ and add them to $\enVerts$. 
For $p,q \in \freesp$, let $\geodesic{p}{q}$ denote the geodesic distance between $p$ and $q$ in $\freesp$.
We call a configuration $(a,b) \in \fdfreesp$ \emph{$x$-separated} if $\abs{x(a) - x(b)} \geq 2$ and 
\emph{$y$-separated} if $\abs{y(a) - y(b)} \geq 2$. $(a,b)$ is always $x$-separated or $y$-separated
(or both) since $\linfnorm{a}{b} \geq 2$.

Given source and target configurations $\fd{s},\fd{t} \in \fdfreesp$, we call an $(\fd{s},\fd{t})$-plan 
$\fdpi: [0,1] \rightarrow \fdfreesp$ \emph{decoupled} if only one robot moves at any time while the other 
robot is \emph{parked} at some point in $\freesp$, and if there is only a finite number of switches
between the moving and parking robots. A decoupled plan can be represented as a finite sequence 
\[
(R_1,\pi_1,p_1), (R_2, \pi_2, p_2), \ldots, (R_k, \pi_k, p_k) ,
\]
where, for each $i$, $(R_i,\pi_i,p_i)$ is called a \emph{move}, with $R_i \in \{\robA,\robB\}$, 
$p_i \in \freesp$, and $\pi_i \subseteq \freept{p_i}$. At such a move, $R_i$ moves along $\pi_i$ 
and the other robot is parked at $p_i$.
The plan $\fdpi$ is the concatenation of Cartesian products
of the form $\pi_i\times \{p_i\}$ or $\{p_i\} \times \pi_i$, depending on which robot is moving
and which is parked.
If $R_1$ is $\robA$ (resp., $\robB$), then we set, for completeness,
$p_0 \assign s_A$ (resp., $p_0 \assign s_B$).
If $R_i \neq R_{i-1}$, then the initial point of 
$\pi_i$ is $p_{i-1}$ and $p_i$ is the final point of $\pi_{i-1}$. Otherwise $R_i=R_{i-1}$ and 
the initial point of $\pi_i$ is the final point of $\pi_{i-1}$ and $p_i = p_{i-1}$. We call a 
move-sequence \emph{minimal} if $R_i \neq R_{i-1}$ for all $1 < i \leq k$. If $R_i = R_{i-1}$, we 
can replace $(R_{i-1}, \pi_{i-1}, p_{i-1}) \circ (R_i, \pi_i, p_i)$ with $(R_i, \pi_{i-1} \ccat \pi_i, p_i)$,
and obtain a shorter sequence (recall that in this case $p_{i-1} = p_i$). Most of the time we will 
be working with a minimal sequence, but sometimes, when we deform a plan, it will be convenient to 
describe a non-minimal sequence, which can then be compressed as above.
For a given plan $\fdpi$, there is a unique minimal move sequence into which $\fdpi$ can be compressed, which we represent 
as $\moveseq{\fdpi}$, and we define $\altern{\fdpi} \assign \abs{\moveseq{\fdpi}}$ to be the 
number of \emph{moves} in $\fdpi$.

\begin{figure}[t]
\centering
\includegraphics[scale=0.90]{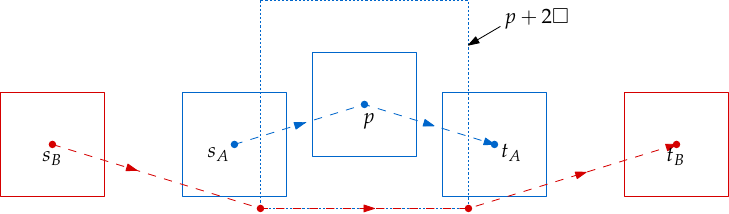}
\caption[Example of an optimal plan.]{The optimal $(\fd{s},\fd{t})$-plan moves $A$ from $s_A$ to $p$, then moves $B$ from $s_B$ to $t_B$, 
and then moves $A$ from $p$ to $t_A$. This example is adapted from \cite{mastersthesisRuizHerrero}.}
\figlab{opt-nontrivial-example}
\end{figure}

For a path $\pi \subset \freesp$ and two values $\lambda,\lambda' \in [0,1]$, $\lambda < \lambda'$, we denote by $\pi(\lambda,\lambda')$ 
the \emph{pathlet} of $\pi$ between times $\lambda$ and $\lambda'$, which itself is a path (with a suitable 
reparameterization). It will be convenient to specify the portion of a path $\pi$ between two points 
$p,p' \in \pi$ using the notation $\plpt{\pi}{p,p'}$.
We define the distance between closest points in a pair of sets using
either the $L_2$-distance or the $L_\infty$-distance. For any pair of subsets 
$\mathsf{X},\mathsf{Y} \subset \Reals^2$, set
\[
\ldist{\mathsf{X}}{\mathsf{Y}}{\ell} \assign \min_{x \in \mathsf{X},y \in \mathsf{Y}} \norm{x-y}_\ell,
\quad\text{for $\ell \in \{2,\infty\}$}.
\]

Lastly, throughout the paper, we refer to the robots $A$ and $B$ by their centers: we say that a robot is ``in'' a region $R$ (at some time $\lambda$) if its center lies in $R$.
Similarly, we say that a robot ``enters'' (resp., ``exits'')
a region $R$ (at some time $\lambda$) its center point is crossing into (resp., out of) $R$.
To describe that the entire robot is contained in $R$, we say $p+\Box \subseteq R$ where $p$ is the placement of its center. 

\mparagraph{Optimal plan for $\envir = \Reals^2$}
Suppose the work environment is the entire plane $\Reals^2$, \ie, there are no obstacles. 
In this case, Esteban \etal~\cite{Esteban2022} proved that an optimal plan is a piecewise-linear 
decoupled plan consisting of at most three moves, and each move consists of at most three line segments.
See \figref{opt-nontrivial-example} for an example. Note that the parking position in some cases 
(such as the one in \figref{opt-nontrivial-example}) is not necessarily near the initial/final placements,
which is one of the challenges in developing an efficient algorithm for computing an optimal plan.

\ifabbrv\else
\begin{figure}[t]
\centering
\includegraphics[keepaspectratio=true]{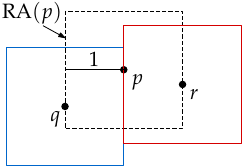}
\caption{Example of a kissing configuration $(q,r) \in \fdfreesp$ with $q,r \in \bd \RA(p)$.}
\figlab{revolving-area}
\end{figure}
\fi
\mparagraph{Revolving area}
A \emph{revolving area} is a unit(-radius) square $p+\Box$, for some $p \in \freesp$, that is contained 
in $\freesp$; we denote it by $\RA(p)$ (\figref{revolving-area}).
For $p_A,p_B \in \bd \RA(p)$ with
$\norm{p_A - p_B}_\infty = 2$, $(p_A,p_B)$ is a \emph{kissing configuration}, and we say that this 
kissing configuration lies in the revolving area $\RA(p)$. In \secref{paths-in-corridor} we 
give useful lemmas regarding revolving areas, which play a key role in deforming an optimal path
into a near-optimal \emph{tame} plan (defined later in \secref{approx-tame}) that is easier to compute.

\mparagraph{Corridor and sanctum}
Intuitively, a corridor $\corridor$ is a (narrow) trapezoid in $\freesp$ bounded by two edges of $\freesp$,
so that if one robot is parked inside $\corridor$, the other one cannot pass around it (within $\corridor$). This implies that when both robots are in the same corridor, their motions are constrained in ways that we
will later explore. We now give a formal definition.
Let $e_i,e_j$ be a pair of edges of $\freesp$ that 
support an axis-aligned square (of any size) contained in $\freesp$, 
\ie, there exists an axis-aligned square $S \subset \freesp$ such that $e_i$ (resp., $e_j$) touches a vertex 
of $S$, say $v_i$ (resp., $v_j$), but does not intersect $\Int(S)$. Let $u_{ij} \in [0,\pi)$ be a direction normal to the segment $v_iv_j$; $u_{ij} = k \pi/4$ for some $0 \leq k \leq 3$.
\footnote{If $v_i$ or $v_j$ is not 
unique, \ie, when $e_i$ and $e_j$ are axis-aligned, we can choose $v_i$ or $v_j$ (or both) so that $u_{ij} \in \{0,\pi/2\}$.}
A \emph{corridor} $\corridor$ bounded by $e_i,e_j$ is a trapezoid such that (i) two of the edges of 
$\corridor$ are portions of $e_i,e_j$, called \emph{blockers}; (ii) the other edges of $\corridor$, 
called \emph{portals}, are normal to the direction $u_{ij}$;
(iii) the $L_\infty$-length of each portal (\ie, the $L_\infty$-distance between its endpoints) is 
at most $2$;
and 
(iv) no vertex of $\freesp$ lies in the interior of $\corridor$.
See \figref{corridor}.
We refer to $u_{ij}$ as the direction of the corridor. The following lemma directly follows from condition (iii).
\ifabbrv\else
\begin{figure}
\centering
\includegraphics[scale=0.90]{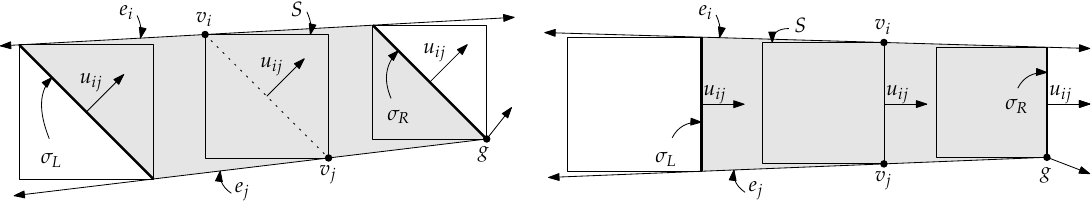}
\caption[Examples of corridors.]{Two examples of corridors $\corridor$ (shaded) with blockers $e_i,e_j$ that contains squares
$S \subset \freesp$ with radii strictly less than $2$ since their centers lie in the interiors of the corridors.
The left (resp., right) corridor has direction vector $u_{ij}$ with angle $\pi/4$ (resp., $0$).
Both examples are maximal since the portals $\sigma_L$ have $L_\infty$-length $2$ and $\sigma_R$ contain a vertex $g$ of $\freesp$.
from portals $\sigma^L,\sigma^R$, respectively (not drawn to scale).}
\figlab{corridor}
\end{figure}

\begin{figure}[t]
\centering
\includegraphics[scale=0.90]{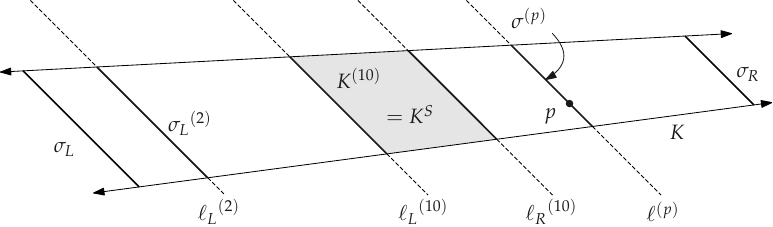}
\caption{Illustrations of various portal-parallel lines supporting segments in $\corridor$, and the sanctum $\sanctum{\corridor}$ of $\corridor$.}
\figlab{shifted-seg-example}
\end{figure}
\fi

\begin{lemma}
\lemlab{short-corridor-segments}
Let $\corridor$ be a corridor with direction vector $u$. For any segment $vw \subset \Int(\corridor)$ normal to $u$, $\linfnorm{v}{w} < 2$. Furthermore, $\ltnorm{v}{w} < 2$ if $u$ is axis-parallel, otherwise $\ltnorm{v}{w} < 2\sqrt{2}$.
\end{lemma}
A corridor $\corridor$ is \emph{maximal} if there is no other corridor that contains $\corridor$.
If $\corridor$ is maximal, condition (iv) is ``tight'' for at least one portal $\sigma$ of $\corridor$
in the sense that there is a vertex of $\freesp$ (not necessarily an endpoint of $e_i$ or $e_j$) on $\sigma$. In particular,
there is a convex vertex of $\freesp$ on the shorter portal of $\corridor$; if both portals have the same length, both contain such vertices.

Let $\corridors$ be the set of all maximal corridors in $\freesp$.
We charge each corridor $\corridor \in \corridors$ to a vertex of $\freesp$ on its shorter portal. The conditions are easily seen to imply that the corridors of $\corridors$ are pairwise disjoint, from which it follows that any vertex of $\freesp$ is charged at most $O(1)$ times. There are $O(n)$ vertices of $\freesp$, which implies $\abs{\corridors} = O(n)$.

Let $\ell_L,\ell_R$ be the lines supporting the portals $\sigma_L,\sigma_R$ of $\corridor$, and let $\corlen{\corridor}$ be the $L_\infty$-distance between $\ell_L,\ell_R$.
Let $u_L$ (resp., $u_R$) be the inner normal of $\sigma_L$ (resp., $\sigma_R$), \ie, pointing toward the interior of 
$\corridor$; $u_L = -u_R$. 
For $D = L,R$ and any value $\tau \geq 0$, let $\shiftseg{\ell_D}{\tau}$ be the line $\ell_D$ shifted in direction $u_D$ at $L_\infty$-distance $\tau$ from $\ell_D$, let $\shiftseg{\sigma_D}{\tau}$ be the segment $\corridor \cap \shiftseg{\ell_D}{\tau}$, and let $\shiftseg{\corridor}{\tau} \subseteq \corridor$ be the (possibly empty) trapezoid bounded by the blockers of $\corridor$ and segments $\shiftseg{\sigma_L}{\tau}, \shiftseg{\sigma_R}{\tau}$. (We assume here that $\tau$ is sufficiently small so as to guarantee the shifts from $\ell_L$ to $\ell^{(\tau)}$ and from $\ell_R$ to $\ell_R^{(\tau)}$ do not collide.)
Note that $\corridor = \shiftseg{\corridor}{0}$. 
Similarly, we define portal-parallel lines and segments by points that they contain: For any point $p \in \corridor$, 
let $\pointseg{\ell}{p}$ be the line normal to $u_L$ (and $u_R$) containing $p$, and 
let $\pointseg{\sigma}{p} \assign \corridor \cap \pointseg{\ell}{p}$.
For any corridor $\corridor \in \corridors$ with $\corlen{\corridor} \geq 20$, we define its \emph{sanctum} to be $\sanctum{\corridor} \assign \shiftseg{\corridor}{10} \subset \corridor$.
See \figref{shifted-seg-example}. A corridor $\corridor \in \corridors$ with $\corlen{\corridor} < 20$ has an empty sanctum.
The following two lemmas capture the essence of a corridor.
\begin{lemma}
\lemlab{no-corridor-crossings}
Let $\corridor \in \corridors$ be a maximal corridor, and let $u$ its direction, \ie one of the unit vectors 
normal to the portals of $\corridor$. Let $I$ be a time interval in a plan $\fdpi$ 
of $\robA$ and $\robB$, during which both robots are in $\corridor$, 
\ie, $\pi_A(\lambda), \pi_B(\lambda) \in \corridor$ for all $\lambda \in I$. 
Then the sign of $g(\lambda) \assign \langle \pi_A(\lambda) - \pi_B(\lambda), u\rangle$ is the same for all $\lambda \in I$, where $\langle \cdot \rangle$ is the inner product.
\end{lemma}

\begin{proof}
Suppose to the contrary that there exist two time instances $\lambda_1,\lambda_2 \in I$, with 
$\lambda_1 < \lambda_2$, such that $g(\lambda_1) < 0$ and $g(\lambda_2) > 0$ (or the other way around). 
Since $\pi_A,\pi_B$ are continuous functions, there exists $\lambda_0 \in (\lambda_1,\lambda_2)$ with 
$g(\lambda_0)=0$. But then $\pi_A(\lambda_0)$ and $\pi_B(\lambda_0)$ lie on a segment parallel to the 
portals of $\corridor$ and thus $\linfnorm{\pi_A(\lambda_0)}{\pi_B(\lambda_0)} < 2$, which means that
the robots intersect at these placements, contradicting the assumption that $\fdpi$ is a feasible plan. 
Hence, $g(\cdot)$ has the same sign over the entire interval $I$. See \figref{same-sign-corridor}.
\end{proof}

\begin{figure}
\centering
\includegraphics[scale=0.90]{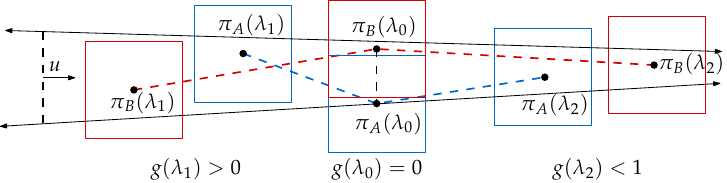}
\caption{Illustration of the proof of \lemref{no-corridor-crossings}.}
\figlab{same-sign-corridor}
\end{figure}

The following lemma describes a crucial relationship between revolving areas and corridors.
\begin{lemma}
\lemlab{ra-exists}
Suppose $p \in \freesp$ is a point such that 
	$p$ does not lie in any corridor of $\corridors$ and $\linfd{p}{\enVerts} \geq 1$, where $\enVerts$ denotes the set of vertices of $\freesp$ plus $\{s_A,s_B,t_A,t_B\}$ and the vertices of all maximal corridors in $\corridors$, \ie, the endpoints of their portals. 
Then there is a revolving area $q + \Box \subseteq \freesp$, for some $q\in\freesp$,
	that contains $p$.
\end{lemma}

\begin{proof}
Let $S \assign p + r\Box$ be the largest axis-aligned square in $\freesp$ centered at $p$, where $r \geq 0$. If $r \geq 1$, then
$r + \Box \subset \freesp$ and the claim holds, so suppose $r < 1$. $S$ is supported by at least two edges $e_i,e_j$ of $\freesp$, otherwise it could be
expanded. Let $\sigma$ be the segment connecting the vertices $v_i,v_j$ of $S$ on edges $e_i,e_j$, respectively.
The $L_\infty$-distance between $v_i,v_j$ is $2r < 2$ by definition, and we have, for any point $q \in \sigma$, $$\linfd{q}{\enVerts} \geq \linfd{p}{\enVerts} - r > 0.$$
Then it is easy to verify that $\sigma$ is itself a corridor, so there is a maximal corridor $\corridor \in \corridors$ such that $\sigma \subseteq \corridor$.
So $p \in \corridor$, which is a contradiction.
\end{proof}

\section{Well-structured Optimal Plans}
\seclab{optimal-plans}

We present a sequence of transformations for optimal plans, which leads to the existence 
of an optimal plan with certain desirable properties.
Using the easily established fact that $\fdfreesp$ is polyhedral, it can be shown that an optimal plan is piecewise linear with its breakpoints lying on $2$-faces of $\bd\fdfreesp$.
We show that there always exists a piecewise-linear, decoupled plan such that a robot is never parked in the 
sanctum of a corridor, and the moving robot \emph{kisses} the parked robot in each move, except possibly in the first and
the last moves. 

We first observe that each facet (three-dimensional face) of $\bd \fdfreesp$ corresponds to a maximal 
connected set of placements at which some vertex (resp., edge) of one of the robots touches some edge 
(resp., vertex) of $\bd \envir$ or of the other robot. This implies that each connected component of 
$\bd \fdfreesp$ is a polyhedral region in $\Reals^4$. The distance between two points $a,b\in \fdfreesp$ 
is the sum of the Euclidean length of the projections of $b-a$ onto the 2-planes formed by the first
and the last pairs of coordinates, so it is the $L_1$-distance of two $L_2$-distances. Still, 
we claim that an optimal path (in $\fdfreesp$) must be piecewise linear, with bends only
at 2-faces (or faces of lower dimension) of $\bd\fdfreesp$. This follows since both the $L_2$ and
$L_1$-distances satisfy the triangle inequality, and since paths that bend at the relative interior 
of some 3-face of $\fdfreesp$ can be shortened.
Hence, from now on we only consider piecewise-linear plans.

\subsection{Decoupled optimal plans}
\subseclab{decoupled}\mynewline
\noindent We begin by proving that there always exists an optimal (piecewise-linear) plan that is decoupled, 
\ie, only one robot moves at any given time. Such \emph{decoupled} plans are desirable, as during the 
motion of the moving robot, the parked robot can be treated as an additional obstacle that 
is part of the environment. 
Thus, given the start and target placements, $s$ and $t$, of the moving 
robot, at some single move in the plan, and the position $p$ of the parked robot, the optimal motion 
for the moving robot is the shortest path from $s$ to $t$ in $\freept{p}$.

\begin{lemma}
\lemlab{decouple}
Given reachable configurations $\fd{s},\fd{t} \in \fdfreesp$, there is always a piecewise-linear, decoupled, optimal $(\fd{s},\fd{t})$-plan.
\end{lemma}

\ifabbrv
\else
\begin{figure}
\centering
\includegraphics[scale=0.90]{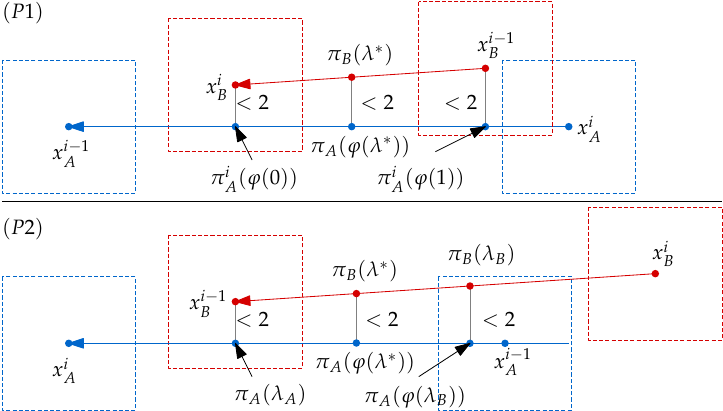}
\caption[Illustrations of the proof of \lemref{decouple}.]{Illustrations of the proof of \lemref{decouple}. The top illustrates an example where predicate (P1) holds, 
and the bottom illustrates an example where predicate (P2) holds, with $\varphi(\lambda_B) < \lambda_B$. The blue 
and red lines are the respective projected paths $\pi_B$ and $\pi_A$, and the dotted blue and red squares are 
axis-parallel squares of radius $2$, \ie, copies of $2\Box$, centered at the four endpoints of the two projected
paths.}
\figlab{decouple-predicates}
\end{figure}
\fi

\ifabbrv
\else
\begin{proof}
Let $\fdpi = \langle s=x^0, x^1, \ldots, x^R = t\rangle$ be a piecewise-linear optimal $(\fd{s},\fd{t})$-plan
in $\fdfreesp$, where $\fdpi^i = x^{i-1}x^i$, for $1 \leq i \leq k$, is a line segment in $\fdfreesp$. 
Let $\pi^i_A$ (resp., $\pi^i_B$) be the line segment $x^{i-1}_Ax^i_A$ (resp., $x^{i-1}_Bx^i_B$) in $\freesp$,
along which $A$ (resp., $B$) moves during plan $\fdpi^i$, \ie, it is the projection of $\fdpi^i$ onto the $A$-plane 
(resp., $B$-plane). We also use $\fdpi^i: [0,1] \rightarrow \fdfreesp$ to denote the (linear) parameterization 
of the segment $x^{i-1}x^i$, and similarly define the projected parameterizations 
$\pi^i_A,\pi^i_B: [0,1] \rightarrow \freesp$, \ie, $\fdpi^i(t) = (\pi_A^i(t), \pi_B^i(t))$, for $t\in [0,1]$.
We claim that we can either move $\robA$ first along $\pi_A^i$ while $\robB$ is parked at $x^{i-1}_B$ followed 
by moving $\robB$ along $\pi^i_B$ while $\robA$ is parked at $x_A^i$ or vice-versa. We note that $\robA$ can 
be moved first followed by $\robB$ if $\pi_A^i \subseteq \freept{x^{i-1}_B}$ and $\pi_B^i\subseteq \freept{x^i_A}$. 
Similarly, $\robB$ can be moved first followed by $\robA$ if $\pi_A^i \subseteq \freept{x^i_B}$ and 
$\pi_B^i\subseteq \freept{x^{i-1}_A}$.

Suppose to the contrary that a decoupled plan does not exist for $\fdpi^i$, \ie, the predicate 
\[
\Bigl(\pi_A^i \subseteq \freept{x_B^{i-1}} \land \pi_B^i \subseteq \freept{x_A^i}\Bigr) \lor 
\Bigl(\pi_A^i \subseteq \freept{x^i_B} \land \pi_B^i\subseteq \freept{x^{i-1}_A}\Bigr) 
\]
is not true. Then at least one of the following four predicates must hold:

\begin{itemize}
    \item[(P1)] $\pi_A^i \not\subseteq \freept{x_B^{i-1}} \land \pi_A^i \not\subseteq \freept{x_B^i}$.
    \item[(P2)] $\pi_A^i \not\subseteq \freept{x_B^{i-1}} \land \pi_B^i \not\subseteq \freept{x_A^{i-1}}$.
    \item[(P3)] $\pi_B^i \not\subseteq \freept{x_A^i} \land \pi_B^i \not\subseteq \freept{x_A^{i-1}}$.
    \item[(P4)] $\pi_B^i \not\subseteq \freept{x_A^i} \land \pi_A^i \not\subseteq \freept{x_B^i}$.
\end{itemize}

In each case we show the existence of a time $\lambda^* \in [0,1]$ for which 
\[
\norm{\pi_A^i(\lambda^*) - \pi_B^i(\lambda^*)}_\infty < 2 , 
\]
which would imply that $\fdpi^i$ is not a feasible plan, and thereby yield the desired contradiction. 
First, consider (P1). Since $\fdpi^i \subseteq \fdfreesp$, $\pi_A^i,\pi_B^i \subseteq \freesp$.
Therefore (P1) implies that 
\[
\pi_A^i \cap \Int(x_B^{i-1}+2\Box) \neq \varnothing \quad\text{and}\quad
\pi_A^i \cap \Int(x_B^i+2\Box) \neq \varnothing .
\]
Then there exist $\lambda_0,\lambda_1 \in [0,1]$ such that 
\[
\norm{\pi_A^i(\lambda_0) - \pi_B^i(0)}_\infty,\; \norm{\pi_A^i(\lambda_1) - \pi_B^i(1)}_\infty < 2
\]
(recall that $x_B^{i-1} = \pi_B^i(0), x_B^i = \pi_B^{i-1}(1)$). For a value $\lambda \in [0,1]$, let 
$\varphi(\lambda) \in [0,1]$ be such that $\pi_A^i(\varphi(\lambda))$ is the point closest to 
$\pi_B^i(\lambda)$ on the segment $x_B^{i-1}x_B^i$. We have
\[
\norm{\pi_A^i(\varphi(0)) - \pi_B^i(0)}_\infty,\; \norm{\pi_A^i(\varphi(1)) - \pi_B^i(1)}_\infty < 2 , 
\]
which holds because of the existence of $\lambda_0$, $\lambda_1$ above. This means that
\begin{align*}
| x(\pi_A^i(\varphi(0))) - x(\pi_B^i(0)) | & < 2 , \\ 
| y(\pi_A^i(\varphi(0))) - y(\pi_B^i(0)) | & < 2 , \\ 
| x(\pi_A^i(\varphi(1))) - x(\pi_B^i(1)) | & < 2 , \\
| y(\pi_A^i(\varphi(1))) - y(\pi_B^i(1)) | & < 2 .
\end{align*}
Recalling that $\pi_A$, $\pi_B$ are line segments, this implies that, for any $\alpha\in [0,1]$, we also have
\begin{align*}
| x(\pi_A^i((1-\alpha)\varphi(0)+\alpha\varphi(1))) - x(\pi_B^i(\alpha)) | & < 2 , \\ 
| y(\pi_A^i((1-\alpha)\varphi(0)+\alpha\varphi(1))) - y(\pi_B^i(\alpha)) | & < 2 .
\end{align*}
That is, $\norm{\pi_A^i((1-\alpha)\varphi(0)+\alpha\varphi(1)) - \pi_B^i(\alpha)} < 2$. This in turn implies
that $\norm{\pi_A^i(\varphi(\lambda)) - \pi_B^i(\lambda)}_\infty < 2$ for all $\lambda \in [0,1]$.
Since $\varphi$ is a continuous function, there exists a $\lambda^* \in [0,1]$ such that 
$\varphi(\lambda^*) = \lambda^*$. But then 
\[
\norm{\pi_A^i(\lambda^*) - \pi_B^i(\lambda^*)}_\infty = 
\norm{\pi_A^i(\varphi(\lambda^*)) - \pi_B^i(\lambda^*)}_\infty < 2 ,
\]
which contradicts the assumption that $\fdpi^i \subseteq \fdfreesp$. Hence (P1) does not hold.

Next, suppose that (P2) holds. Then there exist $\lambda_A,\lambda_B \in (0,1)$ such that 
\[
\norm{\pi_A^i(\lambda_A) - \pi_B^i(0)}_\infty,\; \norm{\pi_B^i(\lambda_B) - \pi_A^i(0)}_\infty < 2 . 
\]
Without loss of generality, assume that $\lambda_A \geq \lambda_B$. For a value $\lambda \in [0,\lambda_B]$, 
let $\varphi(\lambda) \in [0,\lambda_A]$ be the parameter of the closest point to $\pi_B^i(\lambda)$ on the 
segment $\pi_A^i(0)\pi_A^i(\lambda_A)$. As above, 
\[
\norm{\pi_A^i(\varphi(\lambda) - \pi_B^i(\lambda)}_\infty < 2 \quad\text{for all $\lambda \in [0,\lambda_B]$} .
\]
If $\varphi(\lambda_B) \geq \lambda_B$, then 
\[
\norm{\pi_A^i(\lambda_B) - \pi_B^i(\lambda_B)}_\infty \leq 
\max\{\norm{\pi_A^{i-1}(0) - \pi_B^i(\lambda_B)}_\infty, 
\norm{\pi_A^i(\varphi(\lambda_B)) - \pi_B^i(\lambda_B)}_\infty\} < 2 , 
\]
which contradicts the fact that $\fdpi^i(\lambda_B) = (\pi_A^i(\lambda_B), \pi_B^i(\lambda_B)) \in \fdfreesp$.
Hence, assume that $\varphi(\lambda_B) < \lambda_B$. Since $\varphi(0) > 0$, there exists a value 
$\lambda^* \in (0,\lambda_B]$ such that $\varphi(\lambda^*) = \lambda^*$, and we obtain the same 
contradiction as above. Hence (P2) does not hold.

Predicates (P3) and (P4) are analogous to (P1) and (P2), respectively, by either switching the roles of $A$ 
and $B$ or by reversing the time direction. We thus conclude that none of (P1)--(P4) holds, implying that 
there is a decoupled plan for $\fdpi^i$. Repeating this argument for all $i \leq k$, and observing that the
endpoints of each $\fdpi$ do not change by the transformation, we conclude that there 
is a decoupled, piecewise-linear optimal plan.
\end{proof}
\fi

\subsection{Kissing plans}
\begin{figure}
\centering
\includegraphics[scale=.90]{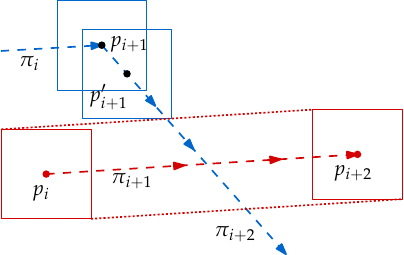}
\caption[Example of three moves, $\pi_{i}$ and $\pi_{i+2}$ of robot $A$ and $\pi_{i+1}$ of robot $B$, in a plan $\fdpi$.]{Example of three moves, $\pi_{i}$ and $\pi_{i+2}$ of robot $A$ and $\pi_{i+1}$ of robot $B$, in a plan $\fdpi$. By modifying $\pi_i,\pi_{i+2}$ so that $A$ parks at $p_{i+1}'$ instead of $p_{i+1}$, $B$ kisses $A$ during move $\pi_i$.}
\figlab{small-kissing}
\end{figure}
\mynewline
\noindent We call a piecewise-linear decoupled plan $\fdpi$ a \emph{kissing plan} if the robots kiss on all but possibly 
the first and the last moves. Formally, let $\moveseq{\fdpi} = (R_1,\pi_1,p_1), \ldots, (R_k,\pi_k,p_k)$ be 
the move sequence of $\fdpi$. Then $\fdpi$ is a kissing plan if, for all $1 < i < k$, there exists 
a point $q_i \in \pi_i$ such that $(p_i,q_i)$ is a kissing configuration. We show that a decoupled plan can be converted into a kissing plan, without changing the images of the paths traveled by $A$ and $B$ in the plan, by reducing the number of moves and adjusting the parking places (\figref{small-kissing}). We obtain the following:

\begin{lemma}
\lemlab{kissing}
Let $\fdpi$ be a piecewise-linear, decoupled, optimal plan with the minimum number of moves. 
There exists a piecewise-linear, decoupled, kissing, optimal plan $\fdpi'$ with the same number of moves, 
such that the first move is made by the same robot as in $\fdpi$, and the pathlet of the first 
move in $\fdpi'$ contains that of $\fdpi$.
\end{lemma}

\begin{proof}
The proof is by induction on $k$. Let $\moveseq{\fdpi} = (R_1,\pi_1,p_1), \ldots, (R_k,\pi_k,p_k)$. If $k=2$, the claim holds trivially, that is,
vacuously, so assume $k > 2$. Without loss of generality, $\robA$ moves first, \ie, $R_1 = \robA$. Then 
$(p_1=s_B), p_3, p_5, \ldots$ are the parking placements of $\robB$; $(p_0 = s_A), p_2,p_4,\ldots$ are the 
parking placements of $\robA$; $p_k = t_A, p_{k+1} = t_B$ if $k$ is odd, and $p_k = t_B, p_{k+1} = t_A$ 
if $k$ is even; $\pi_1 \assign \plpt{\pi_A}{s_A,p_2}$ is the motion of $\robA$ in the first move and 
$\pi_2 \assign \plpt{\pi_B}{s_B,p_3}$ is the motion of $\robB$ in its first move. There are two cases to consider.

\begin{enumerate}
\item 
If $(\pi_3 \oplus \Box) \cap (\pi_2 \oplus \Box) = \varnothing$ then 
\[
\fdpi' \assign 
\begin{cases}
(A,\pi_1 \| \pi_3,p_1), (B, \pi_2 \| \pi_4, p_4), (R_5, \pi_5, p_5), \ldots, (R_k, \pi_k, p_k)
& \text{if $k > 3$} \\
(A, \pi_1 \| \pi_3, s_B), (B, \pi_2, t_A) & \text{if $k = 3$} 
\end{cases}
\]
is a decoupled, optimal plan with fewer than $k$ moves, which contradicts the assumption that
$\fdpi$ has the fewest moves among all decoupled, optimal plans.

\item 
If $(\pi_3 \oplus \Box) \cap (\pi_2 \oplus \Box) \neq \varnothing$, let $p'$ be the first 
point reached on $\pi_3$ such that $(p' + \Box) \cap (\pi_2 \oplus \Box) \neq \varnothing$; note that 
$p'$ may be $p_2$. By the choice of $p'$, the interior of $p' + \Box$ 
is disjoint from $\pi_2 \oplus \Box$, so $B$ kisses $A$ at that placement when moving along $\pi_2$.
Define $\pi_{3<} \assign \plpt{\pi_3}{p_2,p'}$ and $\pi_{3>} \assign \plpt{\pi_3}{p',p_4}$. Again, the 
choice of $p'$ also implies that $\pi_{3<} \oplus \Box$ is interior disjoint from $\pi_2 \oplus \Box$. Then
\[
\fdpi' \assign (A,\pi_1 \| \pi_{3<}, s_B), (B, \pi_2, p'), (A, \pi_{3>}, p_3), 
(R_4, \pi_4, p_4), \ldots, (R_k, \pi_k, p_k)
\]
is a decoupled, optimal $(\fd{s},\fd{t})$-plan in which $\robB$ kisses $\robA$, parked at $p'$, as 
it moves along $\pi_2$.

Set $\fd{s'} \assign (p',s_B)$. Let $\fdpi_0'$ be the decoupled $(\fd{s'},\fd{t})$-plan composed of all 
but the first move of $\fdpi'$. Then $\altern{\fdpi_0'} = \altern{\fdpi'}-1 = \altern{\fdpi}-1$. 
Furthermore, $\fdpi_0'$ is a decoupled, optimal $(\fd{s'},\fd{t})$-plan. We apply the induction 
hypothesis to $\pi_0'$ to obtain a decoupled, kissing, optimal $(\fd{s}',\fd{t})$-plan $\fdpi_0''$ satisfying 
the lemma, with $\robB$ making the first move $(B, \pi_2', p')$. Since the lemma guarantees
that $\pi_2 \subset \pi_2'$, $B$ kisses $A$ (parked at $p'$) during the first move of $\fdpi''_0$. 
Set $\fdpi'' \assign (A,\pi_1 \| \pi_{3<}, s_A) \| \fdpi''_0$. Then the robots kiss on all moves of
$\fdpi''$ except possibly in the first and the last moves. Furthermore 
\[
\altern{\fdpi''} = \altern{\fdpi''_0} + 1 = \altern{\fdpi'_0} + 1 = \altern{\fdpi} ,
\]
and $\pi_1 \subseteq \pi_1 \| \pi_{3<}$. Hence $\fdpi''$ satisfies the lemma, which establishes
the induction step and thus completes the proof of the lemma.
\end{enumerate}
\end{proof}

\subsection{Bounding alternations}
\subseclab{bounding-alternations}\mynewline

\noindent In a sequence of lemmas, we show that for any $\fd{s},\fd{t} \in \fdfreesp$, there exists a decoupled, kissing, optimal $(\fd{s},\fd{t})$-plan $\fdpi$ with $\altern{\fdpi} = O(\plancost{\fdpi}+1)$.
We begin with a simple observation whose proof is omitted.

\begin{lemma}
\lemlab{segment-connected}
Let $e$ be a horizontal or vertical segment of length at most 2. Then $e \cap \freesp$ is a connected (possibly empty) interval.
\end{lemma}

For any region $\nabla \subseteq \freesp$ and any two points $p,q \in \nabla$, let $\georegion{p}{q}{\nabla}$ be the length of the shortest $(p,q)$-path in $\nabla \cap \freesp$. Note that $\geodesic{p}{q} = \georegion{p}{q}{\freesp}$.

\begin{lemma}
\lemlab{simple-component}
Let $S$ be any axis-aligned unit-radius square. (i) $S \cap \freesp$ is composed of $xy$-monotone components (without holes). (ii) At most two components intersect $\bd S$. (iii) For any $p,q$ that lie in the interior of a common component of $S \cap \freesp$, there exists an $xy$-monotone $(p,q)$-path $P$ such that $\abs{P} = \georegion{p}{q}{S} = \geodesic{p}{q}$.
\end{lemma}

\ifabbrv
\else
\begin{proof}$ $
\begin{enumerate}[(i)]
\item The claim is immediate from \lemref{segment-connected}.
\item By \lemref{segment-connected}, at most one connected component of $S \cap \freesp$ intersects each edge of $S$. So if three connected components of $S \cap \freesp$ intersect $\bd S$, two opposite edges, say, horizontal edges of $S$ intersect different connected components of $S \cap \freesp$. Let $C_1$ (resp., $C_2$) with $C_1 \neq C_2$ be the connected component of $S \cap \freesp$ that intersects the bottom (resp., top) edge of $S$, and let $a_1,b_1$ (resp., $a_2,b_2$) be the left and right endpoints of $C_1$ (resp., $C_2$) with the bottom (resp., top) edge. Without loss of generality assume that $x(a_1) < x(a_2)$. Then by \lemref{segment-connected}, $x(a_1) < x(b_1) < x(a_2) < x(b_2)$. Let $C_3$ be the third component of $S \cap \freesp$ that intersects, say, the left edge of $S$, and let $a_3,b_3$ be the intersection segment of $C_3 \cap \bd S$, with $y(a_3) < y(b_3)$. Since $C_1$ does not intersect the top edge of $S$, the highest point of $C_1$, denoted by $q$, lies inside $S$, \ie, $y(a_1) = y(b_1) < y(q) < y(a_2) = y(b_2)$. Furthermore, by \lemref{segment-connected}, $x(a_3) < x(q) < x(a_2)$ and $y(a_3) > y(q)$. Finally, since $q \in \bd \freesp$, there is a point $\widehat{q} \in \bd \envir$ such that $\linfnorm{q}{\widehat{q}} = 1$ and $y(\widehat{q}) = y(q) + 1$. Note that $a_2 \notin \Int(\widehat{q} + \Box)$. Since $y(a_2) - y(q) < 2$, $y(a_2) - y(\widehat{q}) < 1$ implying that $x(a_2) - x(\widehat{q}) \geq 1$. On the other hand, $x(\widehat{q}) \geq x(q) - 1$. Putting these together, we obtain that $y(\widehat{q}) - y(a_3) < y(\widehat{q}) - y(q) = 1$, $y(a_3) - y(\widehat{q}) < y(a_2) - y(\widehat{q}) < 1$, $x(a_3) - x(\widehat{q}) < x(q) - x(\widehat{q}) \leq 1$. $x(\widehat{q}) - x(a_3) \leq x(a_3) - 1 - x(a_3) < 1$ since $x(a_2) - x(a_3) < 2$. In other words, $\abs{x(a_3)-x(\widehat{q})},\abs{y(a_3)-y(\widehat{q})} < 1$, so $\linfnorm{a_2}{q} < 1$ and hence $a_3 \notin \freesp$, contradicting the assumption that $a_3 \in \bd \freesp$. Hence, $c_3$ does not exist. A similar argument shows that a third component of $S \cap \freesp$ cannot intersect the right edge of $S$.

\item Let $C$ be a connected component of $S \cap \freesp$. Since $C$ is $xy$-monotone, for any two points $a,b \in C$, the shortest path from $a$ to $b$ within $C$ is $xy$-monotone, implying that $\georegion{a}{b}{S} \leq \lonorm{a}{b}$. Let $p,q \in C$ be two points such that $\geodesic{p}{q} < \georegion{p}{q}{S}$. Then the shortest path, denoted by $\psi$, from $p$ to $q$ in $\freesp$ leaves $S$. Let $a,b$ be two consecutive intersection points of $\psi$ with $\bd S$ where $\psi$ crosses $\bd S$, \ie, $a,b \in \bd S$ and $\psi(a,b) \cap S = \{a,b\}$. But $\psi(a,b)$ is longer than following $\bd S$ from $a$ to $b$ (along the shorter of the two portions of $\bd S$), which implies that $\pathcost{\psi(a,b)} > \lonorm{a}{b}$. On the other hand, $\georegion{a}{b}{S} \leq \lonorm{a}{b}$, contradicting that $\psi$ is the shortest path from $p$ to $q$ in $\freesp$. Hence, $\georegion{p}{q}{S} = \geodesic{p}{q}$.
\end{enumerate}
\end{proof}
\fi

The next lemma shows that there is a simple optimal motion between configurations as long as they are sufficiently close and both $x$-separated or both $y$-separated.

\begin{lemma}
\lemlab{same-cells-plan}
Let $Q_A,Q_B$ be axis-aligned unit-radius squares. For $\fd{s} = (s_A,s_B), \fd{t} = (t_A,t_B) \in \fdfreesp$ such that $s_A,t_A$ (resp., $s_B,t_B$) lie in a common component of $\Int(Q_A) \cap \freesp$ (resp., $\Int(Q_B) \cap \freesp$) and $\fd{s}$ and $\fd{t}$ are both $x$-separated or both $y$-separated, there exists a (trivially kissing) plan $\fdpi$ with $\plancost{\fdpi} = \geodesic{s_A}{t_A} + \geodesic{s_B}{t_B}$ and $\altern{\fdpi} \leq 2$.
\end{lemma}

\begin{figure}
\centering
\includegraphics[scale=0.90]{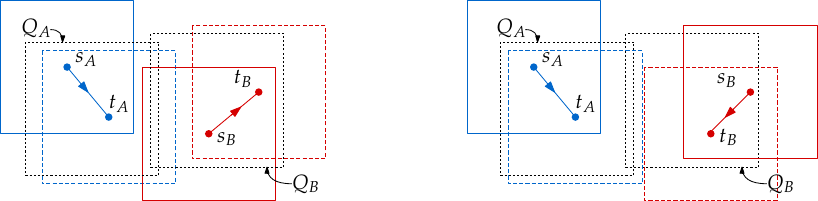}
\caption[Examples of $x$-separated configurations $\fd{s},\fd{t}$ and squares $Q_A,Q_B$ that satisfy \lemref{same-cells-plan}.]{Examples of $x$-separated configurations $\fd{s},\fd{t}$ and squares $Q_A,Q_B$ that satisfy \lemref{same-cells-plan}, $s_A$ is left of $s_B$. (left) $s_A,s_B$ are both left of their respective target placements, $t_A,t_B$. $B$ moves first from $s_B$ to $t_B$ and then $A$ moves from $s_A$ to $t_A$. (right) $s_A$ is left of $t_A$ but $s_B$ is right of $t_B$, so both $2$-move plans are feasible. }
\figlab{same-cells-plan}
\end{figure}

\begin{proof}
Without loss of generality, the configurations are $x$-separated. Using standard transformations as necessary, we can assume $x(s_A) - x(s_B) \geq 2$. Then $x(s_A) - 2 < x(t_A) < x(s_A) + 2$ (resp., $x(s_B) - 2 < x(t_B) < x(s_B) + 2$) since $s_A,t_A$ (resp., $s_B,t_B$) lie in the interior of $Q_A$ (resp., $Q_B$). $(t_A,t_B)$ is $x$-separated so $\abs{x(t_A)-x(t_B)} \geq 2$. If $x(t_A) - x(t_B) \leq - 2$ then $x(t_A) < x(s_B) \leq x(s_A) - 2$, which is a contradiction. Hence $x(t_A) - x(t_B) \geq 2$.
Let $P_A$ be the $xy$-monotone $(s_A,t_A)$-path in $Q_A \cap \freesp$ and let $P_B$ be the $xy$-monotone $(s_B,t_B)$-path in $Q_B \cap \freesp$ from \lemref{simple-component}. There are two cases.

First, suppose $x(s_A)-x(t_A)$ and $x(s_B)-x(t_B)$ are zero or their signs are the same, say, non-negative for concreteness. See \figref{same-cells-plan}(left). Then $P_B$ lies to the right of line $x = s_A+2$ and hence $P_B \subset \freept{s_A}$. Similarly, $P_A$ lies to the left of line $x = t_B-2$ and hence $P_A \subset \freept{t_B}$. 

Otherwise, $x(s_A)-x(t_A)$ and $x(s_B)-x(t_B)$ are non-zero and their signs are different; for concreteness, suppose $x(s_A)-x(t_A) < 0 < x(s_B)-x(t_B)$. See \figref{same-cells-plan}(right). Then $x(s_A) < x(t_A) \leq x(t_B)-2 < x(t_A)-2$. See \figref{same-cells-plan}. Then $P_B$ lies to the right of line $x=x(t_A)+2$, and hence right of line $x=x(s_A)+2$, so $P_B \subset \freept{s_A}$. Similarly, $P_A$ lies to the left of line $x = x(t_B)-2$ so $P_A \subset \freept{t_B}$.

Thus, in either case, the desired plan $\fdpi$ is to first move $B$ along $P_B$ while $A$ is parked at $s_A$, then move $A$ along $P_A$ while $B$ is parked at $t_B$, which is trivially kissing since it has at most two moves. The other cases are symmetric.
\end{proof}

The previous lemma allows us to shortcut kissing plans and to use a packing argument to establish a useful 
upper bound on the number of moves in an optimal plan.

\begin{lemma}
\lemlab{few-parking}
Given reachable configurations $\fd{s},\fd{t} \in \fdfreesp$, there exists a decoupled, kissing, optimal $(\fd{s},\fd{t})$-plan $\fdpi = (\pi_A,\pi_B)$ with $\altern{\fdpi} \leq c(\min\{\pathcost{\pi_A},\pathcost{\pi_B}\}+1)$, for some global constant $c \geq 1$.
\end{lemma}

\begin{proof}
Without loss of generality, assume $\pathcost{\pi_A} \leq \pathcost{\pi_B}$. Let $\grid$ be the axis-aligned uniform grid with square cells of radius $1$ such that all parking places lie in the interior of grid cells and $\fdpi$ does not pass through a vertex of $\grid$. Let $\gridcells \subset \grid$ be the set of grid cells that contain at least one parking place of $A$. It is easily seen that $\abs{\gridcells} \leq 4 \pathcost{\pi_A}$. We will show that we can shortcut $\fdpi$ to obtain a new plan $\fdpi'$ if necessary so that $\plancost{\fdpi'} \leq \plancost{\fdpi}$, $A$ is parked only $O(1)$ times in each cell of $\gridcells$, the parking places of $A$ in $\fdpi'$ are a subset of those in $\fdpi$ and $\fdpi'$ is also a kissing plan. For a cell $g \in \grid$, let $N(g) \subset \grid$ be the set of cells $g' \in \grid$ such that there exists a pair of points $p \in g, q \in g'$ with $\linfnorm{p}{q} = 2$, \ie, $(p,q)$ is a kissing configuration. Note that $\abs{N(g)} \leq 25$.

Fix a cell $g \in \gridcells$. Let $C$ be a connected component of $g \cap \freesp$ that contains a parking place of $A$. Recall that $\fdpi$ is a kissing plan so $B$ kisses $A$ at each parking place of $A$. For each parking place $\xi$ of $A$ in $C$, we label it with cell $\tau \in \grid$ if $B$ was in cell $\tau$ when it kissed $A$ at $\xi$. If there are more than one such cell, we arbitrarily choose one of them. If $C$ contains more than two parking places of $A$ with the same label $\tau$ such that all of them are $x$-separated or all of them are $y$-separated, then we shortcut $\fdpi$ as follows. Let $\lambda^-$ (resp., $\lambda^+$) be the first (resp., last) time instance such that $\fdpi(\lambda^-)$ (resp., $\fdpi(\lambda^+)$) is a $x$-separated kissing configuration with $\pi_A(\lambda^-) \in \xi$, $\pi_B(\lambda^-) \in \tau$ (resp., $\pi_A(\lambda^+) \in \xi, \pi_B(\lambda^+) \in \tau$). We replace $\fdpi(\lambda^-,\lambda^+)$ with the $(\fdpi(\lambda^-),\fdpi(\lambda^+))$-plan described in \lemref{same-cells-plan} of cost $\geodesic{\pi_A(\lambda^-)}{\pi_A(\lambda^+)} + \geodesic{\pi_B(\lambda^-)}{\pi_B(\lambda^+)}$. We repeat this procedure in $C$ until there are no such parking places of $A$ in $g$. We repeat this step for all cells $g \in \gridcells$. Let $\fdpi'$ be the resulting plan. By construction, $\plancost{\fdpi'} \leq \plancost{\fdpi}$ and $\fdpi'$ is a kissing plan.

We now bound $\altern{\fdpi'}$. First note that $\fdpi$ intersects at most two components of $g \cap \freesp$ for each cell $g \in \gridcells$. For each such connected component, the plan $\fdpi'$ has at most $4 \abs{N(g)} \leq 100$ parking places. Therefore $g$ contains at most $200$ parkings of $A$ in the plan $\fdpi'$. Summing over all cells of $\gridcells$, we obtain that $A$ is parked $O(\ceil{\abs{\pi_A}}) = O(\pathcost{\pi_A}+1)$ times in the plan $\fdpi'$. Since $A$ and $B$ park alternately, $\altern{\fdpi'} = O(\pathcost{\pi_A}+1)$.
\end{proof}

\ifabbrv
\else
\section{Paths Inside a Corridor}
\seclab{paths-in-corridor}
In this section
we prove the existence of a decoupled, kissing, optimal plan in which neither of the two robots 
is ever parked in the sanctum of a corridor. We prove this result by introducing some convenient notations and establishing a few properties of a decoupled path inside a corridor.

Suppose $\fdpi$ is an decoupled, kissing, optimal $(\fd{s},\fd{t})$-plan, and let $\corridor \in \corridors$ 
be a corridor such that one of the robots, say, $\robA$, enters $\corridor$ and parks inside the sanctum $\sanctum{\corridor}$ of $\corridor$.
We have that $s_A,s_B,t_A,t_B \notin \Int(\corridor)$ since no point in $\enVerts$ lies in the the interior any corridor by definition. Let $I_A \assign [\lambda_A^-,\lambda_A^+]$ be a maximal time interval during which (the center of) $\robA$ is inside $\shiftseg{\corridor}{2}$, which contains the time $\lambda$ at which $\pi_A(\lambda) \in \sanctum{\corridor}$.
Let $\sigma_0$, $\sigma_1$ be the (not necessarily distinct) portals of $\corridor$ last crossed in $\pi_A(0,\lambda_A^-)$ and first crossed in $\pi_A(\lambda_A^+,1)$, respectively.
Then $\pi_A(\lambda_A^-),\pi_A(\lambda_A^+)$ lie on the edges of $\shiftseg{\sigma_0}{2}$ and $\shiftseg{\sigma_1}{2}$ of $\shiftseg{\corridor}{2}$, respectively, which again are not necessarily distinct.

We show that $\fdpi$ can be transformed to another decoupled, kissing, optimal $(\fd{s},\fd{t})$-plan without increasing the cost, so that $A$ does not park inside the sanctum $\sanctum{\corridor}$ during the interval $I_A$. We accomplish this in stages.
First, we show that $B$ enters $\corridor$ for an interval $I_B$,
with $I_A \cap I_B \neq \varnothing$, and it also enters (resp., exits) $\corridor$ through the portal $\sigma_0$ (resp., $\sigma_1$) (\lemref{both-in-corridor}).
Next, we show (in \lemref{corridor-same-portal}) that if $\sigma_0 = \sigma_1$, \ie, $A$ (and $B$) enters and exits
$\corridor$ at the same portal then $A$ does not enter the sanctum $\sanctum{\corridor}$ during the interval $I_A$ and $B$
also does not enter the sanctum during the interval $I_B$. If $\sigma_0 \neq \sigma_1$, then both $A$ and $B$
cross the sanctum $\sanctum{\corridor}$. In this case, we first argue that $\fdpi$ can be deformed so that $\pi_A(I_A \cap I_B)$
consists of at most one breakpoint and the breakpoint lies near the portals of $\corridor$, so it is outside $\shiftseg{\corridor}{6}$
and $A$ is not parked inside $\sanctum{\corridor}$ during $I_A$.
A similar claim holds for $B$ (see \lemref{opt-in-corridor}).
This deformation may result in losing the kissing property of $\fdpi$ during the interval $I_A \cap I_B$.
Finally, we show that we can reparameterize the paths $\pi_A$ and $\pi_B$ without changing their images, \ie, merging two or more moves into one or adjusting the parking places, so that the resulting $(\fd{s},\fd{t})$-plan is decoupled, kissing, and optimal, and neither $A$ nor $B$ is parked inside 
$\sanctum{\corridor}$ during the interval $I_A$ (\lemref{kissing-outside-sanctum}).

\begin{lemma}
\lemlab{both-in-corridor}
There is a maximal interval $I_B$ such that $I_A \cap I_B \neq \varnothing$ and $B$ is in $\corridor$ during $I_B$, \ie, $\pi_B(\lambda) \in \corridor$ for all $\lambda \in I_B$. Furthermore $B$ enters and exits $\corridor$ during $I_B$ through the same portals as $A$.
\end{lemma}

\begin{proof}
If $\pi_B(\lambda) \notin \corridor$ for all $\lambda \in I_A$, then 
$\sanctum{\corridor} \subseteq \bigcap_{\lambda \in I_A} \freept{\pi_B(\lambda)}$
and there is no need to park $A$ inside $\sanctum{\corridor}$.
That is, we can first move $A$ along $\pi_A(I_A)$ while $B$ is parked at $\pi_B(\lambda_A^-)$,
then park $A$ at the portal $\sigma_1$ and move $B$ along $\pi_B(I_A)$,
and then follow the rest of the plan, $\fdpi(\lambda_A^-,1)$.
So we assume that $\pi_B(\lambda) \in \corridor$ for some $\lambda \in I_A$.

Let $I_B \assign [\lambda_B^-, \lambda_B^+]$ be a maximal interval with $I_A \cap I_B \neq \varnothing$ during which $\robB$ is in $\corridor$.
Let $u$ be a vector normal to $\sigma_0$.
By \lemref{no-corridor-crossings}, the sign of $\inprod{\pi_A(\lambda) - \pi_B(\lambda)}{u}$ is the same for all $\lambda \in I_A \cap I_B$.
If $\robA$ and $\robB$ enter through different portals of $\corridor$, then we claim that $\fdpi$ is not an optimal plan.
Indeed, if $\robA$ enters and exits at the same portal, (\ie, $\sigma_0 = \sigma_1$), then we can shortcut $\pi_A(I_A)$ along $\shiftseg{\sigma_0}{2}$ to obtain a cheaper $(\fd{s},\fd{t})$-plan, and if $\robA$ exits at the other portal (\ie, $\sigma_0 \neq \sigma_1$), we can shortcut $\pi_B(I_B)$ along that portal (at which $\robB$ entered) to obtain a cheaper $(\fd{s},\fd{t})$-plan. A similar short-cutting argument holds if $B$ does not exit through the same portal as $A$.
This completes the proof of the lemma.
\end{proof}

\begin{figure}
\centering
\includegraphics[scale=0.70]{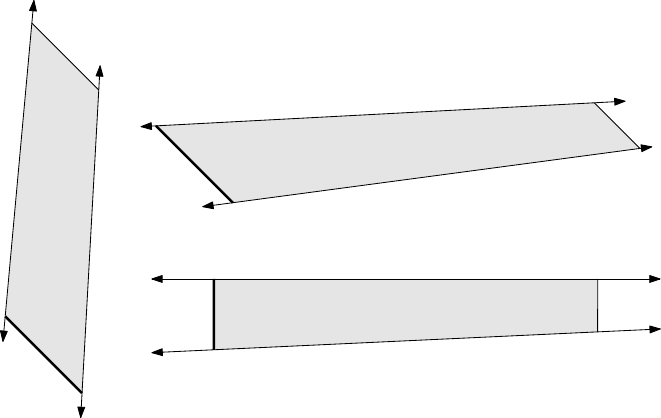}
\caption[Illustrations of corridors.]{Illustrations of corridors with the left portals $\sigma_L$ shown as thick. The left and top-right examples have portals with slope $-1$ and the bottom-right example has vertical portals.}
\figlab{corridor-wlog}
\end{figure}

For simplicity, we make some assumptions about the orientation of the features of $\corridor$ without loss of generality.
Recall that if at least one blocker of $\corridor$ is vertical (resp., horizontal) its portals are horizontal (resp., vertical).
First, we assume that the portals of $\corridor$ are vertical or have slope $-1$ by rotating the setting by $\pi/2$ as necessary.
Then the slopes of the blockers of $\corridor$ are strictly positive if its portals have slope $-1$, otherwise one blocker has non-negative slope
and the other blocker has non-positive slope. Second, we assume that $\robA$ and $\robB$ enter $\corridor$ from its ``left portal,'' which
formally is the portal whose endpoint is the leftmost vertex of $\corridor$ (if the portals are vertical this is obvious);
let $\sigma_L$ be this portal and $\sigma_R$ be the other.
See \figref{corridor-wlog}.

\begin{figure}
\centering
\includegraphics[scale=0.90]{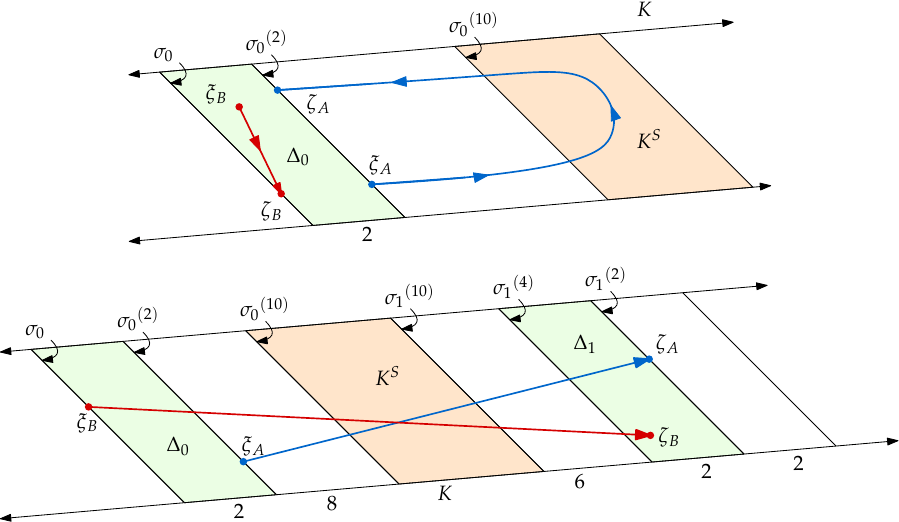}
\caption[Illustrations of $\fd{\xi},\fd{\zeta}$.]{Illustrations of $\fd{\xi},\fd{\zeta}$ when $\sigma_0 = \sigma_1$ (top) and $\sigma_0 \neq \sigma_1$ (bottom), not to scale.}
\figlab{xi-zeta-example}
\end{figure}

In the following, we prove that neither $A$ nor $B$ is parked inside the sanctum of $\corridor$ during the interval $I_A \cap I_B$. Without loss of generality, assume that $\lambda_A^- \leq \lambda_B^-$ because otherwise we can swap $A$ and $B$. We modify the plan $\fdpi$, without changing the images of $\pi_A,\pi_B$, so that $B$ is ``near'' $A$ when $A$ enters or exits $\corridor$: Let $\Delta_0$ be the trapezoid formed by the blockers of $\corridor$ and segments $\sigma_0,\shiftseg{\sigma_0}{2}$. If $\sigma_0 = \sigma_1$, let $\Delta_1 = \Delta_0$, otherwise let $\Delta_1$ be the trapezoid formed by the blockers of $\corridor$ and segments $\shiftseg{\sigma_1}{4},\shiftseg{\sigma_1}{2}$. See \figref{xi-zeta-example}.
Set $\xi_A \assign \pi_A(\lambda_A^-), \zeta_A \assign \pi_A(\lambda_A^+)$.
If $\pi_B(\lambda_A^-) \in \Delta_0$, we set $\xi_B \assign \pi_B(\lambda_A^-)$.
If $\pi_B(\lambda_A^-) \notin \Delta_0$, we park $A$ at $\xi_A$ and move $B$ from $\pi_B(\lambda_A^-)$ until $B$ enters $\Delta_0$ through $\sigma_0$ and park $B$ at this point, which we denote by $\xi_B$.
Then we follow the plan $\fdpi$ as before.
We use $\widehat{\lambda}_A^-$ to denote the time instance at which $A$ is at $\xi_A$ and $B$ is at $\xi_B$, and we continue to use $\fdpi$ to denote the unmodified plan.
Similarly, if $\pi_B(\lambda_A^+) \in \Delta_1$, then set $\zeta_B \assign \pi_B(\lambda_A^+)$. If $\pi_B(\lambda_A^+) \notin \Delta_1$ then we park $A$ at $\zeta_A \assign \pi_A(\lambda_A^+) \in \shiftseg{\sigma_1}{2}$ and move $B$ along $\pi_B$ until it enters $\Delta_1$ through a point $\xi_B \in \sigma_0$ if $\sigma_0 = \sigma_1$ or through a point $\xi_B \in \shiftseg{\sigma_1}{4}$ otherwise. See \figref{xi-zeta-example} again.
We use $\widehat{\lambda}_A^+$ to denote the time instance at which $A$ (resp., $B$) is at point $\xi_A$ (resp., $\zeta_B$).
We continue to use $\fdpi$ to denote the modified plan. It can be verified that the modified plan is feasible.
Since the images of $\pi_A$ and $\pi_B$ remained the same, the cost also remains the same.
Set $\fd{\xi} = (\xi_A,\xi_B) = \fdpi(\widehat{\lambda}_A^-)$ and $\fd{\zeta} = (\zeta_A,\zeta_B) = \fdpi(\widehat{\lambda}_A^+)$.
We note that the resulting plan may not be kissing during the interval $[\lambda_A^-,\lambda_A^+]$, \eg, $A$ may not kiss $B$ on its move to $\zeta_A$ and $B$ may not kiss $A$ on its move to $\zeta_B$. We will convert it into a kissing plan after we are done modifying the plan inside $\corridor$ (see \lemref{kissing-outside-sanctum}).

We extend our notation for portal-parallel lines and segments to define them by points that they contain: For any point $p \in \corridor$, 
let $\pointseg{\ell}{p}$ the line normal to $u_L$ (and $u_R$) containing $p$, and 
let $\pointseg{\sigma}{p} \assign \corridor \cap \pointseg{\ell}{p}$.
We next prove the following technical lemma, which shows that if $A$ is sufficiently deep inside and ahead of $B$ in $\corridor$,
then $B$ does not obstruct $A$ from reaching further inside the corridor; moreover, the shortest path for $A$ inside the corridor is the shortest path in the plane that avoids $B$. By swapping the roles of $A$ and $B$ and reflecting the setting over the $y=-x$ line, it can be used more generally.

\begin{lemma}
\lemlab{shortest-is-feasible}
Let $\corridor$ be a corridor with direction vector $u \in \{0,\pi/4\}$. For any configurations $(s_A,p_B),(t_A,p_B) \in \freesp$ such that $s_A \in \shiftseg{\corridor}{2}$, $p_B,t_A \in \corridor$, $\inprod{p_B}{u} < \inprod{s_A}{u} < \inprod{t_A}{u}$, and $\linfd{\pointseg{\ell}{s_A}}{\pointseg{\ell}{t_A}} \geq 2$, the shortest path $P_A$ from $s_A$ to $t_A$ in $\cl(\Reals^2 \setminus (p_B + 2\Box))$ is such that:
\begin{enumerate}[(i)]
    \item $P_A$ is contained in $\cl(\corridor \setminus (p_B + 2\Box)) \subset \corridor \cap \freept{p_B}$,
    \item if the portals of $\corridor$ are vertical then $P_A$ is segment $s_At_A$ and $\pathcost{P_A} = \ltnorm{s_A}{t_A}$, and
    \item if the portals of $\corridor$ have slope $-1$ then $P_A$ is segment $s_At_A$ if $s_At_A \cap \Int(p_B w) = \varnothing$ and $P_A = s_Aw \ccat w t_A$ otherwise, where $w$ is the top-right vertex $w$ of $p_B + 2\Box$, and $\pathcost{P_A} < 2\sqrt{5} + \ltnorm{s_A}{t_A}$.
\end{enumerate}
\end{lemma}

\begin{figure}
\centering
\includegraphics[scale=0.90]{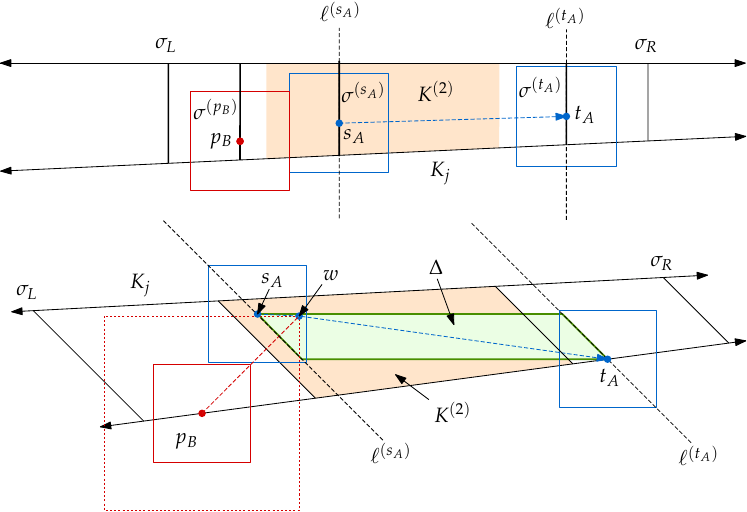}
\caption[Illustrations of the proof of \lemref{shortest-is-feasible}.]{Illustrations of the proof of \lemref{shortest-is-feasible}. (top) $\corridor$ has vertical portals so $P_A$ is a segment. (bottom) $\corridor$ has portals with slope $-1$ and segment $s_At_A$ (not shown) intersects $\Int(p_B+2\Box)$ and hence segment $p_Bw$. The region $\Delta$ contains vertex $w$ of $p_B+2\Box$.}
\figlab{shortest-is-feasible}
\end{figure}

\begin{proof}
See \figref{shortest-is-feasible}. Let $P_A$ be the shortest path from $s_A$ to $t_A$ in $\cl(\Reals^2 \setminus (p_B+2\Box))$. If $P_A$ has no breakpoints, $P_A$ is segment $s_At_A$ so $P_A \subset \corridor \setminus (p_B+2\Box)$ since $s_A,t_A \in \corridor$ and $\corridor$ is convex. Then (i) holds when $P_A$ is a segment. We will prove (i) when $P_A$ is not a segment later.

We next prove (ii). Assume the portals are vertical. Since the portals are vertical, we have $$x(\sigma_L) \leq x(p_B) \leq x(s_A) \leq x(t_A)-2 \leq x(\sigma_R)-4.$$ If $x(s_A) \geq x(p_B) + 2$ then $s_A,t_A$ lie to the right of (the right vertical edge of) $p_B+2\Box$ and hence $P_A = s_At_A$. So assume $\abs{x(s_A)-x(p_B)} < 2$ for sake of contradiction. By definition, $(s_A,p_B) \in \freesp$, so $\abs{y(s_A)-y(p_B)} \geq 2$. $s_A,p_B \in \corridor$ so the segment $p_Bs_A \subset \corridor$ because $\corridor$ is convex. For concreteness, suppose $y(p_B) \leq y(s_A) - 2$ so the segment $p_Bs_A$ has positive slope; the other case is symmetric.
Recall that because $\corridor$ has vertical portals, $\corridor$ has one blocker with non-positive slope and the other has non-negative slope and $\abs{\pointseg{\sigma}{\eta}} < 2$ for all points $\eta \in \Int(\corridor)$ by \lemref{short-corridor-segments}.
Let $Q$ be the rectangle with $s_A$ (resp., $p_B$) as its bottom-left (resp., top-right) vertex and let $e_1$ (resp., $e_2$) be the left (resp. right) vertical edge of $Q$; $\abs{e_1},\abs{e_2} \geq 2$. The blocker with non-positive slope must not intersect $\Int(Q)$ and the blocker with non-negative slope must intersect at most one of $\Int(e_1)$ or $\Int(e_2)$ (otherwise a blocker would intersect $\Int(p_Bs_A)$ which is impossible since $p_Bs_A \subset \corridor$). Then either $e_1 \subset \corridor$ or $e_2 \subset \corridor$. If $e_1 \subset \corridor$ (resp. $e_2 \subset \corridor$) then $\abs{\pointseg{\sigma}{p_B}} \geq \abs{e_1} \geq 2$ (resp., $\abs{\pointseg{\sigma}{s_A}} \geq \abs{e_2} \geq 2$), which is a contradiction. So $P_A = s_At_A$ as claimed. This proves (ii).

We next prove (i) when $P_A \neq s_At_A$. Assume the portals have slope $-1$. If $P_A = s_At_A$ then (iii) holds by the discussion above, so assume otherwise.
Let $h$ be the halfspace containing $t_A$ defined by the line supporting $\pointseg{\sigma}{s_A}$. $\inprod{p_B}{u} < \inprod{s_A}{u} \leq \inprod{t_A}{u}$ implies $p_B \notin h$ since $t_A \in h$ and $u$ is normal to $\pointseg{\sigma}{s_A}$. Then the top-right vertex $w$ of $p_B+2\Box$ is the only vertex that lies in $\Int(h)$. It is easy to see that the shortest path $P_A \subset h$ since $t_A \in h$ (otherwise $P_A$ must leave $h$, wrap around $p_B+2\Box \setminus h$, and then re-enter $h$, which can only be longer). Since $P_A \neq s_At_A$, $s_At_A$ must intersect $p_B+2\Box) \cap h$, and hence it intersects segment $p_Bw$. Then $P_A = s_Aw \ccat wt_A$, as $w$ is the only vertex of $\cl(\Reals^2 \setminus (p_B+2\Box))$ in $h$. It follows that either $x(s_A) < x(w) < x(t_A)$ and $y(s_A) > y(w) > y(t_A)$ or all of the inequalities are reversed. For concreteness, assume the former; the other case is symmetric. We have $\inprod{p_B}{u} < \inprod{s_A}{u} < \inprod{t_A}{u}$, and $\linfd{\pointseg{\ell}{s_A}}{\pointseg{\ell}{t_A}} \geq 2$ by definition. The latter implies $\ldist{\pointseg{\ell}{s_A}}{\pointseg{\ell}{t_A}}{2} \geq 2\sqrt{2}$. Then $$\inprod{w}{u} = \inprod{p_B}{u} + 2\sqrt{2} < \inprod{s_A}{u}+2\sqrt{2} \leq \inprod{s_A}{u} + \ldist{\shiftseg{\ell}{s_A}}{\shiftseg{\ell}{t_A}}{2} = \inprod{t_A}{u}.$$
Then we have $$\inprod{s_A}{u} < \inprod{w}{u} < \inprod{t_A}{u},$$ where the first inequality follows from the fact $w \in \Int(h)$. Now let $\Delta$ be the trapezoid with edges on $\pointseg{\sigma}{s_A},\pointseg{\sigma}{t_A}$ and horizontal lines $y=y(s_A)$, $y=y(t_A)$; $s_A$ (resp., $t_A$) is top-left (resp., bottom-right) vertex of $\Delta$. It follows that $w \in \Delta$ since $y(s_A) \geq x(w) \geq y(t_A)$ and the previous inequalities. Recall that the blockers of $\corridor$ have positive slope. Then the top (resp., bottom) blocker of $\corridor$ lies above the top (resp. bottom) horizontal edge of $\Delta$ and hence $\Delta \subset \shiftseg{\corridor}{2}$. Then the segments $s_Aw,wt_A$ lie in $\shiftseg{\corridor}{2}$ because it is convex. It follows that $P_A = s_Aw \ccat wt_A \subset \cl(\shiftseg{\corridor}{2} \setminus (p_B+2\Box))$ as desired.

It remains to prove $\pathcost{P_A} \leq 2\sqrt{5} + \ltnorm{s_A}{t_A}$ in (iii) when $P_A$ has $w$ as a breakpoint.
By assumption that $\inprod{p_B}{u} \leq \inprod{s_A}{u}$, we have $x(s_A) \geq x(p_B)-2 \geq x(w)-4$, where the second inequality follows from the fact $w$ is the top-right vertex of $p_B+2\Box$. The $L_\infty$-distance between the endpoints of $\pointseg{\sigma}{s_A},\pointseg{\sigma}{t_A}$ are less than $2$ by \lemref{short-corridor-segments} because $s_A,t_A \in \Int(\corridor)$, so $y(s_A) - y(w) \leq y(s_A) - y(t_A) < 2$.
Putting everything together, we have $\abs{x(s_A)-x(w)} \leq 4$ and $\abs{y(s_A)-y(w)} < 2$, and hence $\ltnorm{s_A}{w} < \sqrt{4^2+2^2} = 2\sqrt{5}$. Clearly $\ltnorm{w}{t_A} \leq \ltnorm{s_A}{t_A}$, so $$\pathcost{P_A} = \ltnorm{s_A}{w} + \ltnorm{w}{t_A} \leq 2\sqrt{5} + \ltnorm{s_A}{t_A}.$$
Having shown (i)--(iii), this concludes the proof.

\end{proof}

Now we are ready to prove the the first main lemma of this section.

\begin{lemma}
\lemlab{corridor-same-portal}
If $\sigma_0 = \sigma_1$, \ie, $A$ enters and exits $\corridor$ from the same portal during interval 
$I_A$ then $A$ does not enter the sanctum $\sanctum{\corridor}$ during $I_A$. Similarly $B$ does not enter $\sanctum{\corridor}$ 
during the interval $I_A \cap I_B$.
\end{lemma}

\begin{proof}
For sake of contradiction, suppose $A$ enters $\sanctum{\corridor}$ during interval $I_A$.
Then $\pi_A(I_A)$ intersects segments $\shiftseg{\sigma_0}{2}$ and $\shiftseg{\sigma_0}{10}$.
Let $\gamma_A^0 \in \shiftseg{\sigma_0}{4}$ (resp., $\gamma_A^1 \in \shiftseg{\sigma_0}{4}$) be the first (resp., last) point of $\pi_A(I_A)$ on $\shiftseg{\sigma_0}{4}$.
See \figref{corridor-same-portal}. We now deform $\fdpi$ by replacing $\fdpi(I_A)$ with another plan 
$\widehat{\fdpi} \assign (\widehat{\pi}_A,\widehat{\pi}_B)$:
\begin{enumerate}[(1)]
\item Move $\robA$ from $\xi_A$ to $\gamma_A^0$ while $\robB$ is parked at $\xi_B$,
\item move $\robB$ from $\xi_B$ to $\zeta_B$ while $\robA$ is parked at $\gamma_A^0$,
\item move $\robA$ from $\gamma_A^0$ to $\gamma_A^1$ while $\robB$ is parked at $\zeta_B$, and
\item continue moving $\robA$ from $\gamma_A^1$ to $\zeta_A$ while $\robB$ is parked at $\zeta_B$.
\end{enumerate}

In each move ($i$), $i = 1,\ldots,4$, the moving robot follows the shortest feasible path $\widehat{\pi}_i$ between its start and target placements, while the other robot is parked. Note that $\xi_B,\xi_A,\gamma_A^0$ and $\zeta_B,\zeta_A,\gamma_A^1$ satisfy the conditions as $p_B,s_A,t_A$ in the statement of \lemref{shortest-is-feasible}, respectively, by definition. Then $\pathcost{\widehat{\pi}_1} \leq 2\sqrt{5} + \ltnorm{\xi_A}{\gamma_A^0}$ and $\pathcost{\widehat{\pi}_4} \leq 2\sqrt{5} + \ltnorm{\zeta_A}{\gamma_A^1}$. Furthermore, $\linfd{\Delta_0}{\shiftseg{\sigma_0}{4}} \geq 2$ by definition. It follows that segment $\xi_B\zeta_B \subset \freept{\gamma_0^1}$ and segment $\gamma_A^0\gamma_B^1 \subset \freept{\zeta_B}$, so $\widehat{\pi}_2 = \xi_B\zeta_B$ and $\widehat{\pi}_3 = \gamma_A^0\gamma_A^1$.
Then
\begin{align*}
\pathcost{\widehat{\fdpi}} &\leq (2\sqrt{5} + \ltnorm{\xi_A}{\gamma_A^0}) + \ltnorm{\xi_B}{\zeta_B} + \ltnorm{\gamma_A^0}{\gamma_A^1} + (2\sqrt{5} + \ltnorm{\zeta_A}{\gamma_A^1}) \\
&\leq \ltnorm{\xi_B}{\zeta_B} + \ltnorm{\xi_A}{\gamma_A^0} + \ltnorm{\zeta_A}{\gamma_A^1} + 4\sqrt{5} + 2\sqrt{2},
\end{align*}
where the last inequality follows from the fact $\abs{\shiftseg{\sigma_0}{4}} < 2\sqrt{2}$ by \lemref{short-corridor-segments}. On the other hand, $\pathcost{\pi_B} \geq \ltnorm{\xi_B}{\zeta_B}$ and $$\pathcost{\pi_A} \geq \ltnorm{\xi_A}{\gamma_A^0} + \ltnorm{\zeta_A}{\gamma_A^1} + \pathcost{\plpt{\pi_A}{\gamma_A^0,\gamma_A^1}} \geq \ltnorm{\xi_A}{\gamma_A^0} + \ltnorm{\zeta_A}{\gamma_A^1} + 12,$$ where the last inequality follows from the fact $\pathcost{\plpt{\pi_A}{\gamma_A^0,\gamma_A^1}} \geq 2\linfd{\shiftseg{\sigma_0}{4}}{\shiftseg{\sigma_0}{10}} \geq 2(6) = 12$ by definition. Putting everything together, we have $$\plancost{\fdpi} - \plancost{\widehat{\fdpi}} \geq 12 - 4\sqrt{5} - 2\sqrt{2} > 0.$$ But then $\fdpi$ is not optimal, which is a contradiction.
We conclude that our assumption that $\pi_A(I_A)$ enters the sanctum $\sanctum{\corridor}$ is false. By \lemref{no-corridor-crossings},
$\pi_B(I_A \cap I_B)$ also does not enter the sanctum $\sanctum{\corridor}$.
\end{proof}

\begin{figure}
\centering
\includegraphics[scale=0.90]{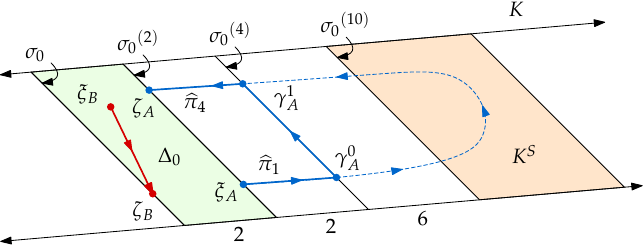}
\caption[Illustration of \lemref{corridor-same-portal}.]{Illustration of \lemref{corridor-same-portal}. $\plpt{\pi_A}{\gamma_A^0,\gamma_A^1}$ is depicted as dashed and $\widehat{\pi}_i$ are thick for $i = 1,\ldots,4$. Note that this example is not drawn to scale.}
\figlab{corridor-same-portal}
\end{figure}


In the rest of this section we further assume that $A$ and $B$ exit $\corridor$ through its right portal in addition to the assumption that they enter $\corridor$ through its left portal, \ie, $\sigma_L = \sigma_0$ and $\sigma_R = \sigma_1$.

\begin{lemma}
\lemlab{opt-in-corridor}
There exists a decoupled, optimal $(\fd{\xi},\fd{\zeta})$-plan $\fd{\psi} = (\psi_A,\psi_B) \subset \corridor
\times \corridor$ that consists of two moves: first move $A$ along the shortest $(\xi_A,\zeta_A)$-path in
$\freept{\xi_B}$ while $B$ is being parked at $\xi_B$, and then move $B$ along the shortest $(\xi_B,\zeta_B)$-path
in $\freept{\zeta_A}$ while $A$ is being parked at $\zeta_A$.
\end{lemma}

\begin{proof}
Let $u$ be the direction of $\corridor$.
We have $\xi_A \in \shiftseg{\sigma_1}{2}$, $\xi_B \in \Delta_0$, $\zeta_A \in \shiftseg{\sigma_1}{4}$, $\zeta_B \in \Delta_1$ and $\inprod{\xi_B}{u} < \inprod{\xi_A}{u} < \inprod{\zeta_B}{u} < \inprod{\zeta_A}{u}$ by definition.

Let $\psi_A$ be the shortest path from $\xi_A$ to $\zeta_A$ in $\Reals^2 \setminus (\xi_B+2\Box)$, and let $\psi_B$ be the shortest path from $\xi_B$ to $\zeta_B$ in $\Reals^2 \setminus (\zeta_B+2\Box)$. Then we have $\psi_A \subset \freept{\xi_B}$ and $\psi_B \subset \freept{\zeta_A}$ by \lemref{shortest-is-feasible}, where the latter is obtained by swapping the roles of $A$ and $B$ and reflecting the setting over the $y=-x$ line to apply the lemma. Let $\fd{\psi} = (\psi_A,\psi_B)$ be the $(\fd{\xi},\fd{\zeta})$-plan. It remains to show that $\fd{\psi}$ is an optimal plan, which does not follow immediately from the fact each of $\psi_A,\psi_B$ are ``locally'' optimal in the sense they are shortest paths between specific placements while the other robot is at another specific placement.
Recall that some optimal plans have more than two (but at most three) moves even for the case $\envir = \freesp = \Reals^2$ \cite{Esteban2022,mastersthesisRuizHerrero}, as mentioned in \secref{prelim} (see the example in \figref{opt-nontrivial-example}). So we have to rely on the structure of $\fd{\xi},\fd{\zeta}$.

We first rule out the easy case, which is when $\psi_A,\psi_B$ are both segments. See \figref{opt-in-corridor}(top-left,bottom). Then $\plancost{\fd{\psi}} = \ltnorm{\xi_A}{\zeta_A} + \ltnorm{\xi_B}{\zeta_B}$, which is optimal since any plan must have at least this cost. By \lemref{shortest-is-feasible}, $\psi_A,\psi_B$ are segments when the portals are vertical, in which case we are done. So suppose one of them, say, $\psi_A$ is not the segment $\xi_A\zeta_A$; the case where $\psi_B \neq \xi_B\zeta_B$ is symmetric. Then $\xi_A\zeta_A \cap \Int(\xi_B+2\Box) \neq \varnothing$, and by \lemref{shortest-is-feasible}, $\psi_A = \xi_A w \ccat w\zeta_A$, where $w$ is the top-right vertex of $\xi_B+2\Box$. It follows that either $x(\xi_A) < x(w) < x(\zeta_A)$ and $y(\xi_A) > y(w) > y(\zeta_A)$, or all inequalities are reversed. We assume the former; the latter case is symmetric. See \figref{opt-in-corridor}(top-right).

\begin{figure}
\centering
\includegraphics[scale=0.8]{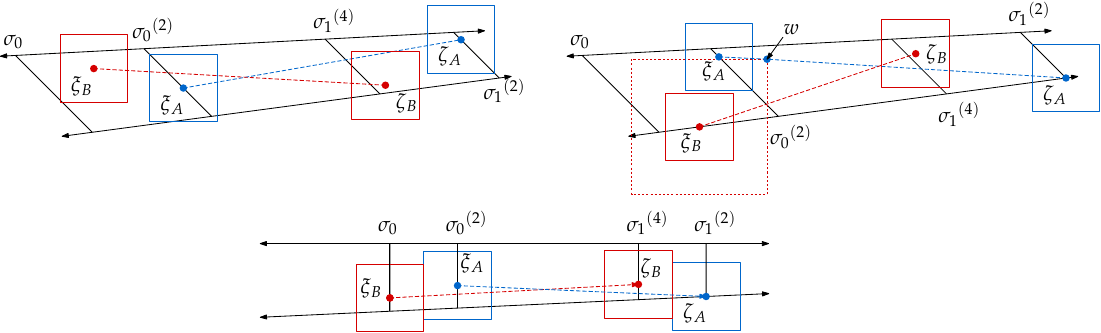}
\caption[Illustrations of cases from the proof of \lemref{opt-in-corridor}.]{Illustrations of cases from the proof of \lemref{opt-in-corridor}. Top-left: The portals have slope $-1$, and $\psi_A,\psi_B$ are segments.
Top-right: The portals have slope $-1$ and $\psi_A$ has a breakpoint $w$ which is a vertex of $\xi_B + 2\Box$; $\psi_B$ is a segment.
Bottom: The portals are vertical, so $\psi_A,\psi_B$ are segments. Only in the bottom example are $\fd{\xi},\fd{\zeta}$ kissing configurations, which is necessary when portals are axis-aligned since the lines containing the placements are at Euclidean distance $2$ in that case.}
\figlab{opt-in-corridor}
\end{figure}

The current $\fd{\xi},\fd{\zeta}$ satisfy a case in \cite{mastersthesisRuizHerrero}, specifically case ``Zone III(2)'' in their Section~4.3.1, which implies that $\fd{\psi}$ is optimal in the plane without obstacles, and hence in our setting, since we have already shown $\fd{\psi}$ is feasible in our setting (\ie, inside $\corridor$). To verify that $\fd{\xi},\fd{\zeta}$ satisfies their case ``Zone III(2),'' it suffices\footnote{We note that their case ``Zone III(2)'' is more general in the sense that properties (P1)--(P5) are only sufficient conditions, which is easy to verify from their paper.} to have the following properties in addition to the inequalities on the $x$- and $y$-coordinates above:
\begin{itemize}
\item[(P1)] $\xi_B \in \Int(\xi_A\zeta_A \oplus 2\Box)$,
\item[(P2)] $\xi_A \in \Int(\xi_B\zeta_B \oplus 2\Box)$,
\item[(P3)] $\linfnorm{\xi_B}{\zeta_B} \geq 2$,
\item[(P4)] $\abs{x(\xi_B) - x(\xi_A)} < 2$ and $y(\xi_B) \leq y(\xi_A) - 2$, \ie, $\xi_B$ lies below $\xi_A+2\Box$.
\item[(P5)] $\abs{y(\zeta_B) - y(\zeta_A)} < 2$ and $x(\zeta_B) \leq x(\zeta_A) - 2$, \ie, $\zeta_B$ lies left of $\zeta_A+2\Box$.
\end{itemize}

See \figref{opt-in-corridor} (top-right) again. (P1) follows from the fact $\xi_A\zeta_A \cap \Int(\xi_B+2\Box) \neq \varnothing$. Since $\inprod{\xi_B}{u} < \inprod{\xi_A}{u} < \inprod{\zeta_B}{u}$, $\xi_B\zeta_B$ crosses $\pointseg{\sigma}{\xi_A}$. Then we have $$\pointseg{\sigma}{\xi_A} \subset \Int(\xi_A + 2\Box) \subset \Int(\xi_A\zeta_A \oplus 2\Box),$$ where the first containment follows from \lemref{short-corridor-segments}, and hence (P2) holds. We next prove (P3). We have that $$\linfnorm{\xi_B}{\zeta_B} \geq \linfd{\Delta_0}{\Delta_1} \geq \linfd{\shiftseg{\ell_0}{2}}{\shiftseg{\ell_1}{4}} \geq \corlen{\corridor} - 2 - 4 \geq 14,$$ where the second inequality follows from the definitions of $\Delta_0,\Delta_1$ being trapezoids bounded between $\sigma_0,\shiftseg{\ell_0}{2}$ and $\shiftseg{\ell_0}{4},\shiftseg{\ell_0}{2}$, respectively, and the last inequality follows from the fact $\corlen{\corridor} \geq 20$ since $\corridor$ has a non-empty sanctum $\sanctum{\corridor} = \shiftseg{\corridor}{10} \subset \corridor$. So (P3) holds.

Next, we have that the segment $\shiftseg{\sigma_0}{2} = \pointseg{\sigma}{\xi_A}$ (resp., $\shiftseg{\sigma_1}{2} = \pointseg{\sigma}{\zeta_A}$) has slope $-1$, and the $L_\infty$-distance between its endpoints is than $2$ by \lemref{short-corridor-segments}. We have $\linfnorm{\xi_A}{\xi_B},\linfnorm{\zeta_A,\zeta_B} \geq 2$ since $\fd{\xi},\fd{\zeta} \in \fdfreesp$ are configurations. It follows that, with $\inprod{\xi_B}{u} < \inprod{\xi_A}{u}$ (resp., $\inprod{\zeta_B}{u} < \inprod{\xi_A}{u}$), $x(\xi_B) < x(\xi_A) + 2$ (resp., $x(\zeta_B) < x(\zeta_A) + 2$). We have $x(\xi_B) + 2 = x(w) > x(\xi_A)$, so together we have $\abs{x(\xi_B) - x(\xi_A)} < 2$, the first part of (P4). The second part follows from $y(\xi_B) - 2 = y(w) < y(\xi_A)$, so we have (P4).

It remains to prove (P5). We have $y(\zeta_A) < y(\xi_A)$. Then $y(\zeta_A) \leq y(\xi_A)$. For sake of contradiction, suppose we have $y(\zeta_B) \leq y(\zeta_A)-2$. The bottom endpoint $q$ of $\shiftseg{\ell_0}{2}$ has $y(q) > y(\xi_A)-2$ since $\xi_A$ lies on the segment and its $L_\infty$-length is less than $2$ as described above. But then $$y(\zeta_B) \leq y(\zeta_A)-2 < y(\xi_A)-2 \leq y(q),$$ which is a contradiction since the blockers of $\corridor$ have positive slope, \ie, the bottom endpoint of segment $\pointseg{\sigma}{\zeta_B}$, which contains $\zeta_B$, lies on or above $y = y(q)$. So $y(\zeta_B) > y(\zeta_A)-2$. A similar argument implies $y(\zeta_B) < y(\zeta_A)+2$. So we have $\abs{y(\zeta_B) - y(\zeta_A)} < 2$, the first part of (P5). Since $\fd{\zeta}$ is a configuration, so we have $\linfnorm{\zeta_A}{\zeta_B} \geq 2$. Since $\abs{y(\zeta_B) - y(\zeta_A)} < 2$, it must be that $\abs{x(\zeta_B)-x(\zeta_A)} \geq 2$. It cannot be the case that $x(\zeta_B) \geq x(\zeta_A) + 2$ since $\inprod{\zeta_B}{u} < \inprod{\zeta_A}{u}$. So $x(\zeta_B) \leq x(\zeta_A) - 2$, and we have (P5). This concludes the proof.
\end{proof}

\begin{lemma}
\lemlab{kissing-outside-sanctum}
Suppose a robot, say $A$, is parked inside the sanctum of a corridor $\corridor$ at time $\lambda \in [0,1]$ in a decoupled, kissing, optimal $(\fd{s},\fd{t})$-plan $\fdpi$, and let $I_A$ be a maximal time interval with $\lambda \in I_A$ during which $A$ is inside $\shiftseg{\corridor}{2}$. Let $I_B$ be the maximal time interval with $I_A \cap I_B \neq \varnothing$ as given by \lemref{both-in-corridor}.
Then there exists a decoupled, kissing, optimal $(\fd{s},\fd{t})$-plan 
    $\fdpi'$ and an interval $I \supseteq I_A \cup I_B$ such that neither $A$ nor $B$ parks inside the sanctum 
    $\sanctum{\corridor}$ of $\corridor$ during $\fdpi'(I)$ and $\fdpi(\lambda)=\fdpi'(\lambda)$ for 
    all $\lambda\not\in I$.
\end{lemma}

\begin{proof}
Following the notation above, by \lemref{corridor-same-portal}, $\sigma_0 \neq \sigma_1$. We then modify the plan during the interval $I_A$ as described above (preceding \lemref{both-in-corridor}). Note that $\fdpi(\widehat{\lambda}_A^-,\widehat{\lambda}_A^+)$ is a $(\fd{\xi},\fd{\zeta})$-plan. By \lemref{opt-in-corridor}, we can replace $[\widehat{\lambda}_A^-,\widehat{\lambda}_A^+]$ with the $(\fd{\xi},\fd{\zeta})$-plan $\fd{\psi}$ without increasing the cost of the overall plan. We thus obtain a decoupled, optimal $(\fd{s},\fd{t})$-plan $\fdpi'$ such that $\pi_A(I_A) \cap \shiftseg{\corridor}{4}$ and $\pi_B(I_A) \cap \shiftseg{\corridor}{4}$ are line segments, $\fdpi'(I_A)$ consists of two moves, and no robot is parked inside $\shiftseg{\corridor}{4}$. However, the resulting plan $\fdpi' = (\pi_A',\pi_B')$ may not be kissing. We convert it into a kissing plan without changing the images of $\pi_A'$ and $\pi_B'$ by applying the construction described in the proof of \lemref{kissing} repeatedly, as follows.

Let $\psi_A$ (resp., $\psi_B$) be the path followed by $A$ (resp., $B$) in the $(\fd{\xi},\fd{\zeta})$-plan. Let $\psi_A^-$ (resp., $\psi_B^-$) be the move of $A$ (resp., $B$) that brought it to $\xi_A$ (resp., $\zeta_A$), and let $p_A^-$ (resp., $p_B^-$) be the initial point of $\psi_A^-$ (resp., $\psi_B^-$), \ie, where $A$ (resp., $B$) was parked before $\xi_A$ (resp., $\xi_B$). Similarly, let $\psi_A^+$ (resp., $\psi_B^+$) be the move of $A$ (resp., $B$) that took it from $\zeta_A$ (resp., $\zeta_B$) to its next parking position denoted by $p_A^+$ (resp., $p_B^+$). These six moves are the only moves which might not be kissing. For simplicity, assume that none of them is the first or last move of $\fdpi'$. First consider the case where $\widehat{\lambda}_A^- = \lambda_A^-$, \ie, $B$ has already parked at $\xi_B$ when $A$ reached $\xi_A$, in which case the sequence of moves is
$$\ldots, (B,\psi_B^-,p_A^-), (A, \psi_A^-, \xi_B), (A, \psi_A, \xi_B), (B, \psi_B, \zeta_A), \ldots$$
We note that $A$ moves in both the second and third move, so we can transform the sequence as
$$\ldots, (B,\psi_B^-,p_A^-), (A, \psi_A^- \ccat \psi_A, \xi_B), (B, \psi_B, \zeta_A), \ldots$$
Since the original plan was kissing, $B$ kisses $A$ while moving along the path $\psi_B^-$. If $A$ kisses $B$ during $\psi_A^- \ccat \psi_A$, we do not need to modify the moves $\psi_B,\psi_A^-\ccat \psi_A$. So assume $A$ does not kiss $B$ in this move. Consider the first point of $\psi_B$, denoted by $\eta$, that intersects $\psi_A^- \ccat \psi_A \oplus 2\Box$. Since both $\psi_A$ and $\psi_B$ cross $\shiftseg{\sigma_0}{2}$ and the $L_\infty$-distance between its endpoints is less than $2$ by \lemref{short-corridor-segments}, $\eta$ lies to the left of $\shiftseg{\sigma_0}{2}$. Let $\psi_B^< \assign \plpt{\psi_B}{\eta,\zeta_B}$. We now park $B$ at $\eta$ instead of $\xi_B$, \ie, the plan becomes
$$\ldots, (B,\psi_B^- \ccat \psi_B^<,p_A^-), (A, \psi_A^- \ccat \psi_A, \eta), (B, \psi_B^>, \zeta_A), \ldots$$
Now both $(B,\psi_B^- \ccat \psi_B^<,p_A^-)$ and $(A, \psi_A^- \ccat \psi_A, \eta)$ are kissing moves and the plan remains feasible.

Next, suppose $\lambda_A < \widehat{\lambda}_A^-$, \ie, $A$ is parked at $\xi_A$ while $B$ moves from $p_B^-$ to $\xi_B$, \ie, the plan $\fdpi'$ is of the form
$$\ldots, (A,\psi_A^-,p_B^-), (A, \psi_B^-, \xi_A), (A, \psi_A, \xi_B), (B, \psi_B, \zeta_A), \ldots$$
Since $\xi_B \in \sigma_0$ and $\xi_A \in \shiftseg{\sigma_0}{2}$ in this case and the original plan was a kissing plan, it is easily seen that $A$ kisses $B$ during the $\psi_A^-$ move. Next, the path $\psi_B^-$ lies to the left of $\sigma_0$ while $\xi_A \in \shiftseg{\sigma_0}{2}$, so if $B$ kisses $A$ while moving along $\psi_B^-$, it happens only at $\xi_B$ in which case $A$ also kisses $B$ during the move $\psi_A$. Suppose $B$ does not kiss $A$ during $\psi_B^-$. Since $\psi_A$ lies to the right of $\shiftseg{\sigma_0}{2}$, $\Int(\psi_A + 2\Box) \cap \psi_B^- = \varnothing$. Therefore, we can combine $\psi_A^-$ and $\psi_A$, \ie, the plan becomes
$$\ldots, (A, \psi_A^- \ccat \psi_A, p_B^-), (B, \psi_B^-, \zeta_A), (B, \psi_B, \zeta_A), \ldots$$
Of course, $A$ kisses $B$ during $\psi_A^- \ccat \psi_A$. $B$ moves in both the second and third move, so we can transform the sequence as
$$\ldots, (A, \psi_A^- \ccat \psi_A, p_B^-), (B, \psi_B^- \ccat \psi_B, \zeta_A), \ldots.$$

In summary, we convert $\fdpi'$ into another plan without changing the images of the paths so that it is kissing until the move that contains $\psi_A$. Furthermore, the new parking place $\eta$ we added (only in the first case) lies outside $\shiftseg{\corridor}{2}$. We continue this process until the moves containing paths $\psi_B,\psi_A^+$, and $\psi_B^+$ also become kissing. However, they push the parkings of $A$ and $B$ only later, \ie, beyond the time at which $B$ leaves $\zeta_B$. Hence, we conclude that the transformation converts $\fdpi'$ into a kissing plan without adding a parking in $\sanctum{\corridor}$ during the interval $I_A$ and also without changing the images of the paths $\pi_A',\pi_B'$. This concludes the proof of the lemma.
\end{proof}

By applying \lemref{kissing-outside-sanctum} repeatedly, we obtain the following corollary.
\begin{corollary}
\corlab{no-sanctum-parkings}
For any reachable configurations $\fd{s},\fd{t} \in \fdfreesp$, there exists a decoupled, kissing, optimal $(\fd{s},\fd{t})$-plan
in which no robot parks inside the sanctum of a corridor of $\corridors$.
\end{corollary}
\fi

\section{Near-Optimal Tame Plans}
\seclab{approx-tame}

Let $\enVerts$ be the set of vertices of $\freesp$ plus $\{s_A,s_B,t_A,t_B\}$ and the vertices of all maximal corridors in $\corridors$, \ie, the endpoints of their portals. In this section, we show that a kissing, decoupled, optimal plan can be deformed by paying a fixed (constant) cost so that all robots are parked near a point of $\enVerts$.
For two parameters $\Delta^-,\Delta^+$ with $0 \leq \Delta^- \leq \Delta^+$, we say that a point $p \in \freesp$ is \emph{$\dclose{\Delta^-}{\Delta^+}$} (to $X$) if $\linfd{p}{\enVerts} \in [\Delta^-,\Delta^+]$. Often we will be interested in only one of $\Delta^-$ and $\Delta^+$, so we say is \emph{$\Delta$-close} (resp., \emph{$\Delta$-far}, \emph{$\Delta$-tight}) if $\linfd{p}{\enVerts} \leq \Delta$ (resp., $\linfd{p}{\enVerts} \geq \Delta, \linfd{p}{\enVerts} = \Delta$). A decoupled $(\fd{s},\fd{t})$-plan $\fdpi = (\pi_A,\pi_B)$ is called \emph{$\Delta$-tame} (or \emph{tame} if the value of $\Delta$ is clear from the context) if every parking place on $\pi_A,\pi_B$ is $\Delta$-close. The following lemma is the main result of this section and one of the crucial properties on which our algorithm relies. Throughout this section, we set $\Delta_0 \assign 30$, which is simply a constant that is sufficiently large for our needs.

\begin{lemma}
\lemlab{no-faraway}
Given reachable configurations $\fd{s},\fd{t} \in \fdfreesp$, let $\fdpi$ be a decoupled, kissing $(\fd{s},\fd{t})$-plan. For any parameter $\Delta \geq \Delta_0$, there exists a decoupled, kissing, $\Delta$-tame $(\fd{s},\fd{t})$-plan $\fdpi'$ such that $\fdpi' = \fdpi$ if $\plancost{\fdpi} \leq \Delta$, and $\plancost{\fdpi'} \leq \plancost{\fdpi} + c_1$ and $\altern{\fdpi'} \leq \altern{\fdpi} + c_2$ otherwise, where $c_1 \geq \Delta_0$ and $c_2 > 0$ are absolute constants that do not depend on $\Delta$.
\end{lemma}
For any $\eps \in (0,1]$ and optimal plan $\fdpi^*$, if $\plancost{\fdpi^*} \leq c_1/\eps$, then $\fdpi^*$ is obviously $(c_1/\eps)$-tame (recalling that $s_A$, $s_B$, $t_A$, $t_B$ are in $X$). Otherwise, by \lemref{no-faraway}, there exists a $(c_1/\eps)$-tame $(\fd{s},\fd{t})$-plan of cost at most $\plancost{\fdpi^*} + c_1 \leq (1+\eps)\plancost{\fdpi^*}$. Hence, using \lemref{few-parking} to bound the number of moves, we obtain:
\begin{corollary}
\corlab{eps-faraway}
Given reachable configurations $\fd{s},\fd{t} \in \fdfreesp$ and $\eps \in (0,1]$, there exists a decoupled, kissing, $(c_1/\eps)$-tame $(\fd{s},\fd{t})$-plan $\fdpi$ with $\plancost{\fdpi} \leq (1+\eps)\plancost{\fdpi^*}$ and $\altern{\fdpi} \leq c_2(\plancost{\fdpi^*} + 1)$, where 
$c_1 \geq \Delta_0$ and $c_2 > 0$ are absolute constants that do not depend on $\eps$.
\end{corollary}

Let $\fdpi$ be an optimal, decoupled, kissing $(\fd{s},\fd{t})$-plan. By \corref{no-sanctum-parkings}, we can assume that no robot is parked inside the sanctum of a corridor. Let $\ell \assign \altern{\fdpi}$ and let $(R_1,\pi_1,p_1),\ldots, (R_\ell,\pi_\ell,p_\ell)$ be the sequence of moves of $\fdpi$. Let $i$ (resp., $j$), $1 < i \leq j < \ell$,
be the smallest (resp., largest) index such that $p_i,p_j$ are $(\Delta-4)$-far, \ie, $p_i$ (resp.\ $p_j$) is the first (resp.\ last) $(\Delta-4)$-far parking place in $\fdpi$.
If there are no such indices, then $\fdpi$ is $\Delta$-tame and we are done. So suppose $i,j$ exist. 
Note that it can be that $i=j$. By the definitions of corridors and sanctums, $p_i$ and $p_j$ do not lie inside a corridor $\corridor \in \corridors$ because any point in $\corridor \setminus \sanctum{\corridor}$ is $(\Delta_0-4)$-close, $p_i,p_j$ are $(\Delta-4)$-far, and $\Delta \geq \Delta_0$. Therefore there is a revolving area around each of $p_i$ and $p_j$ by \lemref{ra-exists}.

\ifabbrv
\else
The proof of \lemref{no-faraway} is based on the following observation, which is proved in \lemref{between-ras}.
\fi
Let $\fd{s} = (s_A,s_B), \fd{t} = (t_A,t_B) \in \fdfreesp$ be reachable kissing configurations with the property that there exist $r^-,r^+ \in \freesp$ such that $s_A,s_B \in \RA(r^-), t_A,t_B \in \RA(r^+)$, and $r^-,r^+$ are $3$-far. Then there exists a decoupled, kissing $(\fd{s},\fd{t})$-plan $\widetilde{\fdpi}$ with $\plancost{\widetilde{\fdpi}} \leq \geodesic{s_A}{t_A} + \geodesic{s_B}{t_B} + O(1)$ and $\altern{\widetilde{\fdpi}} = O(1)$, and all parking places in $\widetilde{\fdpi}$ lie in $\RA(r^-)$ or $\RA(r^+)$. Since $\fdpi$ is a kissing plan, there are kissing configurations $q = (q_A,q_B)$ and $q' = (q_A',q_B')$ on moves $i$ and $j$.
\ifabbrv
\else
If $q_A,q_A',q_B,q_B'$ each is $(\Delta-2)$-close and lies in a revolving area then \lemref{no-faraway} follows from \lemref{between-ras} but we may not be so lucky---$q_A$ or $q_A'$ may not be $(\Delta-2)$-close or may not lie in revolving areas,
so the proof is much more involved.
\fi

\ifabbrv
\else
The proof of \lemref{between-ras}, however,
\fi
can be slightly adapted to prove the following variant:

\begin{lemma}
\lemlab{between-ras-variant}
Let $\fd{u} = (u_A,u_B),  \fd{v} = (v_A,v_B) \in \fdfreesp$ be two configurations such that there exist four points 
	$\overline{u}_A,\overline{u}_B,\overline{v}_A,\overline{v}_B \in \freesp$ with 
	$u_A \in \RAp{\overline{u}_A}, u_B \in \RAp{\overline{u}_B}, v_A \in \RAp{\overline{v}_A}, 
	v_B \in \RAp{\overline{v}_B}$, and $\overline{u}_A,\overline{u}_B,\overline{v}_A,\overline{v}_B$ are $3$-far, then there exists a decoupled $(\fd{u},\fd{v})$-plan $\widetilde{\fdpi}$ with $\plancost{\widetilde{\fdpi}} \leq \geodesic{u_A}{u_B} + \geodesic{v_A}{v_B} + 78$, $\altern{\widetilde{\fdpi}} \leq 40$, and all parking places of $\widetilde{\fdpi}$ lie in the four revolving areas.
\end{lemma}
Using \lemref{between-ras-variant}, we can prove the following weaker version of \lemref{no-faraway}, which guarantees that $\widetilde{\fdpi}$ is $\Delta$-tame but does not guarantee the kissing property.

\begin{lemma}
\lemlab{no-faraway-maybe-kissing}
Given reachable configurations $\fd{s},\fd{t} \in \fdfreesp$, let $\fdpi$ be a decoupled, kissing $(\fd{s},\fd{t})$-plan. For any parameter $\Delta \geq \Delta_0$, there exists a decoupled $\Delta$-tame $(\fd{s},\fd{t})$-plan $\widetilde{\fdpi}$ such that $\widetilde{\fdpi} = \fdpi$ if $\plancost{\fdpi} \leq \Delta$ and $\plancost{\widetilde{\fdpi}} \leq \plancost{\fdpi} + c_1$ and $\altern{\widetilde{\fdpi}} \leq \altern{\fdpi} + c_2$ otherwise, for some absolute constants $c_1,c_2 > 0$ independent of $\Delta$.
\end{lemma}

\begin{proof}
Let $p_i,p_j$ be as defined above. Suppose $R_i = B$, \ie, $A$ moves from $p_{i-2}$ to $p_i$ in the $(i-1)$-st move along $\pi_{i-1}$ and is parked at $p_i$, then $B$ moves from $p_{i-1}$ to $p_{i+1}$ along $\pi_i$ in the $i$-th move. Let $u_A$ be the last point along $\pi_{i-1}$ that is $(\Delta-4)$-close, \ie, $\linfd{\plpt{\pi_{i-1}}{u_A,p_i}}{\enVerts} \geq \Delta-4$. Recall that $p_{i-2}$ is $(\Delta-4)$-close. Note that $u_A$ may be $p_{i-2}$ or $p_i$, and $u_A$ is $(\Delta-4)$-tight. Since $p_i$ does not lie in a corridor, we claim that $u_A$ also does not lie inside a corridor. Indeed if $u_A \in \corridor$ for some $\corridor \in \corridors$, then $A$ exits $\corridor$ at some point $\xi \in \plpt{\pi_{i-1}}{u_A,p_i}$ but then $\linfd{\xi}{\enVerts} < \Delta_0 - 4 \leq \Delta-4$, contradicting that $u_A$ is the last $(\Delta-4)$-close point on $\pi_{i-1}$. Since $u_A$ does not lie in a corridor, by \lemref{ra-exists}, there is a 
	$(\Delta-5,\Delta-3)$-close point $\overline{u}_A \in \freesp$ such that $u_A \in \RAp{\overline{u}_A}$.

Next, $B$ kisses $A$ parked at $p_i$ during the $i$-th move. Since $p_i$ is $(\Delta-4)$-far, $\pi_i$ contains a $(\Delta-6)$-far point. If $\pi_i \cap (u_A+2\Box) = \varnothing$, let $u_B$ be the last $(\Delta-6)$-close point on $\pi_i$ if there exists one and $u_B = p_{i-1}$ otherwise (\ie, all points on $\pi_i$ are $(\Delta-6)$-far). Then $u_B$ is $(\Delta-6)$-tight. On the other hand, if $\pi_i \cap (u_A + 2\Box) \neq \varnothing$,
let $u_B$ be the first intersection point of $\pi_i$ with $u_A+2\Box$, \ie, $\plpt{\pi_i}{p_{i-1},u_B} \cap \Int(u_A + 2\Box) = \varnothing$. Since $u_A$ is $(\Delta-4)$-tight, $u_B$ is $\dclose{\Delta-6}{\Delta-2}$.

Since $p_i$ and $u_A$ are not inside a corridor, a similar argument as above implies that $u_B$ is also not in a corridor. Therefore there exists a $(\Delta-7,\Delta-1)$-close 
point $\overline{u}_B$ such that $u_B \in \RAp{\overline{u}_B}$. Set $\fd{u} = (u_A,u_B)$.

Without loss of generality, assume that $R_j = B$. Then using a symmetric argument, we find points $v_A \in \pi_{j+1}$ such that $v_A \in \RAp{\overline{v}_A}$ and $v_A$ is $(\Delta-4)$-close, and $v_B \in \pi_j$ such that $v_B$ is $\dclose{\Delta-6}{\Delta-4}$ and $v_B \in \RAp{\overline{v}_B}$, for some $(\Delta-7,\Delta-1)$-close points $\overline{v}_A,\overline{v}_B \in \freesp$. Set $\fd{v} = (v_A,v_B)$.

Since $\overline{u}_A,\overline{u}_B,\overline{v}_A,\overline{v}_B$ each is $(\Delta-7)$-far and $\Delta - 7 \geq \Delta_0 - 7 \geq 3$, each is $(3,\Delta-1)$-close.
Let $\fd{\psi}=(\psi_A,\psi_B)$ be the decoupled $(\fd{u},\fd{v})$-plan according to \lemref{between-ras-variant}, with $\moveseq{\fd{\psi}} = (S_1, \psi_1, q_1), \ldots, (S_h, \psi_h, q_h)$. We obtain a new $(\fd{s},\fd{t})$-plan $\widetilde{\fdpi}$ by replacing $\plpt{\pi_A}{u_A,v_A}$ and $\plpt{\pi_B}{u_B,v_B}$ with $\psi_A$ and $\psi_B$, respectively. More precisely,
\begin{align*}
\moveseq{\widetilde{\fdpi}} =&
(R_1, \pi_1, p_1), \ldots, (R_{i-2}, \pi_{i-2}, p_{i-2}), (A, \plpt{\pi_{i-1}}{p_{i-2},u_A}, p_{i-1}), (B, \plpt{\pi_i}{p_{i-1},u_B}, u_A) \\
&\circ \moveseq{\fd{\psi}} \circ \\
&(B, \plpt{\pi_j}{v_B,p_{j+1}}, v_A), (A, \plpt{\pi_{j+1}}{v_A,p_{j+2}}, p_{j+1}), (R_{j+2}, \pi_{j+2}, p_{j+2}), \ldots, (R_\ell, \pi_\ell, p_\ell).
\end{align*}
	It is easily seen that $\widetilde{\fdpi}$ is a (feasible) $(\fd{s},\fd{t})$-plan. 
	By \lemref{between-ras-variant}, all parking places in $\fd{\psi}$ and thus in $\widetilde{\fdpi}$ are $\Delta$-close,
$\plancost{\widetilde{\fdpi}} \leq \plancost{\fdpi} + 78$, and $\altern{\widetilde{\fdpi}} \leq \altern{\fdpi} + 40$. 
\end{proof}

A similar argument as for \corref{eps-faraway}, but using \lemref{no-faraway-maybe-kissing},
implies the following corollary.

\begin{corollary}
\corlab{eps-faraway-maybe-kissing}
Given reachable configurations $\fd{s},\fd{t} \in \fdfreesp$ and $\eps \in (0,1]$, there exists a decoupled $(c_1/\eps)$-tame $(\fd{s},\fd{t})$-plan $\fdpi$ with $\plancost{\fdpi} \leq (1+\eps)\plancost{\fdpi^*}$ and $\altern{\fdpi} \leq c_2(\plancost{\fdpi^*} + 1)$, where 
	$c_1,c_2 > 0$ are absolute constants that do not depend on $\eps$.
\end{corollary}

Returning to the proof of \lemref{no-faraway}, we first briefly sketch the idea.
Let $\lambda_i \in [0,1]$ (resp., $\lambda_j \in [0,1]$) be the earliest (resp., latest) time during the move $i$ (resp., $j$) such that $\fdpi(\lambda_i)$ (resp., $\fdpi(\lambda_j)$) is a kissing configuration; there exists such a value since $\fdpi$ is kissing. If $R_i = B$ then $\pi_A(\lambda_i) = p_i$ and $\pi_B(\lambda_i) \in \pi_i$, and $\pi_A(\lambda_i) \in \pi_i$ and $\pi_B(\lambda_i) = p_i$ otherwise; the same holds for $\lambda_j$. We similarly define $\lambda_{i-1}$ (resp., $\lambda_{j+1}$) to be the latest (resp., earliest) time during the move $i-1$ (resp., $j+1$) such that $\fdpi(\lambda_{i-1})$ (resp., $\fdpi(\lambda_{j+1})$) is a kissing configuration; if no such configuration exists, then $i-1=1$ (resp., $j+1=\ell$) and we set $\lambda_{i-1} = 0$ (resp., $\lambda_{j+1} = 1$). Then $0 \leq \lambda_{i-1} \leq \lambda_i \leq \lambda_j \leq \lambda_{j+1} \leq 1$.
If $i=j$ then $\pi_A(\lambda_i,\lambda_j)$ and $\pi_B(\lambda_i,\lambda_j)$ are points. For $0 \leq r \leq 3$, let $a_r \assign \pi_A(\lambda_{i-1+r})$ and $b_r \assign (\lambda_{i-1+r}$). Without loss of generality, $R_{i-1} = A$ and $R_i = B$, so $A$ moves first from $a_0 = p_{i-2}$ to $a_1 = p_i$ then $B$ moves from $b_0$ to $b_1$ in the given motion plan $\fdpi(\lambda_i,\lambda_j)$.
\ifabbrv
\else
The proof of \lemref{no-faraway} is divided into two cases:
\fi
\begin{enumerate}[(i)]
\item There exists a $(\Delta-6)$-close point on $\plpt{\pi_A}{\lambda_i,\lambda_j}$ or $\plpt{\pi_B}{\lambda_i,\lambda_j}$, say, on $\plpt{\pi_A}{\lambda_i,\lambda_j}$. In this case, we find two $(\Delta-6)$-close points $q^-,q^+$ on $\plpt{\pi_A}{\lambda_i,\lambda_j}$ and modify $\plpt{\pi_A}{\lambda_{i-1},\lambda_{j+1}}$ and $\plpt{\pi_B}{\lambda_{i-1},\lambda_{j+1}}$, using the above observation, so that $A$ and $B$ are parked at $\Delta$-close points near $a_0,b_0,a_3,b_3,q^-$, or $q^+$ and they lie in revolving areas. The surgery on $\pi_A,\pi_B$ increases their lengths by $O(1)$ and adds $O(1)$ new alternations{\ifabbrv\else~(\subsecref{dclose})\fi}.

\item There is no $(\Delta-6)$-close point on $\plpt{\pi_A}{\lambda_i,\lambda_j}$ or $\plpt{\pi_B}{\lambda_i,\lambda_j}$. In this case, we find $\Delta$-close parking places in the vicinity of $\plpt{\pi_A}{\lambda_{i-1},\lambda_i}$, $\plpt{\pi_B}{\lambda_{i-1},\lambda_i}$, $\plpt{\pi_A}{\lambda_{j},\lambda_{j+1}}$, and $\plpt{\pi_B}{\lambda_{j},\lambda_{j+1}}$ and again modify $\plpt{\pi_A}{\lambda_{i-1},\lambda_{j+1}}$ and $\plpt{\pi_B}{\lambda_{i-1},\lambda_{j+1}}$. We cannot always guarantee the existence of revolving areas that contain parking places. Therefore the surgery as well as the analysis is more involved. Nevertheless, we are able to argue that the increase in the cost of the plan and in the number of alternations is $O(1)${\ifabbrv\else~(\subsecref{dfar})\fi}.
\end{enumerate}

\ifabbrv
\else
\subsection{Auxiliary lemmas}
\subseclab{auxiliary}
We next prove a sequence of lemmas for the that show the existence of near-optimal plans so that parking places lie in revolving areas containing initial or final placements. For simplicity, we do not try to minimize the error terms. They are used heavily when proving \lemref{no-faraway}. For a point $q \in \bd \RA(p)$, we 
define $\antip{q,p} \assign (2x(p) - x(q), 2y(p) - y(q))$ as the point of intersection of $\bd\RA(p)$ 
with the open ray emanating from $q$ towards $p$. Note that $\linfnorm{q}{\antip{q,p}} = 2$. We may omit 
the second parameter of $\antip{\cdot,\cdot}$ when it is clear from context.

\begin{figure}
\centering
\includegraphics[scale=0.90]{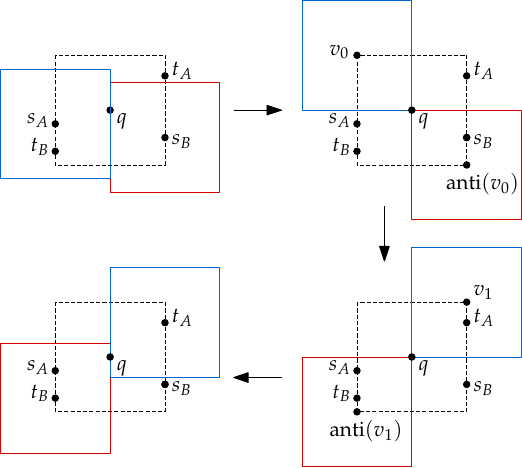}
\caption[Illustration of the plan described in the proof of \lemref{revolve}.]{Illustration of the plan described in the proof of \lemref{revolve}, where $s_A,t_B$ lie on the left edge of $\RA(q)$ and $t_A,s_B$ lie on the right edge of $\RA(q)$.}
\figlab{revolve}
\end{figure}

\begin{lemma}
\lemlab{revolve}
For any revolving area $\RA \assign q \oplus \Box \subseteq \freesp$ and configurations $\fd{s} = (s_A,s_B), \fd{t} = (t_A,t_B) \in \fdfreesp$ with $s_A,s_B,t_A,t_B \in \RA$, there exists a kissing $(\fd{s},\fd{t})$-plan $\fdpi$ such that all parking places lie in $\RA$, $\pi_A,\pi_B \subset \RA$, $\plancost{\fdpi} \leq 12$, and $\altern{\fdpi} \leq 8$.
\end{lemma}

\begin{proof}
Let $\pi_A$ be the shortest $(s_A,t_A)$-path on $\bd \RA$; $\pathcost{\pi_A} \leq 4$. Without loss of generality, assume $s_A$ lies on the left edge of $\RA$, $s_B$ lies on the right edge of $\RA$, and $\pi_A$ traces $\bd\RA$ from $s_A$ to $t_A$ in clockwise direction. See \figref{revolve}. If $t_A$ lies on the same edge as $s_A$, then $t_B$ lies on the same edge as $t_B$ and we move $A$ to $t_A$ and $B$ to $t_B$ with total length at most $4$. Otherwise, we make $A,B$ $y$-separated by moving $A$ up to the top-left vertex $v_0$ of $\RA$ (above $s_A$) and moving $B$ down to the bottom-right vertex $\antip{v_0}$ of $\RA$ (below $s_B$). This consists of two moves with total length at most $4$.

If $t_A$ lies on the top edge of $\RA$, $t_B$ lies on the bottom edge of $\RA$ and we move $A$ right to $t_A$ and $B$ left to $t_B$. This consists of two more moves with total length of at most $4$, so the overall length is at most $8$. Otherwise, $t_A,t_B$ lie on the right and left edges of $\RA$, respectively, and we move $A$ right to the top-right vertex $v_1$ of $\RA$ then move $B$ to the bottom-left vertex $\antip{v_1}$ of $\RA$ for total length $4$. Finally, we move $A$ down to $t_A$ and move $B$ up to $t_B$ with total length $4$. The overall length is $12$ in this case there are at most six moves.
\end{proof}

The following lemma proves that if one robot lies in a revolving area far enough from the set of vertices $\enVerts$, there is a simple near-optimal kissing plan that moves the other robot between any two points that do not lie in the interior of the revolving area.

\begin{figure}
\centering
\includegraphics[scale=0.90]{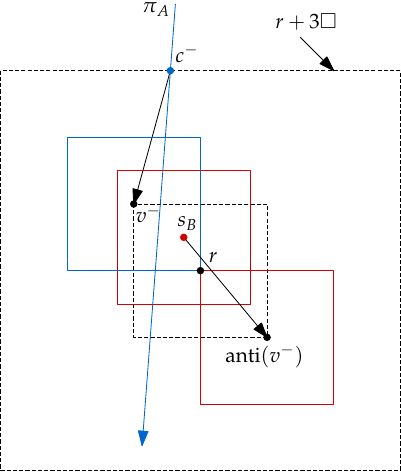}
\caption[Illustration of the plan described in the proof of \lemref{through-ra}.]{Illustration of the plan described in the proof of \lemref{through-ra}. Only the first part, from $(s_A,s_B)$ to $(v^-,\antip{v^-})$, is shown. $A$, centered at $v^-$ and depicted as blue, kisses $B$, centered at $\antip{v^-}$ and depicted as red.}
\figlab{through-ra}
\end{figure}

\begin{lemma}
\lemlab{through-ra}
Let $r \in \freesp$ be a $3$-far point such that $\RA \assign \RA(r) \subseteq \freesp$. Let $s_B,t_B \in \RA$ and $(s_A,s_B),(t_A,t_B) \in \fdfreesp$ such that $s_A,t_A$ lie in the same component of $\freesp$. Then there exists a kissing $((s_A,s_B),(t_A,t_B))$-plan $\fdpi$ such that all parking places lie in $\RA$, $\plancost{\fdpi} \leq \geodesic{s_A}{t_A} + 24$, and $\altern{\fdpi} \leq 14$.
\end{lemma}

\begin{proof}
Let $P$ be a shortest $(s_A,t_A)$-path in $\freesp$. If $P \subset \freept{s_B}$, then define $\fdpi$ as the simple plan where $A$ moves directly from $s_A$ to $t_A$ along $P$, and then moves $B$ from $s_B$ to $t_B$ along the shortest path in $(r + 2\Box) \cap \freept{t_A}$ (such a path must exist). In this case $\plancost{\fdpi} \leq \geodesic{s_A}{s_B} + 4$ and $\altern{\fdpi} \leq 2$ and we are done. So suppose otherwise.

Let $C$ be the component of $(r + 3\Box) \cap \freesp$ containing $r$ (and $s_B,t_B$). Since $s_B,t_B \in \RA(r)$, $s_B+2\Box$,$t_B+2\Box \subset \C$. By assumption, $P$ intersects $\Int(s_B+2\Box)$, so $P$ intersects $\Int(C)$. Let $c^-$ and $c^+$ be the first and last points on $P$ such that $c^-,c^+ \in C$. Note that $c^-,c^+$ may be $s_A,t_A$, respectively. Let $v^-,v^+$ be the closest vertices of $\bd \RA$ to $c^-,c^+$, respectively.
See \figref{through-ra}.
Note that $s_A,t_A \notin \Int(\RA)$ since $(s_A,s_B),(t_A,t_B) \in \fdfreesp$ and $s_B,t_B \in \RA$.

Let $\fdpi_0$ be the kissing $((s_A,s_B),(v^-,\antip{v^-}))$-plan where we move $B$ to $\antip{v^-}$ along segment $s_B\antip{v^-}$ then move $A$ to $v^-$ along $\plpt{P}{s_A,c^-} \ccat c^-v^-$. Then $\plancost{\fdpi_0} \leq \pathcost{\plpt{P}{s_A,c^-}} + 4\sqrt{2}$ and $\altern{\fdpi_0} \leq 3$.
A similar construction gives a kissing $((v^+,\antip{v^+}),\allowbreak(t_A,t_B))$-plan $\fdpi_2$ with $\plancost{\fdpi_2} \leq \pathcost{\plpt{P}{c^+,t_A}} + 4\sqrt{2}$ and $\altern{\fdpi_2} \leq 3$. Next, let $\fdpi_1$ be the kissing $((v^-,\antip{v^-}),(v^+,\antip{v^+}))$-plan from \lemref{revolve} with $\altern{\fdpi_1} \leq 8$ and $\plancost{\fdpi_1} \leq 12$. Putting everything together, $\fdpi \assign \fdpi_1 \circ \fdpi_2 \circ \fdpi_3$ is a kissing $((s_A,s_B),(t_A,t_B))$-plan with $$\plancost{\fdpi} \leq \pathcost{\plpt{P}{s_A,c^-}} + \pathcost{\plpt{P}{c^+,t_A}} + 8\sqrt{2} + 12 \leq \geodesic{s_A}{t_A} + 24$$ and $\altern{\fdpi} \leq 3 + 3 + 8 \leq 14$.
\end{proof}

With the previous lemma in hand, we describe simple near-optimal kissing plans that move both robots from one revolving area to another, provided they are far enough from vertices of $\freesp$.

\begin{lemma}
\lemlab{between-ras}
Let $r_1,r_2 \in \freesp$ be $3$-far points such that $\RA_1 \assign \RA(r_1),\RA_2 \assign \RA(r_2) \subset \freesp$. Let $\fd{s} = (s_A,s_B),\fd{t} = (t_A,t_B) \in \fdfreesp$ be two kissing configurations such that $s_A,s_B \in \RA_1$ and $t_A,t_B \in \RA_2$. Then there exists a kissing $(\fd{s},\fd{t})$-plan $\fdpi$ such that all parking places lie in $\RA_1$ or $\RA_2$, $$\plancost{\fdpi} \leq \geodesic{s_A}{t_A} + \geodesic{s_B}{t_B} + 78,$$ and $\altern{\fdpi} \leq 40$.
\end{lemma}

\begin{proof}
Let $v_1 \in \bd \RA_1$ and $v_2 \in \bd\RA_2$ be any two points such that $\linfnorm{v_1}{v_2} \geq 2$. (Note that $v_1,v_2$ must exist.) Then $(v_1,v_2) \in \fdfreesp$. Let $\fdpi_0$,$\fdpi_3$ be the kissing $((s_A,s_B),(v_1,\antip{v_1,r_1}))$-plan and $((\antip{v_2,r_2},v_2),(t_A,t_B))$-plan by \lemref{revolve}, respectively. Next, let $\fdpi_1,\fdpi_2$ be the kissing $((v_1,\antip{v_1,r_1}),(v_1,v_2))$-plan and $((v_1,v_2),(\antip{v_1,r_2},v_2)$-plan by \lemref{through-ra}, respectively. Putting everything together, $\fdpi \assign \fdpi_0 \circ \fdpi_1 \circ \fdpi_2 \circ \fdpi_3$ is a kissing $((s_A,s_B),(t_A,t_B))$-plan with
\begin{align*}
\plancost{\fdpi} = \sum_{i=0}^3 \plancost{\fdpi_i} &\leq 12 + (\geodesic{\antip{v_1,r_1}}{v_2} + 24) + (\geodesic{v_1}{\antip{v_1,r_2}} + 24) + 12 \\
&\leq \geodesic{\antip{v_1,r_1}}{v_2}+ \geodesic{v_1}{\antip{v_1,r_2}} + 72 \\
&\leq \geodesic{s_A}{t_A} + \geodesic{s_B}{t_B} + 78,
\end{align*}
where the last inequality follows because $\geodesic{p}{q} = \sqrt{2} < 3/2$ for any two points $p,q$ on the boundary of a revolving area. Finally, $$\altern{\fdpi} = \sum_{i=0}^3 \altern{\fdpi_i} \leq 6 + 14 + 14 + 6 = 40.$$
\end{proof}

The following lemma shows that if the initial and final configurations are kissing configurations, the paths traversed by the robots in any optimal plan have similar lengths.

\begin{lemma}
\lemlab{a-b-similar}
Let $\fd{s},\fd{t} \in \fdfreesp$ be two kissing configurations such that $s_A,s_B,t_A,t_B$ lie in revolving areas $\RA(s_A),\RA(s_B),\RA(t_A),\RA(t_B)$, respectively, where $r_{s_A},r_{s_B},r_{t_A},r_{t_B} \in \freesp$ each is $3$-far, and let $\fdpi$ be an optimal $(\fd{s},\fd{t})$-plan. Then $\big| \pathcost{\pi_A} - \pathcost{\pi_B} \big| \leq 150$.
\end{lemma}

\begin{proof}
Suppose not. Without loss of generality, $\pathcost{\pi_A} < \pathcost{\pi_B}-150$.
Since $r_{s_A},r_{s_B},r_{t_A},r_{t_B}$ are $3$-far, we have that segments $r_{s_A}r_{s_B}$ and $r_{t_A}r_{t_B}$ lie in $\freesp$.
Let $\fdpi_0$ be the $((s_A,s_B),\allowbreak(s_A,\antip{s_A,r_{s_A}})$-plan from \lemref{through-ra}, let $\fdpi_1$ be the $((s_A,\antip{s_A,r_{s_A}}),\allowbreak(t_A,\antip{t_A,r_{t_A}}))$-plan from \lemref{between-ras}, and let $\fdpi_2$ be the $((t_A,\antip{t_A,r_{t_A}}),\allowbreak(t_A,t_B))$-plan from \lemref{through-ra}. Then $\fdpi' \assign \fdpi_0 \circ \fdpi_1 \circ \fdpi_2$ is a $(\fd{s},\fd{t})$-plan with
\begin{align*}
\plancost{\fdpi'}=& \sum_{i=0}^2 \plancost{\fdpi_i}\\
=& (\geodesic{s_B}{\antip{s_A,r_{s_A}}})+24)\\
&+(\geodesic{s_A}{t_A}+\geodesic{\antip{s_A,r_{s_A}}}{\antip{t_A,r_{t_A}}} + 78)\\
&+(\geodesic{\antip{t_A,r_{t_A}}}{t_B}+24)\\
\leq& (\abs{s_B\antip{s_A,r_{s_A}}} + \abs{s_As_B}) \\
&+ 2\geodesic{s_A}{t_A} + \abs{s_A\antip{s_A,r_{s_A}}} + \abs{t_A\antip{t_A,r_{t_A}}}\\
&+ (\abs{t_At_B}+\abs{\antip{t_A,r_{t_A}}t_A}) + 126\\
\leq& 9 + 2(\pathcost{\pi_A}) + 3 + 3 + 9 + 126 < 2\pathcost{\pi_A} + 150
\end{align*}
Then, by assumption, we have $\plancost{\fdpi'} < \pathcost{\pi_A} + \pathcost{\pi_B} = \plancost{\fdpi}$, which contradicts the optimality of $\fdpi$.
\end{proof}

\begin{figure}
\centering
\includegraphics[scale=1.0]{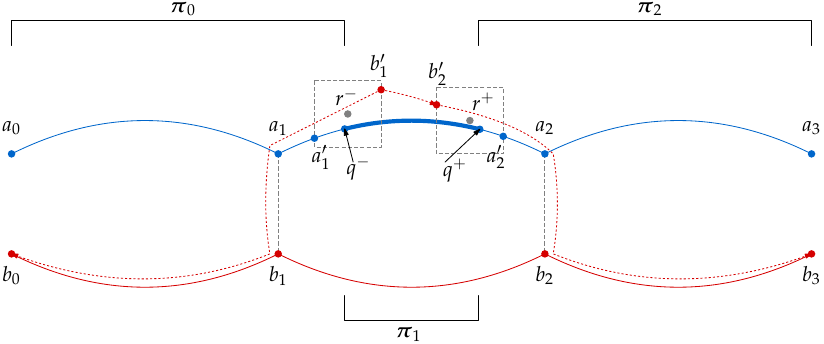}
\caption[Abstract diagram of paths $\plpt{\pi_A}{a_0,a_3},\plpt{\pi_B}{b_0,b_3}$ from \subsecref{dclose}.]{Abstract diagram of the paths $\plpt{\pi_A}{a_0,a_3},\plpt{\pi_B}{b_0,b_3}$ and various points as defined in \subsecref{dclose}. The thick pathlets are $(\Delta-6)$-far, the black dashed lines represent segments in $\freesp$ between kissing configurations, and the grey dashed squares are revolving areas. The dotted red path represents the path of $B$ followed during the new plan.}
\figlab{dclose-middle}
\end{figure}

\subsection{Case (i): Existence of a $(\Delta-6)$-close point on $\pi_A(\lambda_i,\lambda_j)$ or $\pi_B(\lambda_i,\lambda_j)$.}
\subseclab{dclose}
We next prove case (i) of \lemref{no-faraway}. For concreteness, suppose there is a $(\Delta-6)$-close point on $\pi_A(\lambda_i,\lambda_j)$; the other case is similar. Let $q^-$ (resp., $q^+$) be the first (resp., last) point on $\pi_A(\lambda_i,\lambda_j)$ which is $(\Delta-6)$-far and thus $(\Delta-6)$-tight. Then there are revolving areas $\RA^- \assign r^- + \Box$ and $\RA^+ \assign r^+ +\Box$ containing $q^-,q^+$ respectively, for points $r^-,r^+ \in \freesp$. The high-level idea for the plan is to move $A,B$ from $(a_0,b_0)$ into $\RA^-$, then move both to $\RA^+$, then finally move both to $(a_3,b_3)$. All parking places during the plan are near $a_0,b_0,a_3,b_3,q^-,q^+$, so they are $\Delta$-close, as desired.

Let $a_1'$ be the first point on $\plpt{\pi_A}{a_1,q^-} \cap \RA^-$, and set $b_1' \assign \antip{a_1',r^-}$, and define $a_2',b_2' \in \RA^+$ similarly. $r^-,r^+$ are $\dclose{\Delta-7}{\Delta-5}$, so $a_1',b_1',a_2',b_2'$ are $\dclose{\Delta-8}{\Delta-4}$. We next describe a decoupled, kissing $((a_0,b_0),(a_1',b_1'))$-plan $\fdpi_0$. There are two cases.
\begin{enumerate}
\item Suppose $\plpt{\pi_A}{a_1,q^-} \subset \freept{b_0}$. Since $R_{i-1} = A$, $\plpt{\pi_A}{a_0,a_1} \subset \freept{b_0}$, and hence $\plpt{\pi_A}{a_1,q^-} \subset \freept{b_0}$. Let $a_1'$ be the first point on $\plpt{\pi_A}{a_1,q^-} \cap \RA^-$, and set $b_1' \assign \antip{a_1',r^-}$. See \figref{dclose-middle}. In this case, we first move $A$ from $a_0$ directly to $a_1'$ in a single move. Then we move $B$ from $b_0$ to $b_1'$ using the decoupled, kissing plan from \lemref{through-ra}. Let $\fdpi_0$ be the resulting decoupled, kissing $((a_0,b_0),(a_1',b_1'))$-plan.

\item Otherwise, $\plpt{\pi_A}{a_1,q^-} \not\subset \freept{b_0}$. Then $\plpt{\pi_A}{a_1,q^-}$ intersects the interior of $b_0 + 2\Box$. Since all points on $\plpt{\pi_A}{a_1,q^-}$ are $(\Delta-6)$-far, $b_0$ is $(\Delta-4)$-far. Furthermore, $b_0 = p_{i-1}$ is $(\Delta-6)$-close by definition of $p_i$; $b_0$ is $\dclose{\Delta-8}{\Delta-4}$. Let $(a^*,b^*)$ be a decoupled, kissing configuration on the revolving area containing $b_0$. Then $a^*,b^*$ are $\dclose{\Delta-10}{\Delta-2}$. Let $\fdpi_0$ be the decoupled, kissing $((a_0,b_0),(a_1',b_1'))$-plan obtained by first applying the kissing $((a_0,b_0),(a^*,b^*))$-plan from \lemref{through-ra} then applying the $((a^*,b^*),(a_1',b_1'))$-plan from \lemref{between-ras}.
\end{enumerate}

It is easy to verify that, in either case, all parking places in $\fdpi_0$ except $a_0,b_0,a_1',a_1'$ are $\dclose{\Delta-10}{\Delta-2}$, $\plancost{\fdpi_0} \leq \pathcost{\plpt{\pi_A}{a_0,a_1}} + \pathcost{\plpt{\pi_B}{b_0,b_1}} + 2\pathcost{\plpt{\pi_A}{a_1,a_1'}} + 200$ and $\altern{\fdpi_0} \leq 40$. Similarly, we construct a decoupled, kissing $((a_2',b_2'),(a_3,b_3))$-plan $\fdpi_2$ where all parking places except $a_2',b_2',a_3,b_3$ are $\dclose{\Delta-10}{\Delta-2}$, $\plancost{\fdpi_2} \leq \pathcost{\plpt{\pi_A}{a_2,a_3}} + \pathcost{\plpt{\pi_B}{b_2,b_3}} + 2\pathcost{\plpt{\pi_A}{a_2',a_2}} + 200$ and $\altern{\fdpi_0} \leq 40$. Let $\fdpi_1$ be the decoupled, kissing $((a_1',b_1'),(a_2',b_2'))$-plan from \lemref{between-ras} with $$\plancost{\fdpi_1} \leq \geodesic{a_1'}{a_2'} + \geodesic{b_1'}{b_2'} + 78 \leq 2\pathcost{\plpt{\pi_A}{a_1',a_2'}} + 150$$ and $\altern{\fdpi_1} \leq 40$. Then $\fdpi' \assign \fdpi_0 \circ \fdpi_1 \circ \fdpi_2$ is a decoupled, kissing $((a_0,b_0),(a_3,b_3))$-plan with
\begin{align*}
\plancost{\fdpi'} &\leq \pathcost{\plpt{\pi_A}{a_0,a_3}} + \pathcost{\plpt{\pi_B}{b_0,b_1}} + \pathcost{\plpt{\pi_A}{a_1,a_2}} + \pathcost{\plpt{\pi_B}{b_2,b_3}} + 550 \\
&\leq \pathcost{\plpt{\pi_A}{a_0,a_3}} + \pathcost{\plpt{\pi_B}{b_0,b_3}} + 750,
\end{align*}
where the last inequality follows the fact $\pathcost{\plpt{\pi_A}{a_1,a_2}} \leq \pathcost{\plpt{\pi_B}{b_1,b_2}} + 150$ by \lemref{a-b-similar}. Furthermore, all parking places besides $a_0,b_0,a_3,b_3$ are $(\Delta-2)$-close. We replace $\fdpi(\lambda_i,\lambda_j)$ with $\fdpi'$ in $\fdpi$, which completes the proof.

\subsection{Case (ii): No $(\Delta-6)$-close point on $\pi_A(t_i,t_j)$ or $\pi_B(t_i,t_j)$.}
\subseclab{dfar}

\begin{figure}
\centering
\includegraphics[scale=1.0]{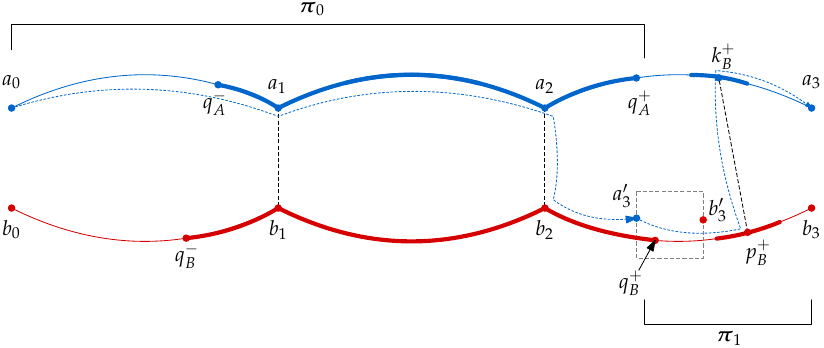}
\caption[Abstract diagram of paths $\plpt{\pi_A}{a_0,a_3},\plpt{\pi_B}{b_0,b_3}$ from \subsecref{dfar} when $p_B^+$ is \emph{distant}.]{Abstract diagram of the paths $\plpt{\pi_A}{a_0,a_3},\plpt{\pi_B}{b_0,b_3}$ and various points as defined in \subsecref{dfar} when $p_B^+$ is \emph{distant}. The thick pathlets are $(\Delta-6)$-far, the black dashed lines represent segments in $\freesp$ between kissing configurations, and the grey dashed square is a revolving area containing $q_B^+$. The dotted blue path represents the path of $A$ followed during the new plan; $B$ mainly follows $\pi_B$.}
\figlab{dfar-middle-distant}
\end{figure}

Without loss of generality, $R_{i-1} = A$ and $R_i = B$. For concreteness, we assume $R_j = A$ and $R_{j+1} = B$; the other case is similar.
Then $\plpt{\pi_A}{a_0,a_1} \subset \freept{b_0}$ and $\plpt{\pi_B}{b_2,b_3} \subset \freept{a_3}$.

The high-level idea is to try to follow the approach taken in the previous proof and find revolving areas centered at $(\Delta-c)$-close points near $\plpt{\pi_A}{a_0,a_1}$ and $\plpt{\pi_B}{b_2,b_3}$ that $B$ and $A$ can reach, respectively, ``without straying too far'' from their original paths in $\fdpi$, for a sufficiently large constant $c$. If we find such revolving areas, we first move $A$ then $B$ to the former revolving area, then move them both to the latter revolving area, then finally move $A$ then $B$ to $(a_3,b_3)$. Otherwise, if we are unable to find such revolving areas, we instead find and use a sequence of kissing configurations whose points may not be in revolving areas. We need to be more careful when choosing such configurations, as not only must they be $\Delta$-close and kissing, we must be able to move the robots to one from $(a_0,b_0)$, between them, and from one to $(a_3,b_3)$ but without relying on the auxiliary lemmas for plans between configurations with at least one robot in a revolving area as done in the previous case.

Let $q_A^-$ (resp., $q_A^+$) be the last (resp., first) $(\Delta-6)$-close point on $\plpt{\pi_A}{a_0,a_1}$ (resp., on $\plpt{\pi_A}{a_2,a_3}$). Let $q_B^-,q_B^+$ be defined similarly. If $\plpt{\pi_A}{a_0,q_A^-} \cap \Int(\plpt{\pi_B}{b_0,q_B^-} \oplus 2\Box) \neq \varnothing$, let $p_A^-$ be the last point on $\plpt{\pi_A}{a_0,q_A^-} \cap (\plpt{\pi_B}{b_0,q_B^-} \oplus 2\Box)$ and let $k_A^-$ be the last point on $\plpt{\pi_B}{b_0,q_B^-}$ contained in $p_A^- + 2\Box$; otherwise, $p_A^-,k_A^-$ are $\nil$. Define $p_B^+,k_B^+$ similarly: If $\plpt{\pi_B}{q_B^-,b_3} \cap \Int(\plpt{\pi_A}{q_A^+,a_3} \oplus 2\Box) \neq \varnothing$, let $p_B^+$ be the first point on $\plpt{\pi_B}{q_B^+,b_3} \cap (\plpt{\pi_A}{q_A^+,a_3} \oplus 2\Box)$ and let $k_B^+$ be the last point on $\plpt{\pi_A}{q_A^-,a_3}$ contained in $p_B^+ + 2\Box$; otherwise, $p_B^+,k_B^+$ are $\nil$. Note that none of these points lie in sanctums of corridors by definition and the fact that none of $a_i,b_i$ lie in sanctums, for $i=0,1,2,3$. Indeed, if $q_A^-$ lies in the sanctum of a corridor $\corridor$, then $\plpt{\pi_A}{q_A^-,a_1}$ crosses the portals of $\corridor$, and such crossings are within $L_1$-distance $1$ of the portal endpoints in $\enVerts$; the same holds for $q_A^+,q_B^-,q_B^+$. If $p_A^-$ lies in a sanctum $\sanctum{\corridor}$ of a corridor $\corridor$, then $k_A^- \in \corridor$ and $\plpt{\pi_A}{a_0,q_A^-}$ and $\plpt{\pi_B}{b_0,q_B^-}$ span $\sanctum{\corridor}$. Then $\plpt{\pi_A}{a_0,q_A^-} \cap \sanctum{\corridor} \subset \Int(\plpt{\pi_B}{b_0,q_B^-} \oplus 2\Box$, by definition of sanctum, so $p_A^-$ lies outside $\sanctum{\corridor}$. A similar argument shows $k_A^-,p_B^+,k_B^+$ do not lie in sanctums. Note that when $p_A^-$ is $\nil$, $a_0$ may be $s_A$, and when $p_B^+$ is $\nil$, $b_3$ may be $t_B$. See \figref{dfar-middle-distant}.
We say $p_A^-$ or $p_B^+$ is \emph{distant} if it is not $\nil$ and it is $(\Delta-8)$-far. There are two main cases.

\mparagraph{At least one of $p_A^-,p_B^+$ is distant} (This is the case that most closely resembles that of case (i).) Without loss of generality, $p_B^+$ is distant. We set $(a_3',b_3')$ to be any kissing configuration where the points lie in (the boundary of) the revolving area centered at a $\dclose{\Delta-7}{\Delta-5}$ point that contains $q_B^+$. $a_3',b_3'$ are $\dclose{\Delta-8}{\Delta-4}$.
Let $$P_A \assign \plpt{\pi_A}{a_0,a_2}\ccat a_2b_2 \ccat \plpt{\pi_B}{b_2,q_B^+} \ccat q_B^+ a_3'.$$ Note that $a_2b_2 \subset \freesp$ since $a_2,b_2$ are $(\Delta-6)$-far, and hence $P_A \subset \freesp$. See \figref{dfar-middle-distant} again. There are two cases.

\begin{enumerate}
\item Suppose $P_A \subset \freept{b_0}$. Then we let $\fdpi_0$ be the decoupled, kissing $((a_0,b_0),(a_3',b_3')$-plan by first moving $A$ from $a_0$ to $a_3'$ along $P_A$, then applying the kissing $((a_3',b_0),(a_3',b_3')$-plan from \lemref{through-ra} to move $B$ from $b_0$ to $b_3'$.

\item Otherwise, $P_A \not\subset \freept{b_0}$. We have that the prefix $\plpt{P}{a_0,q_A^-} \subseteq \freept{b_0}$ and that all points on the suffix $\plpt{P_A}{q_A^-,a_3'}$ are $(\Delta-6)$-far, so if $P_A$ intersects the interior of $b_0 + 2\Box$, it must intersect on $\plpt{P_A}{q_A^-,a_3'}$. Then $b_0$ is $(\Delta-8)$-far, and hence is contained in a revolving area. We have that $b_0$ is $(\Delta-6)$-close by definition of $p_i$. Let $(a^*,b^*)$ be a kissing configuration on the revolving area centered at a $\dclose{\Delta-9}{\Delta-5}$ point that contains $b_0$. $a^*,b^*$ are $\dclose{\Delta-10}{\Delta-4}$. Let $\fdpi_0$ be the decoupled, kissing $((a_0,b_0),(a_3',b_3'))$-plan obtained by first applying the decoupled, kissing $((a_0,b_0),(a^*,b^*))$-plan from \lemref{through-ra} then applying the $((a^*,b^*),(a_3',b_3'))$-plan from \lemref{between-ras}.
\end{enumerate}

\begin{figure}
\centering
\includegraphics[scale=1.0]{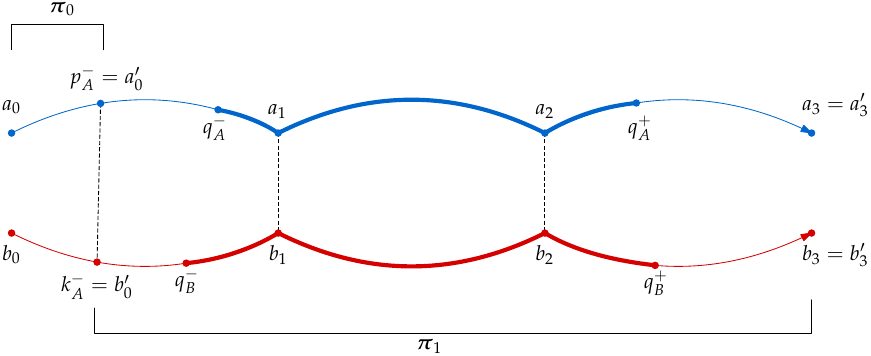}
\caption[Abstract diagram of paths $\plpt{\pi_A}{a_0,a_3},\plpt{\pi_B}{b_0,b_3}$ from \subsecref{dfar} when neither $p_A^-,p_B^+$ are \emph{distant}.]{Abstract diagram of the paths $\plpt{\pi_A}{a_0,a_3},\plpt{\pi_B}{b_0,b_3}$ and various points as defined in \subsecref{dfar} when neither $p_A^-,p_B^+$ are \emph{distant}. In this example, $p_A^-$ is not $\nil$ and $p_B^+$ is $\nil$.}
\figlab{dfar-middle-defs}
\end{figure}

In either case, we have a kissing $((a_0,b_0),(a_3',b_3'))$-plan $\fdpi_0$ where all parking places except $a_0,b_0,a_3',b_3'$ are $\dclose{\Delta-10}{\Delta-4}$. Then let $\fdpi_1$ be the kissing $((a_3',b_3'),(a_3,b_3))$-plan where we first apply the kissing $((a_3',b_3'),(a_3,q_B^+))$-plan from \lemref{through-ra} then move $B$ from $q_B^+$ to $b_3$ along $\plpt{\pi_B}{q_B^+,b_3}$. Then $\fdpi' \assign \fdpi_0 \circ \fdpi_1$ is a kissing $((a_0,b_0),(a_3,b_3))$-plan with all parking places except $a_0,b_0,a_3,b_3$ being $(\Delta-2)$-close.

Let $(\pi_A',\pi_B') = \fdpi'$. An argument similar to that in the proof of \lemref{a-b-similar} implies that $$\big| \pathcost{\plpt{\pi_A}{a_2,k_B^+}} -\pathcost{\plpt{\pi_B}{b_2,p_B^+}} \big| = O(1).$$ Similar to the proof of case (i), a tedious but straightforward analysis that incorporates each upper bound on the costs and number of moves from the $O(1)$ applications of Lemmas~\ref{lemma:through-ra} and \ref{lemma:between-ras} implies that
\begin{align*}
\plancost{\fdpi'} =& \pathcost{\pi_A'} + \pathcost{\pi_B'} \\
\leq& \big( \pathcost{\plpt{\pi_A}{a_0,a_2}} + \pathcost{\plpt{\pi_B}{b_2,p_B^+}} + \pathcost{\plpt{\pi_A}{k_B^+,a_2}} \big) + \\ 
&\big( \pathcost{\plpt{\pi_B}{b_0,b_2}} + \pathcost{\plpt{\pi_B}{b_2,p_B^+}} + \pathcost{\plpt{\pi_B}{p_B^+,b_3}} \big) + O(1) \\
\leq& \big( \pathcost{\plpt{\pi_A}{a_0,a_2}} + \pathcost{\plpt{\pi_A}{a_2,k_B^+}} + \pathcost{\plpt{\pi_A}{k_B^+,a_2}} \big) + \\
&\big( \pathcost{\plpt{\pi_B}{b_0,b_2}} + \pathcost{\plpt{\pi_B}{b_2,p_B^+}} + \pathcost{\plpt{\pi_B}{p_B^+,b_3}} \big) + O(1) \\
\leq& \pathcost{\plpt{\pi_A}{a_0,a_3}} + \pathcost{\plpt{\pi_B}{b_0,b_3}} + O(1) \\
\leq& \plancost{\fdpi(\lambda_i,\lambda_j)}
\end{align*}
for a constant $c_1 > c_0$. A similar bound $\altern{\fdpi'} \leq \altern{\fdpi} + c_2$ for a constant $c_2 > 0$ follows. We replace $\fdpi(\lambda_i,\lambda_j)$ with $\fdpi'$ in $\fdpi$, which completes the proof for this subcase.

\mparagraph{Neither of $p_A^-,p_B^+$ are distant} Since $p_A^-$ is not distant, either (i) $p_A^-$ is $\nil$ or (ii) $p_A^-$ is $(\Delta-8)$-close. If $p_A^-$ is $\nil$, let $(a_0',b_0') \assign (a_0,b_0)$ and let $\fdpi_0$ be the trivial $((a_0,b_0),(a_0,b_0))$-plan where neither robot moves. Otherwise, let $(a_0',b_0') \assign (p_A^-,k_A^-)$ and let $\fdpi_0$ be the $((a_0,b_0),(p_A^-,k_A^-))$-plan obtained by moving $A$ from $a_0$ to $p_A^-$ along $\plpt{\pi_A}{a_0,p_A^-}$ then moving $B$ from $b_0$ to $k_A^-$ along $\plpt{\pi_B}{b_0,k_A^-}$. In either case, $\fdpi_0$ is a kissing $((a_0,b_0),(a_0',b_0'))$-plan and $a_0',b_0'$ are $(\Delta-6)$-close. We similarly define $(a_3',b_3')$ and a kissing $((a_3',b_3'),(a_3,b_3))$-plan $\fdpi_2$. See \figref{dfar-middle-defs}. It remains to define a kissing $((a_0',b_0'),(a_3',b_3'))$-plan, $\fdpi_1$.

By the choice of $a_0',b_0',a_3',b_3'$, we have
\begin{itemize}
\item $\plpt{\pi_A}{a_0',q_A^-} \subset \freept{z}$ for any point $z \in \plpt{\pi_B}{b_0',q_B^-}$,
\item $\plpt{\pi_A}{q_A^+,a_3'} \subset \freept{z}$ for any point $z \in \plpt{\pi_B}{q_B^+,b_3'}$,
\item $\plpt{\pi_B}{b_0',q_B^-} \subset \freept{z}$ for any point $z \in \plpt{\pi_A}{a_0',q_A^-}$, and 
\item $\plpt{\pi_B}{q_B^+,b_3'} \subset \freept{z}$ for any point $z \in \plpt{\pi_A}{q_A^+,a_3'}$.
\end{itemize}

Furthermore, we have that all points on $\plpt{\pi_A}{q_A^-,q_A^+}$ and $\plpt{\pi_B}{q_B^-,q_B^+}$ are $(\Delta-6)$-far. There are two cases. 

\begin{figure}
\centering
\includegraphics[scale=1.0]{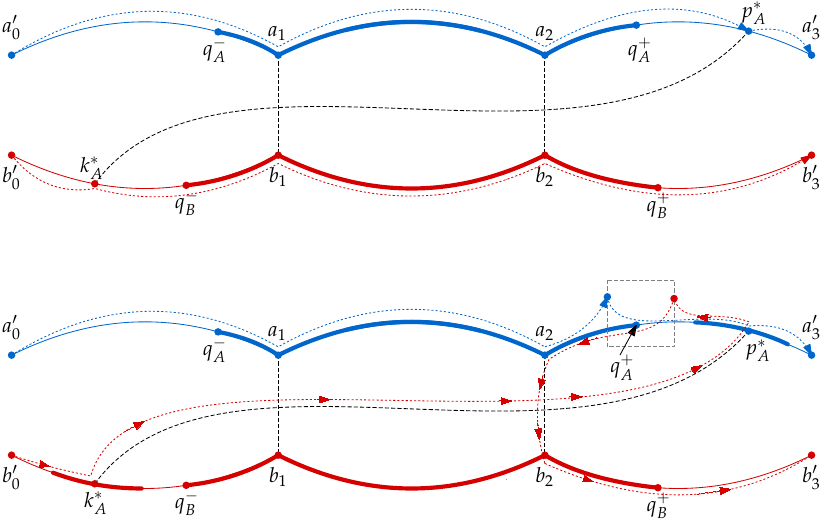}
\caption[Abstract diagrams of paths $\plpt{\pi_A}{a_0',a_3'},\plpt{\pi_B}{b_0',b_3'}$ from \subsecref{dfar} when $q_A^-,q_B^+$ are not \emph{distant} and $\plpt{\pi_A}{q_A^+,a_3'}$ intersects $\plpt{\pi_B}{b_0',q_B^-} \oplus 2\Box$.]{Abstract diagrams of the paths $\plpt{\pi_A}{a_0',a_3'},\plpt{\pi_B}{b_0',b_3'}$ and various points as defined in \subsecref{dfar} when $q_A^-,q_B^+$ are not \emph{distant} and $\plpt{\pi_A}{q_A^+,a_3'}$ intersects the interior of $\plpt{\pi_B}{b_0',q_B^-} \oplus 2\Box$. The thick pathlets are $(\Delta-6)$-far, the black dashed lines represent segments in $\freesp$ between kissing configurations, and the grey dashed square is a revolving area containing $q_A^+$. The dotted blue and red paths represent the paths of $A$ and $B$ followed during the new plan, respectively. (top) $p^*$ is $(\Delta-8)$-close. (bottom) $p^*$ is $(\Delta-8)$-far.}
\figlab{dfar-middle-nodistant}
\end{figure}

\begin{enumerate}
    \item Suppose $\plpt{\pi_A}{q_A^+,a_3'} \cap \Int(\plpt{\pi_B}{b_0',q_B^-} \oplus 2\Box) = \varnothing$. Then $\plpt{\pi_B}{b_0',q_B^-} \subset \freept{a_3'}$. Furthermore, if $\plpt{\pi_A}{q_A^-,q_A^+}$ intersects the interior of $\plpt{\pi_B}{b_0',q_B^-}$, then $b_0'$ is $(\Delta-6)$-far, and hence is contained in a revolving area centered at a $(\Delta-7)$-far point. Since $b_0'$ is $(\Delta-6)$-close, the center of the revolving area is $(\Delta-5)$-close. So we either move $A$ directly to $a_3'$ from $a_0'$ on $\plpt{\pi_A}{a_0',a_3'}$ or we apply \lemref{through-ra} if that path is not in $\freept{b_0'}$. Similarly, we either move $B$ directly to $b_3'$ from $b_0'$ on $\plpt{\pi_B}{b_0',b_3'}$ or we apply \lemref{through-ra} if that path is not in $\freept{a_3'}$. Let $\fdpi_1$ be the resulting kissing $((a_0',b_0'),(a_3',b_3'))$-plan. All parking places in $\fdpi_1$ except $a_0',b_0',a_3',b_3'$ are $(\Delta-4)$-close.

    \item Otherwise, $\plpt{\pi_A}{q_A^+,a_3'} \cap \Int(\plpt{\pi_B}{b_0',q_B^-} \oplus 2\Box) \neq \varnothing$. Let $p_A^*$ be the first point on $\plpt{\pi_A}{q_A^+,a_3'}$ such that $p_A^* \in \plpt{\pi_B}{b_0',q_B^-} \oplus 2\Box$, and let $k_B^*$ be the last point on $\plpt{\pi_B}{b_0',q_B^-}$ such that $k_B^* \in p_A^* + 2\Box$. Then $(p_A^*,k_B^*)$ is a kissing configuration, $\plpt{\pi_A}{q_A^+,p_A^*} \subset \freept{b_0'}$, and $\plpt{\pi_B}{k_B^-,q_B^-} \subset \freept{p_A^*}$. Then we proceed similar to the earlier case where $p_B^+$ is distant: If $p_A^*$ is $(\Delta-8)$-close we construct a kissing $((a_0',b_0'),(a_3',b_3'))$-plan $\fdpi_1$ by applying \lemref{through-ra} at most twice as necessary. Specifically, we move $A$ to $p_A^*$ on $\plpt{\pi_A}{a_0',p_A^*}$, possibly moving $B$ within a revolving area containing $b_0'$ if $A$ collides with $B$, then move $B$ to $k_A^*$ on $\plpt{\pi_B}{b_0',b_3'}$, possibly moving $A$ within a revolving area containing $a_3'$ if $B$ collides with $A$, and then finally move $A$ to $a_3'$ on $\plpt{\pi_A}{p_A^*,a_3'}$. See \figref{dfar-middle-nodistant}(top). Otherwise, $p_A^*$ is $(\Delta-8)$-far, so we cannot park $A$ there. Instead, we construct a kissing $((a_0',b_0'),(a_3',b_3'))$-plan $\fdpi_1$ by first moving $A$ then $B$ to the revolving area containing $q_A^+$ using a plan similar to that in case (1). See \figref{dfar-middle-nodistant}(bottom). Then we move $A$ to $a_3'$ using the plan from \lemref{through-ra}, move $B$ to $b_3'$ along the reversal of $\plpt{\pi_A}{a_2,q_A^*}$ followed by $a_2b_2 \ccat \plpt{\pi_B}{b_2,b_3'}$, possibly moving $A$ within a revolving area containing $a_3'$ if $B$ collides with $A$. In either case, it can be verified that all parking places in $\fdpi_1$ except $a_0',b_0',a_3',b_3'$ are $(\Delta-2)$-close.

\end{enumerate}
In either case we have a $((a_0',b_0'),(a_1',b_1'))$-plan $\fdpi_1$ where all parking places are $(\Delta-4)$-close.
Let $(\pi_A^1,\pi_B^1) = \fdpi_1$.
Suppose case (2) occurs. An argument similar to that in the proof of \lemref{a-b-similar} implies that $$\big| \pathcost{\plpt{\pi_A}{a_2,p_A^*}} - \pathcost{\plpt{\pi_B}{k_A^*,b_1}}\big| = O(1).$$ Then again, in either case (1) or (2), a tedious but straightforward analysis that incorporates each upper bound on the costs and number of moves from the $O(1)$ applications of Lemmas~\ref{lemma:through-ra} and \ref{lemma:between-ras} implies that
\begin{align*}
\plancost{\fdpi'} =& \pathcost{\pi_A^1} + \pathcost{\pi_B^1} \\
\leq& \big( \pathcost{\plpt{\pi_A}{a_0',a_2}}     + \pathcost{\plpt{\pi_A}{a_2,p_A^*}} + \pathcost{\plpt{\pi_A}{p_A^*,a_3'}} \big) + \\ 
&\big(      \pathcost{\plpt{\pi_B}{b_0',k_B^*}}   + \pathcost{\plpt{\pi_A}{a_2,p_A^*}} + \pathcost{\plpt{\pi_B}{b_2,b_3'}} \big) + O(1) \\
\leq& \big( \pathcost{\plpt{\pi_A}{a_0',a_2}}     + \pathcost{\plpt{\pi_A}{a_2,p_A^*}} + \pathcost{\plpt{\pi_A}{p_A^*,a_3'}} \big) + \\ 
&\big(      \pathcost{\plpt{\pi_B}{b_0',k_B^*}}   + \pathcost{\plpt{\pi_B}{k_B^*,b_2}} + \pathcost{\plpt{\pi_B}{b_2,b_3'}} \big) + O(1) \\
\leq& \pathcost{\plpt{\pi_A}{a_0',a_3'}} + \pathcost{\plpt{\pi_B}{b_0',b_3'}} + O(1).
\end{align*}
It is easy to verify that
$$\plancost{\fdpi_0} \leq \pathcost{\plpt{\pi_A}{a_0,a_0'}} + \pathcost{\plpt{\pi_B}{b_0,b_0'}} + O(1)
\text{\ and\ } \altern{\fdpi_0} = O(1)$$ and
$$\plancost{\fdpi_2} \leq \pathcost{\plpt{\pi_A}{a_3',a_3}} + \pathcost{\plpt{\pi_B}{b_3',b_3}} + O(1)
\text{\ and\ } \altern{\fdpi_2} = O(1).$$ Set $\fdpi' \assign \fdpi_0 \circ \fdpi_1 \circ \fdpi_2$. It follows that $\plancost{\fdpi'} = \pathcost{\plpt{\pi_A}{a_0,a_3}} + \pathcost{\plpt{\pi_B}{b_0,b_3}} + O(1)$ and $\altern{\fdpi} = O(1)$. We replace $\fdpi(\lambda_i,\lambda_j)$ with $\fdpi'$ in $\fdpi$, which completes the proof for this subcase.
\fi

\section{Discretizing the Free Space}
\seclab{retraction}

In this section, we further transform the optimal tame plans described in the previous section by ``retracting'' all parking places to a discrete set of points.
For any $\eps \in (0,1)$, let $\grid$ be the axis-aligned grid centered at the origin whose cells are $\eps$-radius squares.

Let $\eps \in (0,1)$ be a parameter. We show how to choose a set $\gridvert \subset \Reals^2$ of $O(n(\Delta/\eps)^2)$ points so that a decoupled, $\Delta$-tame $(\fd{s},\fd{t})$-plan $\fdpi$ can be deformed into another decoupled, $(\Delta+2\eps)$-tame $(\fd{s},\fd{t})$-plan $\widehat{\fdpi}$ such that (i) the robots are parked at points of $\gridvert$, (ii) $\plancost{\widehat{\fdpi}} \leq \plancost{\fdpi} + \eps \altern{\fdpi}$, and (iii) $\altern{\widehat{\fdpi}} = c \altern{\fdpi}$, where $c$ is an absolute constant that does not depend on $\eps$.

Let $\fdpi$ be a decoupled, $\Delta$-tame $(\fd{s},\fd{t})$-plan. We can assume that $\altern{\fdpi} = O(\plancost{\fdpi} + 1)$. Let $\grid$ be the axis-aligned uniform grid with square cells of radius $\eps$ such that all parking places lie in the interior of grid cells and $\fdpi$ does not pass through a vertex of $\grid$. Let $\freesp^\#$ be the overlay of $\grid$ and $\freesp$, restricted to $\freesp$. Each face of $\freesp^\#$ is a connected component of $\freesp \cap g$ for some grid cell $g$ of $\grid$. Let $\gridvert$ be the set of vertices of $\freesp^\#$. Our goal is to ``retract'' the parking places of $\pi$ to the points of $\gridvert$, \ie, the robots are parked at the points of $\gridvert$ instead of their original parking places. Furthermore, since $\fdpi$ is kissing, we want to ensure that the retracted path is \emph{$\eps$-nearly-kissing}, \ie, whenever a robot moves, it comes within $L_\infty$-distance $4\eps$ of the boundary of the other robot (parked at a vertex of $\gridvert$). However, if for a parking place $q$, say, of $\robA$, we pick only one point in $\gridvert$ to park $\robA$ at instead of $q$, $\robB$ may collide with $\robA$ during its next move, especially since $B$ kisses $A$ at $q$ during the next move of $\fdpi$. Hence, we may have to choose multiple points of $\gridvert$ (in the neighborhood of $q$) and move $\robA$ between them during the next move of $\robB$, to ensure that $A$ and $B$ do not collide. Each such move of $A$ increases the cost of the plan by $O(\eps)$, so we cannot move $A$ too many times. Furthermore, we want to maintain the property of being decoupled (\ie, only one robot moves at a time), which means that when we move $A$ between nearby points of $\gridvert$ to make way for $B$, we must first park $B$ somewhere, also in $\gridvert$. These technical constraints make the retraction rather involved. We now describe the retraction in detail, but first state a lemma which follows easily from \lemref{segment-connected}.

\begin{lemma}
\lemlab{few-crossings}
Let $\fdpi$ be a decoupled $(\fd{s},\fd{t})$-plan, and let $\ell$ be a horizontal or vertical line. During a single move of $\fdpi$, if the path $P$ followed by a robot intersects $\ell$ at two points that are less than two distance apart, then $P$ can be shortened without affecting the rest of the plan.
\end{lemma}

We assume that $\fdpi$ satisfies \lemref{few-crossings}. Assume that $\moveseq{\fdpi} = (R_1,\pi_1,p_1), \ldots, (R_k,\pi_k,p_k)$. Assume $p_0 = s_A, p_1 = s_B$, \ie, $R_1 = A$. Let $\lambda_1 < \lambda_2 < \ldots < \lambda_{k-1}$ be the time instances at which $\fdpi$ switches from move $i$ to $i+1$. Set $\lambda_0 \assign 0$ and $\lambda_{k+1} \assign 1$. Then for $0 \leq i \leq k$, $\fdpi(\lambda_i) = (p_i,p_{i+1})$ if $i$ is even and $\fdpi(\lambda_i) = (p_{i+1},p_i)$ if $i$ is odd. We describe the retraction of each move one by one. For each move, the retracted plan consists of $O(1)$ moves and increases the cost by $O(\eps)$.

Let $e_0 \assign (-1,0), e_1 \assign (0,-1), e_2 \assign (1,0)$, and $e_3 \assign (0,1)$ be the four standard directions; $e_i = -e_{i+2 (\mathrm{mod}\, 4)}$. Set $S = \{e_i \mid 0 \leq i \leq 3\}$.
For an edge $\gamma$ of an axis-aligned square, the inner normal of $\gamma$ is one of the $e_i$'s, namely if $\gamma$ is the left edge of the square then the inner normal is $e_2 = (1,0)$, and so on.
For a connected component $C$ of $g \cap \freesp$ of a cell $g \in \grid$ and for $0 \leq i \leq 3$, let $\xi_i(C)$ be an extremal vertex of $C$ in direction $e_i$. (Note that the four vertices may not be distinct.) For a point $q \in \Reals^2$, let $C(q)$ be the face of $\freesp^\#$ that contains $q$. With a slight abuse of notation, we use $\xi_i(q)$ to denote $\xi_i(C(q))$. Set $\Xi(q) \assign \{\xi_i(q) \mid 0 \leq i \leq 3\}$.

The retracted plan $\widehat{\fdpi}$ that we construct maintains the following invariant: for $0 \leq i \leq k$, $\linfnorm{\widehat{\pi}_A(\lambda_i)}{\widehat{\pi}_B(\lambda_i)} \geq 2$, $\widehat{\pi}_A(\lambda_i) \in \freept{\widehat{\pi}_B(\lambda_i)}$, and $\widehat{\pi}_B(\lambda_i) \in \freept{\widehat{\pi}_A(\lambda_i)}$. We set $\widehat{\pi}_A(\lambda_0) \assign s_A$ and $\widehat{\pi}_B(\lambda_0) \assign s_B$. The invariant obviously holds for $\lambda_0$. Assume that we have retracted the first $i-1$ moves of $\fdpi$. We now describe the retraction of the $i$-th move. Without loss of generality, assume that $i$ is even, so $R_i = B$, $A$ is parked at $p_i$, $B$ moves along $\pi_i$, $\fdpi(\lambda_{i-1}) = (p_i,p_{i-1})$ and $\fdpi(\lambda_i) = (p_i,p_{i+1})$. Let $C$ be the face of $\freesp^\#$ containing $p_i$, $\Box_i \assign p_i + 2\Box$, and $\annular_i \assign p_i + 2(1+\eps)\Box$. The intersection of $\annular_i \setminus \Box_i$ with the line supporting an edge of $\Box_i$ consists of two segments, each connecting the edges of $\annular_i$ and $\Box_i$. We refer to these eight segments, over the four edges of $\Box_i$, as \emph{extension chords}. These extension chords partition $\annular_i \setminus \Box_i$ into four \emph{corner squares} and four \emph{side rectangles}.
Since $\fdpi$ is a feasible plan, $\pi_i \cap \Int(\Box_i) = \varnothing$. For a point $q \in \Reals^2$, let $S(q) \assign \{e_j \in S \mid \abs{(q-p_i) \cdot e_j} \geq 2\}$ and $C(q) \assign \{\xi_j(q) \mid e_j \in S(q)\}$. Since $\fdpi$ is feasible, $S(q) \neq \varnothing$ for all $q \in \pi_i$. To define the retraction of the $i$-th move, we define events during the interval $[\lambda_{i-1},\lambda_i]$ at which $A$ is (possibly) moved from one point of $\Xi(p_i)$ to another while $B$ is parked at a point of $\gridvert$. There are two types of events:

\begin{enumerate}[(i)]
\item \emph{Boundary event}. A time instance $\lambda$ is a boundary event if $\pi_i(\lambda) \in \bd \annular_i$ and $B$ enters $\annular_i$ immediately after $\lambda_i$.
\item \emph{Separation event}. A time instance $\lambda$ is a separation event if $\pi_i(\lambda) \in \annular_i$ and $p_i + \Box$ and $\pi_i(\lambda) + \Box$ stop being $x$-separated or $y$-separated, \ie, $\abs{x(p_i)-x(\pi_i(\lambda))} = 2$ or $\abs{y(p_i) - y(\pi_i(\lambda))} = 2$, and hence $\pi_i(\lambda)$ lies on one of the eight extension chords of $\annular_i$ and $\pi_i$ leaves a corner square of $\annular_i$ and enters a side rectangle.
\end{enumerate}

By \lemref{few-crossings}, $\pi_i$ crosses each edge of $\annular_i$ at most three times and each extension chord at most once, so there are at most six boundary events and eight separation events during the $i$-th move. Hence, there are at most 14 events. We also add $\lambda_i$ as an event to ensure that at the end of the $i$-th move $B$ is parked at a point of $\gridvert$, so that the invariant is maintained. Let $\theta_1 < \theta_2< \ldots < \theta_{i-1} < (\theta_i = \lambda_i)$ be the sequence of events along $\pi_i$. For any time $\theta$ between two consecutive events $(\theta_{i-1},\theta_i)$, $S(\pi_i(\theta))$ does not change, so $B$ move along $\pi_i(\theta_{i-1},\theta_i)$ and $A$ remains parked at a vertex in $C(\pi_i(\theta))$. By the above invariant, $A$ is parked at a vertex of $C(\pi_i(\lambda_{i-1}))$ at time $\lambda_{i-1}$. So assume that we have processed events $\theta_1, \ldots, \theta_{j-1}$ and retracted the plan until $\theta_j$. $B$ crosses an edge of $\annular_i$ or one of its extension chords. Let $e_k$ be the inner normal of the extension chord. If $A$ is currently parked at the vertex $\xi_k(p_i)$ then $\xi_k(p_i) \in C(\pi_i(\theta))$ for all $[\theta_j,\theta_{j+1})$ and hence no action is needed. Otherwise, we first move $B$ to the vertex $\xi_{k+2}(\pi_i(\theta_j))$ from $\pi_i(\theta_j)$ along an $xy$-monotone path in $C(\pi_i(\theta_j))$ and park it there, then move $A$ from its current position to the vertex $\xi_k(p_i)$ within $C(p_i)$ by first moving it along an $xy$-monotone path to $p_i$ and then from $p_i$ to $\xi_k(p_i)$ again along an $xy$-monotone path. After $A$ is parked at $\xi_k(p_i)$, we move $B$ back to $\pi_i(\theta_j)$ then move it along $\pi_i$ from $\pi_i(\theta_j)$ toward $\pi_i(\theta_{j+1})$. It can be verified that these paths are feasible.

If $\theta_j = \lambda_i$, then we have to find an appropriate parking place for $B$. By construction, $A$ is parked at a vertex $\xi_k(p_i) \in C(p_{i+1})$. Then we move $B$ from $p_{i+1}$ to $\xi_{k+2}(p_{i+1})$, \ie, the farthest vertex of $C(p_{i+1})$ in direction $-e_k$, as above. This step ensures that the invariant is satisfied after move $i$.

Let $\fdpi'$ be the plan obtained by retracting the given plan $\fdpi$. Overall, each move of $\fdpi$ involves at most 15 events, and each event involves three new moves along $xy$-monotone paths in faces of $\freesp^\#$. By definition of $\freesp^\#$, each new move has length at most $2\eps\sqrt{2} < 3\eps$. Hence the total increase in cost is less than $135\eps\altern{\fdpi}$. Furthermore, all parking places of $\fdpi'$ are at vertices of $\gridvert$ within $L_\infty$-distance $2\eps$ of parking places in $\fdpi$. By re-picking $\eps \assign \eps/135$, we have the following lemma.

\begin{lemma}
\lemlab{nearly-kissing}
Let $\eps \in (0,1)$ be a parameter, and let $\fdpi$ be a decoupled, $\Delta$-tame $(\fd{s},\fd{t})$-plan. There exists a decoupled, $(\Delta+2\eps)$-tame, $(\fd{s},\fd{t})$-plan $\fdpi'$ such that $\plancost{\fdpi'} \leq \plancost{\fdpi} + \eps \altern{\fdpi}$ and $\altern{\fdpi'} = c \altern{\fdpi}$, and every parking place of $\fdpi'$ is in $\gridvert$, for some constant $c > 0$ that does not depend on $\eps,\Delta$. If $\fdpi$ is kissing, then $\fdpi'$ is $\eps$-nearly-kissing.
\end{lemma}

\section{Algorithm}
\seclab{algorithm}

We are now ready to describe our algorithm to compute an $(\fd{s},\fd{t})$-plan $\fdpi$ with 
$\plancost{\fdpi} \leq (1+\eps)\plancost{\fdpi^*}$ for any $\eps \in (0,1]$.
We first describe an $n^3 \eps^{-O(1)} \log n$-time algorithm (\lemref{graph-correspondence}) under the assumption that $\plancost{\fdpi^*} > 1/4$. With further efforts,
we present a near-quadratic time algorithm (\lemref{graph-correspondence-kissing}) and how to remove the assumption (\subsecref{remove-assumption}).

The algorithm consists of three stages. First, we choose a set $\nearvert$ of $O(n/\eps^4)$ points so that a robot is always parked at one of the points in $\nearvert$. Next, we construct a graph $\graph = (\gverts,\gedges)$ where $\gverts \subseteq \nearvert \times \nearvert$ is a set of (feasible) configurations and each edge is a (decoupled) plan between a pair of configurations of $\gverts$ with one move. We compute a shortest path in $\graph$, which corresponds to an $(\fd{s},\fd{t})$-plan $\widehat{\fdpi}$ with $\plancost{\widehat{\fdpi}} \leq 
(1+\eps)\plancost{\fdpi^*} + O(\eps) \le \plancost{\fdpi} \leq (1+O(\eps))\plancost{\fdpi^*}$ for 
$\plancost{\fdpi^*}\ge 1/4$.

Set $\overline{\eps} \assign \eps/c_0$ and $\Delta \assign c_1/\overline{\eps}$ where $c_0,c_1 > 0$ are sufficiently large constants (independent of $\eps$) to be chosen later. Let $\grid$, $\freesp^\#$, and $\gridvert$ be the same as in \secref{retraction} but using $\overline{\eps}$ for $\eps$. Let $\gridfaces$ be the set of faces of $\freesp^\#$ that contain a $\Delta$-close point; any point in a face $C \in \gridfaces$ is $(\Delta+2\overline{\eps})$-close.
Let $\nearvert$ be the set of vertices of $\gridfaces$;
$\abs{\nearvert} = O(n\Delta^2/\overline{\eps}^2) = O(n/\eps^4)$. We now describe the weighted graph $\graph = (\gverts,\gedges)$. We set $\gverts \assign \{(a,b) \in \nearvert \times \nearvert \mid \linfnorm{a}{b} \geq 2\}$.
Note that $\fd{s},\fd{t} \in \gverts$ and $\gverts \subset \fdfreesp$. We construct $\gedges$ as follows: 
For every ordered triple $(u,v,p) \in \nearvert \times \nearvert \times \nearvert$ with $u \neq v$ and 
$\ltnorm{p}{u},\ltnorm{p}{v} \geq 2$, we set 
$\omega((u,p) \rightarrow (v,p)) = \omega((p,u) \rightarrow (p,v)) \assign \geopt{u}{v}{p}$,
and if this value is not $\infty$ we add edges $(u,p)\rightarrow(v,p)$ and $(p,u)\rightarrow(p,v)$ to $\gedges$ 
with $\omega((u,p)\rightarrow(v,p)) = \omega((p,u) \rightarrow (p,v))$ as their weight, which corresponds to moving $A$ (resp., $B$) from $u$ to $v$ along a shortest path in $\freept{p}$ while $B$ (resp., $A$) is parked at $p$. Then $\abs{\gedges} = \abs{\nearvert}^3 = O(n^3/\eps^{12})$.

Finally, we compute a shortest path (by weight) $\Phi$ in $\graph$ from $\fd{s}$ to $\fd{t}$. After having computed $\Phi$, the $(\fd{s},\fd{t})$-plan corresponding to $\Phi$ can be retrieved in a straightforward manner, and the cost of the resulting plan is the same as the weight of the path.
We conclude by stating the following lemma:

\begin{lemma}
\lemlab{graph-correspondence}
Given $\fd{s},\fd{t} \in \fdfreesp$, and $\eps \in (0,1)$, there exists a path $\Phi$ from $\fd{s}$ to $\fd{t}$ in $\graph$, if $\fd{s},\fd{t}$ are reachable, whose weight is at most $(1+\eps)\plancost{\fdpi^*} + O(\eps)$, which is $(1+O(\eps))\plancost{\fdpi^*}$ if $\plancost{\fdpi^*} > 1/4$, where $\fdpi^*$ is a decoupled, optimal $(\fd{s},\fd{t})$-plan. Conversely, a path $\Phi$ from $\fd{s}$ to $\fd{t}$ in $\graph$ corresponds to an $(\fd{s},\fd{t})$-plan $\widehat{\fdpi}$ of cost $\omega(\Phi)$. Furthermore, a shortest path from $\fd{s}$ to $\fd{t}$ in $\graph$ can be computed in $O(n^3\eps^{-12} \log n)$ time.
\end{lemma}

\begin{proof}
By \corref{eps-faraway-maybe-kissing}, there exists a decoupled $(\Delta = c_1/\overline{\eps})$-tame plan $\fdpi$ with $\plancost{\fdpi} \leq (1+\overline{\eps})\plancost{\fdpi^*}$ and $\altern{\fdpi} \leq c_2(\altern{\fdpi^*}+1)$ for some constants $c_1,c_2 > 0$.
(We make use of the stronger \corref{eps-faraway} that guarantees $\fdpi$ is kissing when we improve the algorithm in the next subsection.)
Then, by \lemref{nearly-kissing} with $\overline{\eps}$ as parameter $\eps$, there exists a $(\Delta+2\overline{\eps})$-tame, decoupled plan $\fdpi'$ such that all parking places belong to $\nearvert$ and
$$\plancost{\fdpi'} \leq \plancost{\fdpi} + \overline{\eps}\altern{\fdpi}
\leq (1+\overline{\eps})\plancost{\fdpi^*} + c_2\overline{\eps}(\plancost{\fdpi^*}+1)
\leq (1+\overline{\eps}(1+c_2))\plancost{\fdpi^*} + c_2\overline{\eps}.$$
At this point, we have additive error $O(\eps)$. Here we make use of our assumption that $\plancost{\fdpi^*} > 1/4$ and have
$$\plancost{\fdpi'}
\leq (1+\overline{\eps}(1+5c_2))\plancost{\fdpi^*}.$$
Then by choosing $c_0 \assign 1+5c_2$, we have $\overline{\eps} = \eps/(1+5c_2)$ and hence $$\plancost{\fdpi'} \leq (1+\eps)\plancost{\fdpi^*}.$$

Let $\moveseq{\fdpi'} = (R_1,\pi_1,p_1), \ldots, (R_\ell,\pi_\ell,p_\ell)$. Without loss of generality, assume that $R_1 = A$. Then we map $\widehat{\fdpi}$ to a path from $\fd{s}$ to $\fd{t}$ in $\graph$ as follows. For each $1 \leq i \leq \ell$, $\pi_i$ is a path followed by one of the robots from $p_{i-1}$ to $p_{i+1}$ while the other is parked at $p_i$, so $\linfnorm{p_{i+1}}{p_i} \geq 2$ and $\geopt{p_{i-1}}{p_{i+1}}{p_i} \leq \pathcost{\pi_i}$. Therefore $(p_{i-1},p_i) \rightarrow (p_{i+1},p_i), (p_i,p_{i-1}) \rightarrow (p_i,p_{i+1}) \in \gedges$ with their weights being at most $\pathcost{\pi_i}$. Hence $\fd{s} = (p_0,p_1) \rightarrow (p_2,p_1) \rightarrow (p_2,p_3) \rightarrow \ldots \rightarrow \fd{t}$ is a path in $\graph$ of weight at most $\plancost{\fdpi}$.

Converting $\Phi$ to a decoupled $(\fd{s},\fd{t})$-plan of cost at most $\omega(\Phi)$ is straightforward and omitted from here. It remains to analyze the runtime of the algorithm. $\freesp$ and $\nearvert$ can be computed in $O(n \log^2 n + \abs{\nearvert}) = O(n(\log^2 n + 1/\eps^4))$ time \cite{DBLP:books/lib/BergCKO08}. 
For any ordered pair $(u,p) \in \nearvert \times \nearvert$, $\freept{p}$ can be computed from $\freesp$ in $O(n \log n)$ time and processed \cite{DBLP:journals/siamcomp/HershbergerS99} in $O(n \log n)$ time into a data structure that answers $O(\log n)$-time shortest-path queries from $u$ to any query point $v \in \freept{p}$. So we can compute $\omega((u,p)\rightarrow (v,p)) = \omega((p,u)\rightarrow (p,v))$ in $O(n \log n + \abs{\nearvert}\log n) = O((n/\eps^4)\log n)$ time, for all $v \in \nearvert$. Repeating this process for all $O((n/\eps^4)^2)$ pairs $(u,p) \in \nearvert \times \nearvert$, we compute $\graph$ and its edge weights in $O(\abs{\gedges} \log n) = O((n/\eps^4)^3 \log n)$ time. Finally, computing the shortest path $\fd{\Phi}$ in $\graph$ and reporting its corresponding plan takes $O(\abs{\gedges} + \abs{\gverts}\log \abs{\gverts})$ time using Dijkstra's algorithm, which is dominated by the $O(\abs{\gedges} \log n)$ time to build $\graph$. Therefore the overall running time is $O(n^3\eps^{-12} \log n)$.
\end{proof}

\subsection{Reducing the runtime}
\subseclab{improve-runtime}\mynewline
\noindent Now we describe how to reduce the runtime to $O(n^2\eps^{-O(1)}\log n)$ using \corref{eps-faraway} (instead of \corref{eps-faraway-maybe-kissing}). The high-level idea is to reduce the number of vertices, $\abs{\gverts}$, from $O(n^3 \polyn(\log n, 1/\eps))$ to $O(n^2 \polyn(\log n, 1/\eps))$ while maintaining the $O(\abs{\nearvert})$ degree of each node. The effect is that the size of eachof $\abs{\gverts},\abs{\gedges}$ reduces by a factor of $n$, which reduces the overall runtime by a factor of $n$.

We first describe the graph $\graph = (\gverts,\gedges)$. We set $\gverts \assign \{(a,b) \in \nearvert \times \nearvert \mid 2 \leq \linfnorm{a}{b} \leq 2(1+\overline{\eps})\}$.
Note the new condition that $\linfnorm{a}{b} \leq 2(1+\overline{\eps})$.
For a pair of nearby configurations $\fd{u} = (u_A,u_B),\fd{v} = (v_A,v_B) \in \gverts$, we consider two possible $(\fd{u},\fd{v})$-plans: (i) keep $A$ parked at $u_A$ while $B$ moves from $u_B$ to $v_B$ along a shortest path in $\freept{u_A}$, then park $B$ at $v_B$ and move $A$ from $u_A$ to $v_A$ along a shortest path in $\freept{v_B}$, and (ii) keep $B$ parked at $u_B$ while $A$ moves from $u_A$ to $v_A$ along a shortest path in $\freept{u_B}$, then park $A$ at $v_A$ and move $B$ from $u_B$ to $v_B$ along a shortest path in $\freept{v_A}$.
Set $$\omega(\fd{u},\fd{v}) = \min\{\geopt{u_A}{v_A}{u_B} + \geopt{u_B}{v_B}{v_A}, \geopt{u_B}{v_B}{u_A} + \geopt{u_A}{v_A}{v_B}\}.$$ If $\omega(\fd{u},\fd{v}) < \infty$, we add $\fd{u} \rightarrow \fd{v}$ to $\gedges$ with $\omega(\fd{u},\fd{v})$ as its weight. Then $\abs{\gedges} = \abs{\nearvert}^2 = O(n^2/\eps^8)$. For a fixed configuration $\fd{u} \assign (u_A,u_B) \in \gverts$, we compute the shortest path from $u_A$ to all points of $\nearvert$ within $\freept{u_B}$, using the same data structure as before \cite{DBLP:journals/siamcomp/HershbergerS99}, and do the same for $u_B$ to all points of $\nearvert$ in $\freept{u_A}$. After repeating this step for all configurations in $\gverts$, we have all the information to compute $\omega(\fd{u},\fd{v})$ for all $(\fd{u},\fd{v}) \in \gverts \times \gverts$. The overall runtime can be shown to be $O(\abs{\gedges} \log n)$ as before, which is $O(n^2\eps^{-8}\log n)$ here.

A similar argument for \lemref{graph-correspondence} that uses \corref{eps-faraway} instead of \corref{eps-faraway-maybe-kissing} proves the following lemma, which is the same as \lemref{graph-correspondence} except that the plan $\widehat{\fdpi}$ is $\eps$-nearly-kissing.

\begin{lemma}
\lemlab{graph-correspondence-kissing}
Given $\fd{s},\fd{t} \in \fdfreesp$, and $\eps \in (0,1)$, there exists a path $\Phi$ from $\fd{s}$ to $\fd{t}$ in $\graph$, if $\fd{s},\fd{t}$ are reachable, whose weight is at most $(1+\eps)\plancost{\fdpi^*} + O(\eps)$, which is bounded by $(1+O(\eps))\plancost{\fdpi^*}$ if $\plancost{\fdpi^*} > 1/4$, where $\fdpi^*$ is a decoupled, kissing, optimal $(\fd{s},\fd{t})$-plan. Conversely, a path $\Phi$ from $\fd{s}$ to $\fd{t}$ in $\graph$ corresponds to a decoupled, $\eps$-nearly-kissing $(\fd{s},\fd{t})$-plan $\widehat{\fdpi}$ of cost $\omega(\Phi)$. Furthermore, a shortest path from $\fd{s}$ to $\fd{t}$ in $\graph$ can be computed in $O(n^2\eps^{-8} \log n)$ time.
\end{lemma}

\subsection{Handling nearby configurations}
\subseclab{remove-assumption}\mynewline
\noindent We now describe how we compute an $(\fd{s},\fd{t})$-plan of cost at most $(1+\eps)\plancost{\fdpi^*}$ even when $\plancost{\fdpi^*} \leq 1/4$.
The algorithm described in the following \secref{small-opt} (\cf \lemref{small-opt-alg}) either reports an $8$-approximation $\gamma \leq 2$ of $\plancost{\fdpi^*}$, \ie, $\plancost{\fdpi^*} \leq \gamma \leq 8\plancost{\fdpi^*}$, or it reports that $\plancost{\fdpi^*} > 1/4$. So we first run this algorithm. If it reports $\plancost{\fdpi^*} > 1/4$, we run the algorithm above (with improved runtime). Otherwise, we have $\gamma \leq 2$ and $\plancost{\fdpi^*} \leq \gamma \leq 8 \plancost{\fdpi^*}$. Then $\gamma/8 \leq \plancost{\fdpi^*} \leq \gamma \leq 2$.
In this case, we simply run the above algorithm except we set $\overline{\eps} \assign \gamma\eps/c_0$ for a parameter $c_0 > 0$ to be chosen later and set $\Delta \assign \gamma$.

Then $\nearvert$ contains $(\gamma+2\overline{\eps})$-close points and $\abs{\nearvert} = O(n\Delta^2/\overline{\eps}^2) = O(n\gamma^2/(\gamma\eps)^2) = O(n/\eps^2)$.
Following the same argument as in the proof of \lemref{graph-correspondence}, we claim that $c_0$ can be chosen so that there exists a $(\Delta+2\overline{\eps})$-tame plan $\fdpi'$ with $\plancost{\fdpi'} \leq (1+\eps)\plancost{\fdpi^*}$ and all parking places of $\fdpi'$ are in $\nearvert$.

To prove the claim, note that $\fdpi^*$ is trivially $(\Delta=\gamma)$-tame since $\plancost{\fdpi^*} \leq \gamma$. By \lemref{few-parking}, we have $$\altern{\fdpi^*} \leq c_2 (\plancost{\fdpi^*}+1) \leq 3c_2$$ for a constant $c_2 > 0$. Then, by \lemref{nearly-kissing} with $\overline{\eps}$ as parameter $\eps$, there exists a decoupled, $(\Delta+2\overline{\eps})$-tame, $\overline{\eps}$-nearly-kissing plan $\fdpi'$ with all parking places of $\fdpi'$ in $\nearvert$ and $$\plancost{\fdpi'} \leq \plancost{\fdpi^*} + \overline{\eps}\altern{\fdpi^*} \leq \plancost{\fdpi^*} + 3c_2 \overline{\eps} = \plancost{\fdpi^*} + 3c_2\gamma\eps/c_0 \leq (1 + 24c_2\eps/c_0)\plancost{\fdpi^*},$$ where the last inequality follows by $\plancost{\fdpi^*} \geq \gamma/8$. So we choose $c_0 \assign 1/(24c_2)$. This proves the claim. The rest of the analysis follows from the previous algorithm, including the runtime analysis, since the algorithm from \lemref{small-opt-alg} only takes $O(n \log^2 n)$ additional time.

\section{$O(1)$-Approximate Plans for Close Configurations}
\seclab{small-opt}

In this section we prove \lemref{small-opt-alg}, which states we can compute either an $8$-approximation $\gamma$ of $\fdpi^*$ or detect that $\plancost{\fdpi^*} > 1/4$ in $O(n \log^2 n)$ time. First, we introduce some notations.

If all moves of a plan $\fdpi$ are $xy$-monotone, we say $\fdpi$ is $xy$-monotone. For a (piecewise-linear) $xy$-monotone $(\fd{s},\fd{t})$-plan $\fdpi$, $\fd{s},\fd{t} \in \fdfreesp$, let $\loplancost{\fdpi}$ be the \emph{$L_1$-cost} of $\fdpi$, \ie, if $\mathopen\langle u_1, u_2, \ldots, u_g \mathclose\rangle$ (resp., $\mathopen\langle v_1,v_2,\ldots,v_h \mathclose\rangle$) is the sequence of vertices of $\pi_A$ (resp., $\pi_B$), then $$\loplancost{\fdpi} = \sum_{i=1}^{g-1} \lonorm{u_i}{u_{i+1}} + \sum_{i=1}^{h-1} \lonorm{v_i}{v_{i+1}}.$$
Recall that we say a configuration $(a,b) \in \fdfreesp$ is \emph{$x$-separated} if $\abs{x(a) - x(b)} \geq 2$ and is
\emph{$y$-separated} if $\abs{y(a) - y(b)} \geq 2$.
We now describe the algorithm given in the following main lemma of this section.

\begin{lemma}
\lemlab{small-opt-alg}
Let $\envir$ be a polygonal environment with $n$ vertices, let $\robA,\robB$ be two robots each modeled as a unit square, and let $\fd{s} = (s_A,s_B),\fd{t} = (t_A,t_B) \in \fdfreesp$. There is an algorithm that in $O(n \log^2 n)$ time either reports a value $\gamma \leq 2$ with $\plancost{\fdpi^*} \leq \gamma \leq 8 \plancost{\fdpi^*}$ or reports that $\plancost{\fdpi^*} > 1/4$; when both such a value $\gamma$ exists and $\plancost{\fdpi^*} > 1/4$, it reports either outcome arbitrarily.
\end{lemma}

\mparagraph{Algorithm} Let $\Abox \assign (s_A + (1/4)\Box) \cap (t_A + (1/4)\Box)$ and $\Bbox \assign (s_B + (1/4)\Box) \cap (t_B + (1/4)\Box)$. The algorithm searches for a $(\fd{s},\fd{t})$-plan $\fdpi = (\pi_A,\pi_B)$ contained in $\Abox \times \Bbox$ with $\altern{\fdpi} \leq 4$ and minimum $L_1$-cost, $\loplancost{\fdpi}$. As we will prove, the search only needs to be successful at finding such a plan when $\plancost{\fdpi^*} \leq 1/4$, so the algorithm is described assuming that is true.

We now describe the rest of the algorithm assuming that $A$ moves first; by repeating the subroutines with the roles of $A$ and $B$ swapped we cover both cases. There are three main steps.

\mparagraph{Step (I)} We first do a simple check.
Let $C_A$ (resp., $C_B$) be the component of $\Abox \cap \freesp$ (resp., $\Bbox \cap \freesp$) containing $s_A$ (resp., $s_B$).
If $t_A \notin C_A$ (resp., $t_B \notin C_B$) then $s_A,t_A$ (resp., $s_B,t_B$) lie in different components of $\freesp \cap \Abox$ (resp., $\freesp \cap \Bbox$) and we report that $\plancost{\fdpi^*} > 1/4$. Otherwise, we proceed to Step (II).

\mparagraph{Step (II)} Now let $C_A' \subseteq C_A$ (resp., $C_B' \subseteq C_B$) be the component of $C_A \cap \freept{s_B}$ (resp., $C_B \cap \freept{t_A}$) containing $s_A$ (resp., $t_B$). It is possible that $C_A = C_A'$ or $C_B = C_B'$. We next check if there exists a plan with at most two moves: We first check if $t_A \in C_A'$ and $s_B \in C_B'$. If so, there exists an $xy$-monotone path $\pi_A$ from $s_A$ to $t_A$ in $C_A'$, \ie, while $B$ is parked at $s_B$, and an $xy$-monotone path from $s_B$ to $t_B$ in $C_B'$, \ie, while $A$ is parked at $t_A$, by \lemref{simple-component}. Then we report the cost $\plancost{\fdpi}$ of the corresponding $xy$-monotone plan $\fdpi$.
Otherwise, we proceed to the next step, Step (III).

We will later prove that if $\plancost{\fdpi^*} \leq 1/4$ and $\fd{s},\fd{t}$ are both $x$-separated or both $y$-separated, then Step (II) must find and report a plan $\fdpi$. Hence $\fd{s}$ is only $x$-separated and $\fd{t}$ is only $y$-separated, or vice-versa.
So, we continue our search for a plan $\fdpi$ assuming without loss of generality that $\fd{s}$ is $x$-separated and $\fd{t}$ is $y$-separated in Step (III).
\begin{figure}
\centering
\includegraphics[scale=0.90]{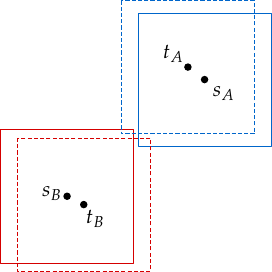}
\caption[Illustration of $s_A,t_A,s_B,t_B$ positioned as assumed in Step (III) of the algorithm in \lemref{small-opt-alg}.]{Illustration of $s_A,t_A,s_B,t_B$ positioned as assumed in Step (III) of the algorithm in \lemref{small-opt-alg}, where $s_A + \Box, s_B+\Box$ are solid and $t_A + \Box, t_B + \Box$ are dashed.}
\figlab{small-opt-sep}
\end{figure}

\mparagraph{Step (III)} For concreteness, assume that $$x(s_B) \leq x(s_A) - 2 \text{\ and\ } y(t_B) \leq y(t_A) - 2.$$
Under the assumption $\plancost{\fdpi^*} \leq 1/4$ it can be shown that $$x(s_B) \leq x(s_A) - 2 \text{\ and \ } y(s_A) - 2 < y(s_B) \leq y(s_A) - 7/4,$$ and $$y(t_B) \leq y(t_A) - 2 \text{\ and\ } x(t_A) - 2 < x(t_B) \leq x(t_A) - 7/4.$$
See \figref{small-opt-sep}. We search for a plan $\fdpi$ with at most four moves, \ie, $\fdpi$ is of the following form, for two points $p_A \in \Abox \cap \freesp$ and $p_B \in \Bbox \cap \freesp$:
$$\moveseq{\fdpi} = (A,\pi_1,s_B), (B,\pi_2,p_A), (A,\pi_3,p_B), (B,\pi_4,t_A).$$
If $p_A$ (resp., $p_B$) is $s_A$ or $t_A$ (resp., $s_B$ or $t_B$), then the first or last move by $A$ (resp., $B$) is the trivial path, respectively. For any configuration $(p_A,p_B) \in \fdfreesp$, let $\fd{\Pi}(p_A,p_B)$ be the plan where $A$ moves from $s_A$ then parks at $p_A$ on the first move, $B$ moves from $s_B$ then parks at $p_B$ on the second move, $A$ moves from $p_A$ then parks at $t_A$ on the third move, and finally $B$ moves from $p_B$ to $t_B$ on the fourth move (if possible). We define a set of candidate configurations in $\fdfreesp$, then we choose and report the plan $\fd{\Pi}(p_A,p_B)$ which is feasible and minimizes its $L_1$-cost $\loplancost{\fd{\Pi}(p_A,p_B)}$ over all candidate configurations $(p_A,p_B)$. The details are as follows.

Let $\lines_A$ be the set of axis-parallel lines that contain $s_A,t_A$, \ie, $$\lines_A \assign \{x=x(s_A), y=y(s_A), x=x(t_A), y=y(t_A)\},$$ and let $\lines_B$ be the set of axis-parallel lines that contain $s_B,t_B$. Let $\phi$ be the vector $(1,1)$. Let $\widetilde{Q}$ be the overlay of $(C_A' \cup \lines_A)-\phi$ and $(C_B' \cup \lines_B) + \phi$. Finally, let $\widetilde{Q}^{||}$ be the overlay of $\widetilde{Q}$ with a set of vertical lines through every vertex of $\widetilde{Q}$. Then every vertex of $\widetilde{Q}^{||}$ lies on a vertical line that contains at least one real vertex of $\widetilde{Q}$. 
Let $\widetilde{V}$ be the vertices of $\widetilde{Q}^{||}$.
Let $\widetilde{\bm{\EuScript{P}}} \subseteq \widetilde{V} \times \widetilde{V}$ be the subset of pairs $(\widetilde{p}_A,\widetilde{p}_B)$ such that:
\begin{enumerate}[(i)]
\item $\widetilde{p}_A + \phi \in C_A'$ and $\widetilde{p}_B - \phi \in C_B'$, and
\item $x(\widetilde{p}_A) = x(\widetilde{p}_B)$ and $y(\widetilde{p}_A) \geq y(\widetilde{p}_B)$, \ie, $\widetilde{p}_B$ lies below $\widetilde{p}_A$ on the same vertical line,
\item and the plan $\fd{\Pi}(\widetilde{p}_A+\phi,\widetilde{p}_B-\phi)$ is feasible.
\end{enumerate}
If $\abs{\widetilde{\bm{\EuScript{P}}}} = \varnothing$, we report that $\plancost{\fdpi^*} > 1/4$. Otherwise, we report the $L_2$-cost $\plancost{\fd{\Pi}(\widetilde{p}_A+\phi,\widetilde{p}_B-\phi)}$ of the corresponding plan $\fd{\Pi}(\widetilde{p}_A+\phi,\widetilde{p}_B-\phi)$ for the pair $(\widetilde{p}_A,\widetilde{p}_B) \in \widetilde{\bm{\EuScript{P}}}$ that minimizes the $L_1$-cost $\loplancost{\fd{\Pi}(\widetilde{p}_A+\phi,\widetilde{p}_B-\phi)}$.

This concludes the algorithm.

\mparagraph{Correctness}
It is easy to verify that if the algorithm succeeds to find a plan $\fdpi$ and reports its $L_2$-cost $\plancost{\fdpi}$ in Step (II) or Step (III) that $\fdpi \subset \Abox \times \Bbox$, $\altern{\fdpi} \leq 4$, and $\fdpi$ is feasible.
If the algorithm reports $\plancost{\fdpi^*} > 1/4$ in Step (I), then $s_A,t_A$ (resp., $s_B,t_B$) lie in different components of $C_A$ (resp., $C_B$) and hence the path $\pi_A$ (resp., $\pi_B$) in any feasible $(\fd{s},\fd{t})$-plan $(\pi_A,\pi_B)$ must exit $\Abox$ (resp., $\Bbox$). So the algorithm behaves correctly in this case.
If the algorithm reports a plan $\fdpi$ in Steps (II) or (III), all parking places of $A$ (resp., $B$) are contained in $\Abox$ (resp., $\Bbox)$ and hence the $L_2$-cost of each ($xy$-monotone) move is at most $1/2$. It follows that $\plancost{\fdpi} \leq (1/2)\altern{\fdpi} \leq 2$.

First suppose $\plancost{\fdpi^*} > 1/4$. If the algorithm fails in both Step (II) and Step (III) to find any path and report its cost, it correctly reports $\plancost{\fdpi^*} > 1/4$. Otherwise, the algorithm reports the cost $\plancost{\fdpi}$ of a plan $\fdpi$, where $\plancost{\fdpi} \leq 2$ by the discussion above. Then $$\plancost{\fdpi^*} \leq \plancost{\fdpi} \leq 2 \leq 8\plancost{\fdpi^*}.$$ In either case, the algorithm behaves as claimed.

Next, suppose $\plancost{\fdpi^*} \leq 1/4$. If the algorithm succeeds in Step (II), the cost reported is $\plancost{\fdpi^*}$, and the algorithm behaves as claimed. So suppose Step (II) fails. As claimed in the description of the algorithm, it must be that $\fd{s}$ is only $x$-separated and $\fd{t}$ is only $y$-separated, or vice-versa. Indeed, for sake of contradiction, suppose $\fd{s},\fd{t}$ are both, say, $x$-separated. Then \lemref{same-cells-plan} implies there exists an (optimal $xy$-monotone) $(\fd{s},\fd{t})$-plan with at most two moves since there is a unit square that contains $\Abox$ and one that contains $\Bbox$. Step (II) checks for such plans, so it must succeed in this case, which is a contradiction. Henceforth, we assume $A$ moves first in $\fdpi^*$ and $\fd{s},\fd{t}$ are oriented as assumed in the algorithm, \ie, $\fd{s}$ is only $x$-separated and $\fd{t}$ is only $y$-separated, $$x(s_B) \leq x(s_A) - 2 \text{\ and \ } y(s_A) - 2 < y(s_B) \leq y(s_A) - 7/4,$$ and $$y(t_B) \leq y(t_A) - 2 \text{\ and\ } x(t_A) - 2 < x(t_B) \leq x(t_A) - 7/4.$$

To finish the proof, we prove that Step (III) succeeds to find a plan $\fdpi \subset \Abox \times \Bbox$, under the assumption that $\plancost{\fdpi^*} \leq 1/4$, with $\plancost{\fdpi} \leq 8\plancost{\fdpi^*}$. Let $(\pi_A^*,\pi_B^*) = \fdpi^*$. Since $\pi_A^*,\pi_B^*$ are continuous, there is a time instance $\lambda \in (0,1)$ such that $(q_A,q_B) = \fdpi^*(\lambda)$ is both $x$-separated and $y$-separated, in particular, $\abs{x(q_A)-x(q_B)} = 2$. Then $q_A \in \Abox$, $q_B \in \Bbox$, and $x(q_B) = x(q_A) - 2$ since $\plancost{\fdpi^*} \leq 1/4$ and $x(s_B) < x(s_A) - 2$. By \lemref{simple-component}, there exists an $xy$-monotone optimal $(\fd{s},\fd{q})$-plan $\fdpi_0$ with at most two moves, since $\fd{s},\fd{q}$ are both $x$-separated, and there exists an $xy$-monotone optimal $(\fd{q},\fd{t})$-plan $\fdpi_1$ with at most two moves, since $\fd{q},\fd{t}$ are both $y$-separated. Then $\moveseq{\fdpi_0} \circ \moveseq{\fdpi_1}$ is an $xy$-monotone optimal $(\fd{s},\fd{t})$-plan. So assume $\moveseq{\fdpi^*} = \moveseq{\fdpi_0} \circ \moveseq{\fdpi_1}$. In particular, $\fdpi^*$ is of the form $$\moveseq{\fdpi^*} = (A,\pi_1,s_B), (B,\pi_2,q_A), (A,\pi_3,q_B), (B,\pi_4,t_A).$$

Then $\pi_1 \subset C_A'$ (resp., $\pi_4 \subset C_B'$) and hence $q_A \in C_A'$ (resp., $q_B \in C_B'$). Let $\widetilde{q}_A \assign q_A-\phi$ (resp., $\widetilde{q}_B \assign q_B+\phi$), and let $\widetilde{g}_A$ (resp., $\widetilde{g}_B$) be the cell of $\widetilde{Q}^{||}$ containing $\widetilde{q}_A$ (resp., $\widetilde{q}_B$).
By definition of $\widetilde{Q}^{||}$ and the fact that $q_A \in C_A'$ (resp., $q_B \in C_B'$), we have $\widetilde{g}_A \subseteq C_A'-\phi$ (resp., $\widetilde{g}_B \subseteq C_B'+\phi$). See that $x(\widetilde{q}_B) = x(\widetilde{q_A})$ and $y(\widetilde{q}_B) \leq y(\widetilde{q}_A)$ since $x(q_B) = x(q_A) - 2$ and $y(q_B) \leq y(q_A) - 2$. That is, $\widetilde{q}_A,\widetilde{q}_B$ lie on the same vertical line with $\widetilde{q}_A$ above $\widetilde{q}_B$.

Let $\widetilde{p}_A \assign \widetilde{q}_A$ and $\widetilde{p}_B \assign \widetilde{q}_B$ initially. Then $\loplancost{\fd{\Pi}(\widetilde{p}_A+\phi,\widetilde{p}_B-\phi)} = \loplancost{\fdpi^*}$. Using the fact that $\widetilde{Q}$ includes the lines of $\lines_A-\phi,\lines_B+\phi$ and the convexity of $\widetilde{g}_A,\widetilde{g}_B$, it can be shown that $\widetilde{p}_A,\widetilde{p}_B$ can be shifted to vertices of $\widetilde{g}_A,\widetilde{g}_B$, respectively, while maintaining that $x(\widetilde{p}_B) = x(\widetilde{p}_A), y(\widetilde{p}_B) \leq y(\widetilde{p}_A)$, and $\loplancost{\fd{\Pi}(\widetilde{p}_A+\phi,\widetilde{p}_B-\phi)} = \loplancost{\fdpi^*}$.
Then there exists a pair $(\widetilde{p}_A,\widetilde{p}_B) \in \widetilde{V} \times \widetilde{V}$ where $\widetilde{p}_A$ (resp., $\widetilde{p}_B$) is a vertex of $\widetilde{g}_A$ (resp. $\widetilde{g}_B$) with $\loplancost{\fd{\Pi}(\widetilde{p}_A+\phi,\widetilde{p}_B+\phi)} \leq \loplancost{\fdpi^*}$. That is, $(\widetilde{p}_A,\widetilde{p}_B)$ satisfies conditions (i) and (ii) in the definition of $\widetilde{\bm{\EuScript{P}}}$.

To finish the proof, it suffices to prove $(\widetilde{p}_A,\widetilde{p}_B)$ satisfies condition (iii) so that $(\widetilde{p}_A,\widetilde{p}_B) \in \widetilde{\bm{\EuScript{P}}}$. To this end, we show that conditions (i) and (ii) imply (iii); \ie, (iii) is only included to make it by definition that the algorithm only reports costs of feasible plans in Step (III). Let $(\widetilde{p}_A,\widetilde{p}_B) \in \widetilde{V} \times \widetilde{V}$ be a pair that satisfies conditions (i) and (ii). Let $p_A \assign \widetilde{p}_A+\phi$ and $p_B \assign \widetilde{p}_B-\phi$. By definition of $C_A'$ and $C_B'$ and conditions (i) and (ii), we have $p_A \in C_A', p_B \in C_B'$, and $(p_A,p_B) \in \fdfreesp$. Then there is a $xy$-monotone path $\pi_1 \subset \freept{s_B}$ from $s_A$ to $p_A$ and an $xy$-monotone path $\pi_4 \subset \freept{t_A}$ from $p_B$ to $t_B$ by \lemref{simple-component}. We next show there is an $xy$-monotone path $\pi_3 \subset \freept{p_B}$ from $p_A$ to $t_A$. Since $p_A,t_A \in C_A$, there exists an $xy$-monotone path $\pi_3 \subset C_A$ from $p_A$ to $t_A$ by \lemref{simple-component}. It remains to show $\pi_3 \cap \Int(p_B+2\Box) = \varnothing$ so that we may conclude $\pi_3 \subset C_A \setminus (p_B+2\Box) \subset \freept{p_B}$. By definition, $C_B' \subset \freept{t_A}$, and $p_B \in C_B'$, so $(t_A,p_B) \in \fdfreesp$. Then $x(p_B) +2 \leq x(t_A)$ or $y(p_B)+2 \leq y(t_A)$. By condition (ii) and the fact $\pi_3$ is $xy$-monotone, in the former case we have $x(p_B) + 2 = x(p_A) \leq x(t_A)$ so $\pi_3$ does not cross left of the line $x = x(p_B)+2$, and in the latter case we have $y(p_B) + 2 \leq y(p_A),y(t_A)$ so $\pi_3$ does not cross below the line $y = y(p_B)+2$. Hence, in either case, $\pi_3 \cap \Int(p_B+2\Box) = \varnothing$, \ie, $\pi_3 \subset \freept{p_B}$, as desired. A symmetric argument implies there is an $xy$-monotone path $\pi_2 \subset \freept{p_A}$ from $s_B$ to $p_B$. Putting everything together, the plan $\fdpi$ with $\moveseq{\fdpi} = (A, \pi_1, s_B), (B, \pi_2, p_A), (A, \pi_3, p_B), (B, \pi_4, t_A)$ is a $(\fd{s},\fd{t})$-plan with $\loplancost{\fdpi} = \loplancost{\fdpi^*}$. Then $\plancost{\fdpi} \leq \sqrt{2}\loplancost{\fdpi^*} \leq 8 \plancost{\fdpi^*}$, which completes the proof.

\mparagraph{Runtime analysis}
We first compute the components $C_A,C_A',C_B,C_B'$ in $O(n \log^2 n)$ time \cite{DBLP:books/lib/BergCKO08}. Then Step (I) and Step (II) take $O(n)$ time. Consider Step (III). Since $C_A'$ and $C_B'$ are $xy$-monotone by \lemref{simple-component}, the $O(1)$ lines in $\lines_A-\phi,\lines_B+\phi$ each intersect $O(1)$ segments of $C_A'-\phi,C_B'+\phi$, so the overlay $\widetilde{Q}$ has $O(n)$ vertices and is computed in $O(n \log n)$ time. Furthermore, the vertical lines overlayed with $Q$ to define $\widetilde{Q}^{||}$ each intersects $O(1)$ segments of $\widetilde{Q}$, so $\widetilde{Q}^{||}$ and its set of vertices $\widetilde{V}$ is also computed in $O(n \log n)$ time. Then there are $O(1)$ vertices in $\widetilde{V}$ that lie on any vertical line. It follows that, by condition (ii) in the definition of $\widetilde{\bm{\EuScript{P}}}, \abs{\widetilde{\bm{\EuScript{P}}}} = O(n)$; in particular, $O(1)$ pairs in $\widetilde{\bm{\EuScript{P}}}$ lie on any common vertical line. For a given pair $(\widetilde{p}_A,\widetilde{p}_B) \in \widetilde{V} \times \widetilde{V}$, condition (i) is checked in $O(1)$ time by marking the faces $g$ of $\widetilde{Q}^{||}$ which of $C_A'-\phi,C_B'+\phi$ (possibly both) that contain $g$ when $\widetilde{Q}^{||}$ is computed. Condition (iii) is implied by (i) and (ii), as argued above, so it does not need to be checked directly (it is only included in the defintion so that the algorithm obviously only reports costs of feasible plans). It follows that $\widetilde{\bm{\EuScript{P}}}$ can be computed in $O(n)$ time. If $\abs{\widetilde{\bm{\EuScript{P}}}} = 0$ we report $\plancost{\fdpi^*} > 1/4$, otherwise we find the pair $(\widetilde{p}_A,\widetilde{p}_B)$ for which $\plancost{\fd{\Pi}(\widetilde{p}_A+\phi,\widetilde{p}_B-\phi)}$ is minimized in $O(1)$ time per pair, using the fact that $$\plancost{\fd{\Pi}(\widetilde{p}_A+\phi,\widetilde{p}_B-\phi)} = \lonorm{s_A}{(\widetilde{p}_A+\phi)} + \lonorm{(\widetilde{p}_A+\phi)}{t_A} + \lonorm{s_B}{(\widetilde{p}_B-\phi)} + \lonorm{(\widetilde{p}_B-\phi)}{t_B}.$$ Overall, the algorithm takes $O(n \log^2 n)$ time.

\section{Conclusion}

We have described a $(1+\eps)$-approximation algorithm for the min-sum motion planning problem for two congruent square robots in a planar polygonal environment with running time $n^2\eps^{-O(1)}\log n$, \ie, our algorithm is an FPTAS. We also describe an $O(n \log^2 n)$-time $8$-approximation algorithm for the problem when the cost of the optimal plan is less than $1/4$, which is used as a subroutine in our FPTAS. We conclude with some questions for future work. Can our techniques be extended
\begin{enumerate}[(i)]
\item to obtain a $(1+\eps)$-approximation algorithm for min-sum motion planning for $k > 2$ robots with running time $(n/\eps)^{O(k)}$?
\item to work for translating robots with congruent shapes other than squares, such as other centrally-symmetric regular polygons, disks, or convex polygons?
\item to optimize both \emph{clearance} and the total lengths of the paths in some fashion, where clearance is the minimum distance from any robot to any other robot or obstacle during the plan?
\end{enumerate}

\mparagraph{Acknowledgement} We thank Vera Sacrist\'{a}n and Rodrigo Silveira for helpful discussions.



\bibliographystyle{abbrv}
\bibliography{ms.bib}

\end{document}